\documentclass[journal]{IEEEtran}
\usepackage{amsmath,amsfonts,amssymb,amsthm,mathrsfs}
\usepackage{algorithmic}
\usepackage[linesnumbered,ruled,vlined]{algorithm2e}
\usepackage{array}
\usepackage{mathtools}
\usepackage{textcomp}
\usepackage{url}
\usepackage{verbatim}
\usepackage{cite}
\usepackage{color}
\usepackage{xcolor}
\usepackage{dsfont}
\usepackage{multirow}
\usepackage[normalem]{ulem}
\useunder{\uline}{\ul}{}
\usepackage{graphicx}
\usepackage{subfigure}
\usepackage{mdframed}
\usepackage{booktabs}
\usepackage{makecell}
\usepackage{enumerate}
\usepackage{enumitem}

\useunder{\uline}{\ul}{}
\allowdisplaybreaks[4]
\DeclareMathOperator*{\argmax}{arg\,max}

\DeclareMathOperator*{\TV}{\operatorname{TV}}
  
\theoremstyle{definition}
\newtheorem{assumption}{Assumption}
\newtheorem{theorem}{Theorem}
\newtheorem{definition}{Definition}
\newtheorem{corollary}{Corollary}
\newtheorem{lemma}{Lemma}

\newtheorem{proposition}{Proposition}

\begin{document}

\title{One-Shot Federated Model Editing for Device Dynamics at the Edge}

\author{Youming~Tao,
Weijun~Yin, 
Huanyi~Xie, 
Yanqing~Yao,
Dongxiao~Yu,~\IEEEmembership{Senior Member,~IEEE},
Xiuzhen~Cheng,~\IEEEmembership{Fellow,~IEEE},
Di~Wang,~\IEEEmembership{Member,~IEEE},
Falko~Dressler~\IEEEmembership{Fellow,~IEEE}
\thanks{This work was supported by the National Natural Science Foundation of China (Grant No. 62572025), Beijing Advanced Innovation Center for Future Blockchain and Privacy Computing (Grant No. GJJ-25-xxx), and the Fundamental Research Funds for the Central Universities, China.}
\thanks{Youming Tao and Falko Dressler are with the School of Electrical Engineering and Computer Science, TU Berlin, 10623 Berlin, Germany (e-mail: tao@ccs-labs.org; dressler@ccs-labs.org).}%
\thanks{Weijun Yin and Yanqing Yao are with the School of Cyber Science and Technology, Beihang University, Beijing 100191, China, and also with the Beijing Advanced Innovation Center for Future Blockchain and Privacy Computing, Beihang University, Beijing 100083, China (e-mail: yinweijun@buaa.edu.cn; yaoyq@buaa.edu.cn).}%
\thanks{Huanyi Xie and Di Wang are with the Division of Computer, Electrical and Mathematical Sciences and Engineering, King Abdullah University of Science and Technology, Saudi Arabia (e-mail: huanyi.xie@kaust.edu.sa; di.wang@kaust.edu.sa).}%
\thanks{Dongxiao Yu and Xiuzhen Cheng are with the School of Computer Science and Technology, Shandong University, Qingdao 266237, China (e-mail: dxyu@sdu.edu.cn; xzcheng@sdu.edu.cn).}%
\thanks{The first two authors contributed equally.}
\thanks{Corresponding author: Y. Yao and D. Wang.}
}



\maketitle

\begin{abstract}
Standard federated learning (FL) requires persistent device connectivity across multiple communication rounds, making it unsuitable for dynamic environments such as vehicular networks, where devices frequently join and leave. One-shot FL offers a promising alternative by completing training in a single communication round. 
However, one-shot FL suffers from model solidification and cannot readily adapt to subsequent client dynamics, including new client contributions and the withdrawal of previous contributions. We therefore consider exact model editing, which removes withdrawn contributions through unlearning or incorporates newly contributed ones through incremental learning without retraining from scratch.
We identify that efficient editing fundamentally hinges on learning stability. 
To this end, we propose F-SOSA, the first stable one-shot FL algorithm that enables rapid and exact model editing. F-SOSA leverages sub-sampling of local models with carefully tuned rate for stable aggregation and performs curvature-aware model fusion to preserve utility. 
We prove that F-SOSA achieves vanishing training error and reasonably small generalization error bounds.
Building on this stability, we further develop two single-round editing algorithms: F-SOSA-U for unlearning and F-SOSA-I for incremental learning. 
They ensure that the edited model, along with all intermediate algorithmic states, is statistically indistinguishable from retraining while incurring only an expected $O(\rho)$ fraction of its cost, where $\rho$ quantifies stability.
Experiments demonstrate that our approach matches retraining performance while reducing cumulative time and communication overhead by at least 36\% and 22\%, respectively, under realistic urban mobility simulations.

\end{abstract}


\section{Introduction}
Federated learning (FL)~\cite{mcmahan2017communication} enables collaborative model training across distributed devices without sharing raw data. 
In mobile edge computing (MEC) scenarios such as vehicular networks, client participation is inherently dynamic. 
Devices frequently join or leave due to mobility or network instability. 
This challenges standard FL algorithms like FedAvg~\cite{mcmahan2017communication}, FedProx~\cite{li2020federated}, and SCAFFOLD~\cite{karimireddy2020scaffold}. These methods rely on a largely fixed client population and multiple rounds of communication. Such assumptions rarely hold in practice.

Communication efficiency is also critical in bandwidth-constrained edge environments~\cite{tao2023byzantine-resilient,tao2024private,chen2026fedsubmuon}. To reduce communication rounds and reliance on persistent availability, one-shot FL~\cite{guha2019one} has emerged, which completes training in a single round. 
Recent studies~\cite{jhunjhunwala2024fedfisher,liu2024fedlpa,zhang2022dense} show that one-shot FL can achieve competitive performance over non-iid data. 
This makes it well-suited for dynamic systems. 
Vehicular networks provide a natural setting for collaborative learning across vehicles~\cite{yu2021decentralized}. 
A representative use case is vehicular micro-clouds (VMCs)~\cite{dressler2019virtual}, where nearby vehicles temporarily pool their resources for edge collaboration. Vehicle mobility leads to changes in VMC membership as vehicles join and leave~\cite{zhou2024optimizing}, making persistent multi-round participation difficult to maintain.
One-shot FL is therefore a natural solution.

However, one-shot FL suffers from model solidification. Once trained, the global model cannot adapt to changes in the client population. When clients arrive or depart, the set of participants contributing to the collaborative model changes. A newly arrived client can contribute its local model and benefit from the resulting global model, whereas a departing client may withdraw its previous contribution for privacy before disconnecting from the coordinator. Supporting such contribution additions and withdrawals requires the one-shot model to accommodate incremental learning and unlearning after its initial training. This issue is particularly relevant in dynamic edge applications such as VMCs, where vehicles participate only temporarily and repeated retraining after every arrival or departure would incur substantial communication and computation overhead.
Approximate editing methods such as fine-tuning or model calibration can reduce this cost, but successive approximate updates may deviate from the model produced by retraining and propagate this deviation to later edits. Exact editing instead restores the same distribution as retraining after every change, preventing such history-dependent drift.
These challenges motivate the following question:

\begin{mdframed}[
    leftmargin=0pt,
    rightmargin=0pt,
    innerleftmargin=4pt,
    innerrightmargin=4pt,
    innertopmargin=4pt,
    innerbottommargin=4pt,
]
\textit{Can we efficiently edit a one-shot trained FL model to reflect an updated client set without retraining?}
\end{mdframed}
We formalize this goal as \textit{exact federated model editing}, which
includes two tasks: (1) \textit{client unlearning} to remove a client's previous contribution upon a withdrawal request, and (2) \textit{client incremental learning} to incorporate the contribution of a newly participating client. By exactness, we require statistical indistinguishability between editing and retraining, which applies to  both the final model and all intermediate algorithmic states. To achieve this goal, we propose the first exact model editing framework for one-shot FL. The framework consists of a stable one-shot FL algorithm, F-SOSA, and two editing algorithms: F-SOSA-U for client unlearning and F-SOSA-I for client incremental learning. Our main contributions are summarized as follows:

\begin{enumerate}[leftmargin=*]
    \item We formalize the definition of exact federated model editing through statistical indistinguishability between editing and retraining. Specifically, we require that the joint distributions of the obtained model and all intermediate algorithmic states produced by editing match those generated by retraining on the updated client set.
    \item We design F-SOSA as an edit-friendly one-shot FL algorithm that combines two complementary mechanisms. First, we reformulate curvature-aware model fusion as a projection characterized by client-additive sufficient statistics, so that a single-client change can be incorporated by updating only the corresponding contribution rather than rerunning the full fusion procedure. Second, we introduce randomized client subsampling to make the one-shot output stable in total variation, thereby allowing many membership changes to reuse the existing model and states without recomputation.
    \item We provide theoretical utility guarantees for F-SOSA. Specifically, we establish upper bounds on both training and generalization errors for over-parameterized neural networks. The training error vanishes as network width and local training increase. The generalization error converges to a constant governed by data non-degeneracy and becomes favorably small as the numbers of data samples and clients grow.
    \item Building upon F-SOSA, we develop F-SOSA-U and F-SOSA-I for client unlearning and incremental learning. Thanks to F-SOSA’s total variation stability quantified by $\rho<1$, not all editing requests need execution. We design a lightweight verification to determine the necessity. As a result, only an expected $O(\rho)$-fraction of requests trigger the editing process.
    \item We conduct extensive experiments on standard benchmarks to evaluate model utility, editing exactness, and efficiency. F-SOSA achieves strong accuracy compared with state-of-the-art one-shot FL methods. F-SOSA-U/I match retraining performance while significantly reducing time and communication costs. We verify the exactness of editing using membership inference attacks. We also present a case study with simulated vehicular scenarios to demonstrate the practical applicability of our proposed approach.
\end{enumerate}

\section{Related Work} 
We review three related research directions in FL: unlearning, continual learning, and dynamic client participation.

\noindent\textbf{Federated unlearning.} 
Federated unlearning was first studied in~\cite{liu2021federaser}, which removes a client approximately by reversing its updates. While faster than retraining, it remains computationally expensive. 
Subsequent works improve efficiency through noise perturbation~\cite{fraboni2024sifu,wu2026certified,wu2026trajectoryaware}, knowledge distillation~\cite{wu2022knowledge}, class-discriminative pruning~\cite{wang2022federated}, or gradient ascent~\cite{wu2022federated}. 
However, all these methods provide only approximate removal. 
As a result, the model remains vulnerable to membership inference and gradient leakage~\cite{shokri2017membership,song2020analyzing,wang2019beyond,zhu2019deep,wang2025towards,chen2025fedmua,li2026not}.
Recent studies explore exact unlearning. Liu et al.~\cite{liu2022right} propose a Newton-style retraining method that uses Fisher information for curvature approximation. However, their retraining path deviates from the original trajectory and relies on convexity assumptions. 
Tao et al.~\cite{tao2024communication} introduce a stabilized FedAvg variant called FATS for exact unlearning via resampling and rapid retraining. Their one-shot variant, however, suffers from degraded utility due to simple averaging and does not support client addition.

\noindent\textbf{Federated continual learning.} 
Federated continual learning assumes a fixed set of clients whose local data evolve over time. 
Existing work considers different settings including task-incremental~\cite{bakman2024federated,zhang2024cross,yue2021inexact,ding2024understanding}, class-incremental~\cite{babakniya2023data,dong2022federated,yoon2021federated}, and domain-incremental~\cite{huang2022learn,ma2022efl,zhang2023federated}, which correspond to sequential tasks, new classes, and domain shifts, respectively. 
The primary goal is to mitigate catastrophic forgetting and preserve past knowledge. 
As they prioritize knowledge retention over exact alignment with new data, these methods cannot guarantee statistical equivalence to retraining.
In addition, they rely on multi-round training and continued access to past clients. In contrast, we study dynamic client populations where new clients may join after others have left.

\noindent\textbf{Federated Learning with dynamic participation.} 
Recent studies consider federated learning with dynamic participants. Ruan et al.~\cite{ruan2021towards} and Chang et al.~\cite{chang2024federated} allow clients to join or leave during training and provide convergence guarantees. Their setting differs fundamentally from ours. They study multi-round federated training where client changes occur across rounds and the goal is to adapt during training.  In contrast, we focus on one-shot FL, which involves only a single round. In this setting, client dynamics arise after training rather than during it. Our goal is therefore to enable exact model editing of the trained model to accommodate changes in the client population.

\noindent\textbf{Summary.}
Existing studies address related but distinct challenges. Federated unlearning focuses on removing clients but typically provides only approximate guarantees or suffers from utility degradation in one-shot settings. Federated continual learning assumes a fixed client set and prioritizes knowledge retention rather than exact alignment with updated data. Work on dynamic participation adapts to client churn during multi-round training and does not apply to post-training model editing in one-shot FL. Consequently, none of these directions directly address our setting, where client additions and withdrawals occur after one-shot training. This challenge is amplified under repeated edits, since any residual deviation may propagate to subsequent updates. The goal is therefore not merely post-training adaptation, but efficient and exact editing that remains valid across successive client-set changes.

\section{Preliminaries}\label{sec-pre}

\subsection{System Model}

We consider an edge system consisting of a trusted infrastructure-based coordinator and a dynamic set of devices. The coordinator can, for example, be an RSU or MEC server in a vehicular micro-cloud (VMC), with vehicles serving as participating devices. Time proceeds in discrete rounds. Devices may join or leave at each round. We do not impose a particular distribution on client arrivals or departures; their occurrence may be non-uniform across devices and time. Let $\mathcal M(t)$ denote the set of active devices at round $t$, and let $M(t)=|\mathcal M(t)|$.
The coordinator maintains knowledge of the active device set. When a new device joins, it establishes a connection with the coordinator and may submit its local model for inclusion in the collaborative model. Before an active device departs and loses connectivity to the coordinator, it may issue an unlearning request to withdraw its previous contribution. Standard authenticated communication between devices and the coordinator is assumed.

\subsection{Federated Learning}\label{sec-pre-FL}

Let $M$ denote the number of active devices. Each device $i \in [M]$ holds a local dataset $\mathcal{D}_i = \{(\boldsymbol{x}_{ij}, y_{ij})\}_{j=1}^n$, where $\boldsymbol{x}_{ij} \in \mathbb{R}^p$ denotes the feature vector and $y_{ij} \in \mathbb{R}$ denotes the label. The global dataset is $\mathcal{D} = \bigcup_{i=1}^M \mathcal{D}_i$ with total size $N = Mn$.
The learning objective can be formulated as the following distributed optimization problem:
\begin{equation}\label{eq-flopt}
    \min_{\boldsymbol{W} \in \mathbb{R}^d} 
L_\mathcal{D}^{f,\ell}(\boldsymbol{W}) \coloneqq \frac{1}{M} \sum_{i=1}^M 
L_{\mathcal{D}_i}^{f,\ell}(\boldsymbol{W}),
\end{equation}
where $L_{\mathcal{D}_i}^{f,\ell}(\boldsymbol{W}) \coloneqq \frac{1}{n} \sum_{(\boldsymbol{x}, y) \in \mathcal{D}_i} 
\ell(f(\boldsymbol{W}, \boldsymbol{x}), y)$, $f(\boldsymbol{W}, \cdot)$ denotes the model parameterized by $\boldsymbol{W}$, and $\ell(\cdot,\cdot)$ denotes the loss function. Following the standard streaming-edit setting~\cite{ullah2021machine}, we consider editing sequences in which the active population remains within a constant factor of its initial size, i.e., $M(0)/2 \le M(t) \le 2M(0)$.

\subsection{Federated Model Editing}

To support dynamic client participation, our goal is to enable the global model to adapt efficiently to changes in the client set. Specifically, we aim to (i) remove the influence of departed clients that withdraw their previous contributions (client unlearning), and (ii) incorporate newly arrived clients (client incremental learning), while preserving model utility and avoiding full retraining.

Let $\mathfrak{M}$ denote the space of devices, $\mathcal{W}$ the space of model parameters, and $\mathcal{S}$ the space of internal algorithmic states, which will be specified later. A one-shot FL algorithm is a randomized map $\mathcal{A}: \mathfrak{M}^* \to \mathcal{W} \times \mathcal{S}$. Given a set of devices, it produces a global model together with internal algorithmic states. A model editing algorithm is a map $\mathcal{U}: \mathcal{W} \times \mathcal{S} \times \mathfrak{M}^* \to \mathcal{W} \times \mathcal{S}$. It takes the current model, the associated states, and a set of client addition or removal requests, and outputs an updated model and states. Examples of internal states include sketches of device information or intermediate computations that can be reused during editing.

\begin{definition}[Exact Federated Model Editing]\label{def-data-unlearning}
Let $\mathcal{M}$ and $\mathcal{M}^\prime$ denote the device sets before and after a change. Given a one-shot FL algorithm $\mathcal{A}$, a model editing algorithm $\mathcal{U}$ satisfies exact model editing if for any measurable event $\mathcal{E} \subseteq \mathcal{W} \times \mathcal{S}$,
\[
\mathbb{P}\big[\mathcal{A}(\mathcal{M}^\prime) \in \mathcal{E}\big] =
\mathbb{P}\big[\mathcal{U}(\mathcal{A}(\mathcal{M}), \mathcal{M} \oplus \mathcal{M}^\prime) \in \mathcal{E}\big],
\]
where $\mathcal{M} \oplus \mathcal{M}^\prime$ denotes the symmetric difference, i.e.,
$\mathcal{M} \oplus \mathcal{M}^\prime = (\mathcal{M} \setminus \mathcal{M}^\prime) \cup (\mathcal{M}^\prime \setminus \mathcal{M})$.
\end{definition}

This definition requires that both the output model and all internal algorithmic states after editing be statistically indistinguishable from those obtained by retraining on the updated device set. Requiring exactness for the internal states is important for repeated editing: these states are reused by subsequent editing operations, and any approximation retained in them may be propagated to later model updates. If one edit produces exactly the same joint distribution as retraining, the resulting state can be treated as a fresh execution of the learning algorithm on the updated client set. The guarantee thus extends to an arbitrary sequence of edits by induction, without accumulating approximation error across requests.

\subsection{Total Variation (TV) Stability}

Tao et al.~\cite{tao2024communication} showed that exact federated unlearning can be achieved efficiently when the FL algorithm is stable in total variation (TV) distance. We extend this insight to general model editing. 
Let $P = \mathcal{A}(\mathcal{M})$ and $Q = \mathcal{A}(\mathcal{M}^\prime)$ denote the distributions of the one-shot FL output, including the model and internal states, before and after editing. 
Model editing can be viewed as transporting probability mass from $P$ to $Q$.

Retraining corresponds to a trivial transport plan that samples independently from $Q$, but it incurs high computational cost. Efficient editing instead exploits the overlap between $P$ and $Q$ by reusing existing computations. Let $\Pi(P,Q)$ denote the set of couplings between $P$ and $Q$. To quantify editing cost, we follow~\cite{tao2024communication} and adopt the 0-1 cost function $c(x, y) = \mathds{1}\{x \ne y\}, \quad \forall x, y \in \Omega \coloneqq \mathcal{W} \times \mathcal{S}$. Let $\TV(P,Q)$ denote the total variation distance between $P$ and $Q$:
\begin{equation*}
    \TV(P,Q) \coloneqq \sup_{\text{measurable } Z} |P(Z) - Q(Z)| 
    = \frac{1}{2}\|\phi_P - \phi_Q\|_1,
\end{equation*}
where $\phi_P$ and $\phi_Q$ are the probability densities of $P$ and $Q$. The minimal expected editing cost is then
\begin{equation*}
    \inf_{\pi \in \Pi(P,Q)} \mathbb{E}_{(x,y)\sim\pi}[\mathds{1}\{x \ne y\}] = \TV(P,Q).
\end{equation*}
Thus, only an $O(\TV(P,Q))$ fraction of editing requests requires actual recomputation.
In practice, the future device changes are unknown at training time. When generating $P$, we do not know the corresponding target distribution $Q$. To enable efficient editing under arbitrary future changes, a one-shot FL algorithm must guarantee uniformly bounded TV distance across all possible single-device edits. This motivates the following notion of stability.

\begin{definition}[$\rho$-TV-Stable One-Shot FL Algorithm]\label{def-tvstability}
A one-shot FL algorithm $\mathcal{A}$ is $\rho$-TV-stable if for any device sets $\mathcal{M}$ and $\mathcal{M}^\prime$ with $|\mathcal{M} \oplus \mathcal{M}^\prime| = 1$, we have
$\TV\big(\mathcal{A}_1(\mathcal{M}), \mathcal{A}_1(\mathcal{M}^\prime)\big) \le \rho$,
where $\mathcal{A}_1(\cdot)$ denotes the model output of $\mathcal{A}$.
\end{definition}
Note that Definition~\ref{def-tvstability} only requires stability of the marginal distribution of model output. The internal states need not be considered. For multiple device changes, suppose $|\mathcal{M} \setminus \mathcal{M}^\prime| = k_1$ and $|\mathcal{M}^\prime \setminus \mathcal{M}| = k_2$. By the triangle inequality of TV distance and repeated application of Definition~\ref{def-tvstability}, we have
$\TV\big(\mathcal{A}_1(\mathcal{M}), \mathcal{A}_1(\mathcal{M}^\prime)\big) \le (k_1 + k_2)\rho$.

\begin{algorithm}[t]
\caption{\texttt{F-SOSA}: Fisher-Informed Subsample One-Shot Averaging}
\label{algo-FSOSA}
\KwIn{Initial model $\boldsymbol{W}^{(0)}$, number of local iterations $K$, learning rate $\eta$, stability $\rho$, sample size $P$}

\textbf{Coordinator} samples $\mathcal P$ uniformly from the size-$P$ subsets of $\mathcal M$\;

\textbf{Coordinator} broadcasts $\boldsymbol{W}^{(0)}$ to clients\;

\For{\textbf{\textup{every}} \textbf{client} $i\in\mathcal{P}$ \textbf{\textup{in parallel}}}{
    Set $\boldsymbol{W}_i^{(0)}\gets \boldsymbol{W}^{(0)}$ \;
    \For {$k\gets 1,\cdots, K$}{
        $\boldsymbol{W}_i^{(k)}\gets \boldsymbol{W}_i^{(k-1)}-\eta\nabla L_{\mathcal{D}_i}(\boldsymbol{W}_i^{(k-1)})$\;
    }
    Set $\boldsymbol{\tilde{W}}_i\gets\boldsymbol{W}_i^{(K)}$\;
    Compute the Fisher $\boldsymbol{F}_i\gets\frac{1}{n}\sum\limits_{j=1}^{n}\mathbb{E}_{\boldsymbol{y}}[\nabla \log\mathbb{P}(\boldsymbol{y}|\boldsymbol{x}_{ij}, \boldsymbol{\tilde{W}}_i)\nabla \log\mathbb{P}(\boldsymbol{y}|\boldsymbol{x}_{ij}, \boldsymbol{\tilde{W}}_i)^\top]$\;
    Communicate $\boldsymbol{\tilde{W}}_i$ and $\boldsymbol{F}_i$ to the coordinator\;
}

\textbf{Coordinator} computes and stores $\boldsymbol{\bar{W}}_{\mathcal{P}}$, $\boldsymbol{F}_{\mathcal{P}}$ and $\boldsymbol{b}_{\mathcal{P}}$: $\boldsymbol{\bar{W}}_{\mathcal{P}}\!\gets\!\sum\limits_{i\in\mathcal{P}}\boldsymbol{\tilde{W}}_i/P$, $\boldsymbol{F}_{\mathcal{P}}\!\gets\!\sum\limits_{i\in\mathcal{P}}\boldsymbol{F}_i$, $\boldsymbol{b}_{\mathcal{P}}\!\gets\!\sum\limits_{i\in\mathcal{P}}\boldsymbol{F}_i\boldsymbol{\tilde{W}}_i$\;

\textbf{Coordinator} solves the linear system in $\boldsymbol{\Lambda}\in\mathbb{R}^d$: 
$\boldsymbol{F}_{\mathcal{P}}^{2}\boldsymbol{\Lambda}
=2\big(\boldsymbol{F}_{\mathcal{P}}\boldsymbol{\bar{W}}_{\mathcal{P}}-\boldsymbol{b}_{\mathcal{P}}\big)$\;

\textbf{Coordinator} sets the global model as
$\boldsymbol{\bar{W}}_{\mathcal{P}}^{\text{F}}
\gets
\boldsymbol{\bar{W}}_{\mathcal{P}}-\tfrac{1}{2}\boldsymbol{F}_{\mathcal{P}}\boldsymbol{\Lambda}$\;

\KwOut{$\boldsymbol{\bar{W}}_{\mathcal{P}}^{\text{F}}$}
\end{algorithm}

\section{Our Method}\label{sec-osfr}

\subsection{TV-Stable One-Shot Federated Learning: F-SOSA}
\label{sec:f-sosa}

The federated learning admits a posterior-inference interpretation. Specifically, the learning objective in~\eqref{eq-flopt} corresponds to finding the mode of a global posterior distribution over model parameters. This viewpoint has recently been used to derive one-shot FL algorithms that aggregate local posterior information in a single communication round~\cite{jhunjhunwala2024fedfisher}. We adopt this probabilistic formulation as a technical starting point, but our design goal is fundamentally different. Existing posterior-based one-shot FL methods primarily optimize a single fusion operation. In contrast, our setting requires the fused model to support repeated post-training edits as the client set changes. This creates two additional requirements. First, the fusion state should admit efficient single-client addition and removal without rerunning the entire fusion procedure. Second, the output should remain stable under a single-client change so that previously computed results can be reused whenever possible. We address these requirements in two stages: we first derive a curvature-aware projection characterized by client-additive statistics, and then introduce randomized client subsampling to control total-variation stability.

We first recall the probabilistic formulation. Following standard treatments in posterior-based FL~\cite{jhunjhunwala2024fedfisher}, we assume the loss corresponds to a negative log-likelihood from an exponential-family model, which holds for common choices such as squared loss and cross-entropy.

\begin{assumption}
Given $\boldsymbol{z}=f(\boldsymbol{W},\boldsymbol{x})$, $\ell(\boldsymbol{z},\boldsymbol{y})\propto -\log\mathbb{P}(\boldsymbol{y}|\boldsymbol{z})$ where $\mathbb{P}(\boldsymbol{y}|\boldsymbol{z})=h(\boldsymbol{y})\exp(\boldsymbol{z}^\top T(\boldsymbol{y})-A(\boldsymbol{z}))$ and $h(\boldsymbol{y}), T(\boldsymbol{y}), A(\boldsymbol{z})$ are some real-valued functions. 
\end{assumption}

Under this assumption, the likelihood of a data point $(\boldsymbol{x}, \boldsymbol{y})$ given parameters $\boldsymbol{W}$ can be written as
\begin{equation*}
\begin{aligned}
    \mathbb{P}((\boldsymbol{x}, \boldsymbol{y})|\boldsymbol{W})
    &=\mathbb{P}(\boldsymbol{x})\mathbb{P}(\boldsymbol{y}|\boldsymbol{x},\boldsymbol{W})\\
    &\propto q(\boldsymbol{x})\exp\!\big(-\ell(f(\boldsymbol{W}, \boldsymbol{x}),\boldsymbol{y})\big),
\end{aligned}
\end{equation*}
where $q(\cdot)$ is a prior on $\boldsymbol{x}$ independent of $\boldsymbol{W}$. Under a flat prior $\mathbb{P}(\boldsymbol{W})\propto 1$, the global Maximum A Posteriori (MAP) estimator coincides with the minimizer of~\eqref{eq-flopt}~\cite{al2021federated}. Moreover, the global posterior factorizes as
\begin{equation*}
    \mathbb{P}(\boldsymbol{W}|\mathcal{D}) \propto \prod_{i=1}^M \mathbb{P}(\boldsymbol{W}|\mathcal{D}_i).
\end{equation*}
This factorization suggests a one-shot protocol where each client transmits a summary of its local posterior once and the server combines these summaries to obtain a global MAP estimate. Unfortunately, an analytical expression of $\mathbb{P}(\boldsymbol{W}|\mathcal{D}_i)$ is generally intractable for modern models. A standard remedy is to approximate each local posterior near its mode. Let
\begin{equation*}
    \boldsymbol{\tilde{W}}_i \in \argmax_{\boldsymbol{W}} \mathbb{P}(\boldsymbol{W}|\mathcal{D}_i),
\end{equation*}
which is obtained by local training. A second-order expansion around $\boldsymbol{\tilde{W}}_i$ yields a Gaussian (Laplace) approximation:
\begin{equation*}
\log\mathbb{P}(\boldsymbol{W}|\mathcal{D}_i)
\approx
\log\mathbb{P}(\boldsymbol{\tilde{W}}_i|\mathcal{D}_i)
-\frac{n}{2}(\boldsymbol{W}-\boldsymbol{\tilde{W}}_i)^\top
\boldsymbol{H}_i
(\boldsymbol{W}-\boldsymbol{\tilde{W}}_i),
\end{equation*}
where
\[
\boldsymbol{H}_i
= -\frac{1}{n}\nabla^2\log\mathbb{P}(\boldsymbol{W}|\mathcal{D}_i)\big|_{\boldsymbol{W}=\boldsymbol{\tilde{W}}_i}.
\]
To ensure a tractable and well-behaved curvature surrogate, $\boldsymbol{H}_i$ is commonly replaced by the Fisher information at $\boldsymbol{\tilde{W}}_i$:
\begin{equation*}
\boldsymbol{F}_i=\frac{1}{n}\sum_{j=1}^n\mathbb{E}_{\boldsymbol{y}}
\!\left[\nabla\log\mathbb{P}(\boldsymbol{y}|\boldsymbol{x}_{ij},\boldsymbol{W})
\nabla\log\mathbb{P}(\boldsymbol{y}|\boldsymbol{x}_{ij},\boldsymbol{W})^\top\right],
\end{equation*}
evaluated at $\boldsymbol{W}=\boldsymbol{\tilde{W}}_i$. $\boldsymbol{F}_i$ is symmetric and positive semi-definite (PSD). 
Computing, storing, and communicating full curvature matrices can be prohibitively expensive in large-scale distributed learning~\cite{chen2025adaptive}. 
In our setting, the full Fisher matrix $\boldsymbol F_i\in\mathbb R^{d\times d}$ incurs quadratic storage and communication costs. 
Our fusion formulation does not depend on a particular Fisher representation, so structured approximations can be adopted according to the available computation, storage, and communication resources. For a layer with weight matrix of size $o_\ell\times i_\ell$ and $d_\ell=o_\ell i_\ell$ parameters, Table~\ref{tab:fisher_complexity} summarizes representative per-client Fisher representations. A diagonal approximation retains only the diagonal curvature and has linear representation cost, whereas K-FAC preserves richer layer-wise curvature through Kronecker factors while avoiding the quadratic representation cost of a full Fisher matrix. Thus, the choice of Fisher approximation provides a trade-off between curvature fidelity and resource overhead.
\begin{table}[t]
\centering
\caption{Structural complexity of representative per-client Fisher representations for a layer with $d_\ell=o_\ell i_\ell$ parameters.}
\label{tab:fisher_complexity}
\begin{tabular}{c|c|c}
\hline
Structure & Storage / Communication & Curvature Structure \\
\hline
Full
& $O(d_\ell^2)$
& Full \\
Diagonal
& $O(d_\ell)$
& Diagonal only \\
K-FAC
& $O(i_\ell^2+o_\ell^2)$
& Kronecker-factorized \\
\hline
\end{tabular}
\end{table}
Combining these local posterior approximations yields a Fisher-weighted quadratic form for the global log-posterior:
\begin{equation*}
\begin{aligned}
    &\log\mathbb{P}(\boldsymbol{W}|\mathcal{D})\\
    &\approx
    \sum_{i=1}^M \log\mathbb{P}(\boldsymbol{\tilde{W}}_i|\mathcal{D}_i)
    -\frac{n}{2}\sum_{i=1}^M(\boldsymbol{W}-\boldsymbol{\tilde{W}}_i)^\top
    \boldsymbol{F}_i
    (\boldsymbol{W}-\boldsymbol{\tilde{W}}_i).
\end{aligned}
\end{equation*}
Maximizing this approximation is equivalent to solving
\begin{equation}\label{eq-fishermin}
\min_{\boldsymbol{W}\in\mathbb{R}^d}\sum_{i=1}^M(\boldsymbol{W}-\boldsymbol{\tilde{W}}_i)^\top \boldsymbol{F}_i(\boldsymbol{W}-\boldsymbol{\tilde{W}}_i).
\end{equation}
Since each $\boldsymbol{F}_i\succeq 0$, any minimizer satisfies the first-order condition
\begin{equation}\label{eq-solution}
\left(\sum_{i=1}^M\boldsymbol{F}_i\right)\boldsymbol{W}
=\sum_{i=1}^M\boldsymbol{F}_i\boldsymbol{\tilde{W}}_i.
\end{equation}
In over-parameterized regimes ($d\gg N$), the system in~\eqref{eq-solution} can admit multiple solutions. 
To select a unique and well-behaved estimator, an effective selection rule is to choose the feasible solution that stays closest to the collection of local optima~\cite{jhunjhunwala2024fedfisher}. This leads to the following constrained least-squares formulation
\begin{equation}\label{eq-fedfisher}
\min_{\boldsymbol{W}\in\mathbb{R}^d}\left\{\sum_{i=1}^M\|\boldsymbol{W}-\boldsymbol{\tilde{W}}_i\|_2^2 \;\text{s.t.}\;
\left(\sum_{i=1}^M\boldsymbol{F}_i\right)\boldsymbol{W}
=\sum_{i=1}^M\boldsymbol{F}_i\boldsymbol{\tilde{W}}_i\right\}.
\end{equation}
A solution to~\eqref{eq-fishermin} can be obtained by iterative optimization~\cite{jhunjhunwala2024fedfisher}. Starting from the average initialization $\boldsymbol{W}^{[0]}=\frac{1}{M}\sum_{i=1}^M\boldsymbol{\tilde{W}}_i$, gradient descent with step size $\eta_S \le 1/\lambda_{\max}$, where $\lambda_{\max}$ is the largest eigenvalue of $\sum_{i=1}^M\boldsymbol{F}_i$, generates iterates $\boldsymbol{W}^{[t]}
\gets
\boldsymbol{W}^{[t-1]}
-
\tilde{\eta}
\sum_{i=1}^M
\boldsymbol{F}_i
(\boldsymbol{W}^{[t-1]}-\boldsymbol{\tilde{W}}_i)$,
which converge asymptotically to a minimizer $\boldsymbol{W}^\star$. 
While such iterative fusion is suitable for one-time aggregation, it is poorly matched to repeated model editing. When a single client changes, the fusion objective changes accordingly, and the iterative procedure generally needs to be rerun to obtain the updated solution. Moreover, the previous optimization trajectory is not directly reusable for updating only the affected client contribution. This motivates the following projection formulation. Let
\begin{equation}
    \boldsymbol{\bar{W}} \coloneqq \frac{1}{M}\sum_{i=1}^M \boldsymbol{\tilde{W}}_i.
\end{equation}
Using the identity
\[
\sum_{i=1}^M\|\boldsymbol{W}-\boldsymbol{\tilde{W}}_i\|_2^2
=
M\|\boldsymbol{W}-\boldsymbol{\bar{W}}\|_2^2
+
\sum_{i=1}^M\|\boldsymbol{\bar{W}}-\boldsymbol{\tilde{W}}_i\|_2^2,
\]
problem~\eqref{eq-fedfisher} is equivalent to
\begin{equation}\label{eq-newfedfisher}
    \min_{\boldsymbol{W}\in\mathbb{R}^d}
    \left\{
        \|\boldsymbol{W}-\boldsymbol{\bar{W}}\|_2^2
        \;\text{s.t.}\;
        \boldsymbol{F}\boldsymbol{W} = \boldsymbol{b}
    \right\},
\end{equation}
where
\begin{equation}
\boldsymbol{F} \coloneqq \sum_{i=1}^M \boldsymbol{F}_i,
\qquad
\boldsymbol{b} \coloneqq \sum_{i=1}^M \boldsymbol{F}_i \boldsymbol{\tilde{W}}_i.
\end{equation}
Hence, the fusion state needed for subsequent editing is summarized by $(\bar{\boldsymbol W},\boldsymbol F,\boldsymbol b)$, each of which decomposes over clients. Under a single-client replacement, these statistics can be updated by subtracting the outgoing contribution and adding the incoming one, while all unchanged client contributions are reused.
Since each $\boldsymbol F_i\succeq 0$, $\boldsymbol b = \sum_i \boldsymbol F_i\widetilde{\boldsymbol W}_i \in \mathrm{Range}(\boldsymbol F)$, so the affine constraint $\boldsymbol F\boldsymbol W=\boldsymbol b$ in~\eqref{eq-newfedfisher} is feasible. Thus, solving~\eqref{eq-newfedfisher} reduces to computing the orthogonal projection of $\boldsymbol{\bar{W}}$ onto the affine set $\{\boldsymbol{W}:\boldsymbol{F}\boldsymbol{W}=\boldsymbol{b}\}$. This projection admits an explicit characterization via the KKT conditions. Consider the Lagrangian
\begin{equation}\label{eq-lagrangian}
\mathcal{L}(\boldsymbol{W},\boldsymbol{\Lambda})
=
\|\boldsymbol{W}-\boldsymbol{\bar{W}}\|_2^2
+
\boldsymbol{\Lambda}^\top(\boldsymbol{F}\boldsymbol{W}-\boldsymbol{b}),
\end{equation}
where $\boldsymbol{\Lambda}\in\mathbb{R}^d$ is the dual variable. Since the objective is strictly convex and the constraint is affine, the KKT conditions are necessary and sufficient for optimality. Taking derivatives of~\eqref{eq-lagrangian} yields the stationarity and primal feasibility conditions:
\begin{equation}\label{eq-kkt-stationarity}
\left\{
\begin{aligned}
2(\boldsymbol{W}-\boldsymbol{\bar{W}})+\boldsymbol{F}\boldsymbol{\Lambda} &= \boldsymbol{0},\\
\boldsymbol{F}\boldsymbol{W}-\boldsymbol{b} &= \boldsymbol{0}.
\end{aligned}
\right.
\end{equation}
From the first condition,
\begin{equation}\label{eq-W-from-lambda}
\boldsymbol{W}
=
\boldsymbol{\bar{W}}-\tfrac{1}{2}\boldsymbol{F}\boldsymbol{\Lambda}.
\end{equation}
Substituting into the constraint yields a linear system in the dual variable $\boldsymbol{\Lambda}$:
\begin{equation}\label{eq-lambda-system}
\boldsymbol{F}^2\boldsymbol{\Lambda}
=
2(\boldsymbol{F}\boldsymbol{\bar{W}}-\boldsymbol{b}).
\end{equation}
Therefore, model fusion reduces to solving the linear system~\eqref{eq-lambda-system}. The linear system in~\eqref{eq-lambda-system} remains consistent even when the aggregated Fisher matrix $\boldsymbol F$ is sparse, singular, or low-rank. Since $\boldsymbol b\in\mathrm{Range}(\boldsymbol F)$ and $\boldsymbol F\bar{\boldsymbol W}\in \mathrm{Range}(\boldsymbol F)$, we have $\boldsymbol F\bar{\boldsymbol W}-\boldsymbol b \in \mathrm{Range}(\boldsymbol F)$. Moreover, because $\boldsymbol F$ is symmetric and positive semidefinite, $\mathrm{Range}(\boldsymbol F^2) = \mathrm{Range}(\boldsymbol F)$. Therefore,~\eqref{eq-lambda-system} always admits a solution. Rank deficiency may make the dual solution $\boldsymbol\Lambda^\star$ non-unique. However, if $\boldsymbol\Lambda_1^\star$ and $\boldsymbol\Lambda_2^\star$ are two solutions, then $\boldsymbol F^2 (\boldsymbol\Lambda_1^\star-\boldsymbol\Lambda_2^\star)=0$. For symmetric positive semidefinite $\boldsymbol F$, $\mathrm{Null}(\boldsymbol F^2) = \mathrm{Null}(\boldsymbol F)$, and hence $\boldsymbol F (\boldsymbol\Lambda_1^\star-\boldsymbol\Lambda_2^\star)=0$. Thus, all dual solutions yield the same $\boldsymbol F\boldsymbol\Lambda^\star$ and therefore the same projected primal solution $\boldsymbol W^\star = \bar{\boldsymbol W} -\frac{1}{2}\boldsymbol F\boldsymbol\Lambda^\star$. Hence, rank deficiency may affect numerical conditioning and dual uniqueness, but not the existence or uniqueness of the projected primal solution.
Once a solution $\boldsymbol{\Lambda}^\star$ is obtained, the optimal model follows directly from~\eqref{eq-W-from-lambda}:
\begin{equation}\label{eq-projection-solution}
\boldsymbol{W}^\star
=
\boldsymbol{\bar{W}}-\tfrac{1}{2}\boldsymbol{F}\boldsymbol{\Lambda}^\star.
\end{equation}
This formulation makes clear that the fusion problem decomposes into two steps: solving a linear system defined by aggregated curvature statistics, and then recovering the projected model in closed form.

The linear system~\eqref{eq-lambda-system} in $\boldsymbol{\Lambda}$ can be solved efficiently and accurately by standard linear solvers. For example, conjugate gradient (CG) applies here because $\boldsymbol F^2$ is symmetric positive semidefinite and the linear system~\eqref{eq-lambda-system} is always consistent. In particular, CG minimizes the quadratic energy
\begin{equation}\label{eq-cg-energy}
\min_{\boldsymbol{\Lambda}}
\ \tfrac{1}{2}\boldsymbol{\Lambda}^\top\boldsymbol{F}^2\boldsymbol{\Lambda}
-
\boldsymbol{\Lambda}^\top 2(\boldsymbol{F}\boldsymbol{\bar{W}}-\boldsymbol{b}),
\end{equation}
whose first-order optimality condition is exactly~\eqref{eq-lambda-system}. CG terminates in at most $r=\mathrm{rank}(\boldsymbol{F})$ iterations, which is typically much smaller than $d$ in over-parameterized settings. This follows from the fact that CG minimizes the quadratic energy over an expanding Krylov subspaces and  the Krylov subspace associated with the nonzero eigenspace has dimension
at most $r=\operatorname{rank}(\boldsymbol F)$. 
The practical cost of the projection step also depends on the Fisher representation. A diagonal approximation incurs the lowest structural overhead, while K-FAC retains richer curvature information and substantially reduces representation and communication costs relative to a full Fisher matrix. However, a sum of client-wise K-FAC matrices does not generally preserve a single Kronecker factorization. Since our coordinator stores only the aggregated curvature statistic rather than individual K-FAC factors, this structure may be lost after aggregation, and the coordinator-side projection cost can approach that of a less structured representation. Hence, the Fisher approximation can be selected according to the desired trade-off among curvature fidelity, communication overhead, and coordinator-side computation.

To enable efficient exact model editing, the one-shot FL algorithm must be $\rho$-TV-stable (Definition~\ref{def-tvstability}) for a pre-specified stability level $\rho$. The fusion rule above uses all available clients and is deterministic. This benefits model utility, but it yields no distributional stability. In particular, for any neighboring device sets $\mathcal{M}$ and $\mathcal{M}^\prime$ with $|\mathcal{M}\oplus\mathcal{M}^\prime|=1$, the outputs $\mathcal{A}_1(\mathcal{M})$ and $\mathcal{A}_1(\mathcal{M}^\prime)$ are Dirac measures supported at different points. Hence, $\TV(\mathcal{A}_1(\mathcal{M}),\mathcal{A}_1(\mathcal{M}^\prime))=1$, which cannot meet an arbitrary target $\rho<1$. We address this by introducing randomized client subsampling for model aggregation. Let \(M_0\) denote the number of clients at initialization. We set \(P=\lfloor\rho M_0\rfloor\) once and retain this subset size throughout the editing sequence. The coordinator samples a client subset $\mathcal P \subseteq [M]$ of size $P$ for local training and model fusion. This shared randomness induces overlap between the output distributions under neighboring client sets, which is directly controlled by the sampling rate $\rho$. 
Note that, for the general dynamic-device learning setting considered in this work, we assume that client contributions have comparable aggregation priority, under which uniform random subsampling is appropriate. 

Combining Fisher-informed projection with client subsampling yields our $\rho$-TV-stable one-shot FL algorithm, \textsc{F-SOSA} (Fisher-Informed Subsample One-Shot Averaging). 
Algorithm~\ref{algo-FSOSA} summarizes the procedure. The coordinator samples $\mathcal{P}$ (step~1) and broadcasts the initialization $\boldsymbol{W}^{(0)}$ (step~2). Here, $P=\rho M$ is determined at the initial execution of F-SOSA
and remains fixed during all subsequent editing operations. Each client $i\in\mathcal{P}$ runs $K$ steps of local gradient descent to obtain $\boldsymbol{\tilde{W}}_i$ (steps~3--7), computes its Fisher matrix $\boldsymbol{F}_i$ at $\boldsymbol{\tilde{W}}_i$ (step~8), and uploads $(\boldsymbol{\tilde{W}}_i,\boldsymbol{F}_i)$ to the coordinator (step~9). The coordinator then forms the sufficient statistics
\[
\boldsymbol{\bar{W}}_{\mathcal{P}}=\frac{1}{P}\sum_{i\in\mathcal{P}}\boldsymbol{\tilde{W}}_i,\quad
\boldsymbol{F}_{\mathcal{P}}=\sum_{i\in\mathcal{P}}\boldsymbol{F}_i,\quad
\boldsymbol{b}_{\mathcal{P}}=\sum_{i\in\mathcal{P}}\boldsymbol{F}_i\boldsymbol{\tilde{W}}_i
\]
as shown in step~10. Finally, it outputs the global model by projecting $\boldsymbol{\bar{W}}_{\mathcal{P}}$ onto the affine set $\{\boldsymbol{W}:\boldsymbol{F}_{\mathcal{P}}\boldsymbol{W}=\boldsymbol{b}_{\mathcal{P}}\}$. This projection is computed by solving the linear system for the dual variable $\boldsymbol{\Lambda}$(steps~11) and then recovering the final fused model $\boldsymbol{\bar{W}}_{\mathcal{P}}^{\text{F}}$ (step~12).


\begin{algorithm}[t]
\caption{\texttt{F-SOSA-U}}
\label{algo-FSOSAU}
\KwIn{Unlearning request from client $i^-$ with its local contribution $(\widetilde W_{i^-},F_{i^-})$, initial model $\boldsymbol{W}^{(0)}$, number of local iterations $K$, learning rate $\eta$}

\If{$i^-\in\mathcal{P}$}{
    \textbf{Coordinator} samples $i^+\sim \operatorname{Uniform}([M]\backslash \mathcal{P})$\;
    \textbf{Coordinator} updates $\mathcal{P}\gets\mathcal{P}\backslash\{i^-\}\cup\{i^+\}$\;
    \textbf{Coordinator} communicates $\boldsymbol{W}^{(0)}$ to client $i^+$\;

    \textbf{Client $i^+$} performs local training and Fisher computation following
    Lines 4--9 of Algorithm~\ref{algo-FSOSA}, and sends $(\boldsymbol{\tilde{W}}_{i^+},\boldsymbol{F}_{i^+})$ to the coordinator\;

    \textbf{Coordinator} updates $\boldsymbol{\bar{W}}_{\mathcal{P}}$, $\boldsymbol{F}_{\mathcal{P}}$ and $\boldsymbol{b}_{\mathcal{P}}$:
    \[
    \boldsymbol{\bar{W}}_{\mathcal{P}}\gets\boldsymbol{\bar{W}}_{\mathcal{P}}+\tfrac{1}{P}\big(\boldsymbol{\tilde{W}}_{i^+}-\boldsymbol{\tilde{W}}_{i^-}\big),
    \]
    \[
    \boldsymbol{F}_{\mathcal{P}}\gets\boldsymbol{F}_{\mathcal{P}}+\big(\boldsymbol{F}_{i^+}-\boldsymbol{F}_{i^-}\big),
    \]
    \[
    \boldsymbol{b}_{\mathcal{P}}\gets\boldsymbol{b}_{\mathcal{P}}+\big(\boldsymbol{F}_{i^+}\boldsymbol{\tilde{W}}_{i^+}-\boldsymbol{F}_{i^-}\boldsymbol{\tilde{W}}_{i^-}\big) ;
    \]

    \textbf{Coordinator} solves the linear system for $\boldsymbol{\Lambda}\in\mathbb{R}^d$:
    $\boldsymbol{F}_{\mathcal{P}}^{2}\boldsymbol{\Lambda}
    =2\big(\boldsymbol{F}_{\mathcal{P}}\boldsymbol{\bar{W}}_{\mathcal{P}}-\boldsymbol{b}_{\mathcal{P}}\big)$\;

    \textbf{Coordinator} sets the updated global model:
    $\boldsymbol{\bar{W}}_{\mathcal{P}}^{\text{F}}
    \gets
    \boldsymbol{\bar{W}}_{\mathcal{P}}-\tfrac{1}{2}\boldsymbol{F}_{\mathcal{P}}\boldsymbol{\Lambda}$\;
}

\KwOut{$\boldsymbol{\bar{W}}_{\mathcal{P}}^{\text{F}}$}
\end{algorithm}

\begin{algorithm}[t]
\caption{\texttt{F-SOSA-I}}
\label{algo-FSOSAI}

\KwIn{new client index $i^+$}

\If{$\operatorname{Bernoulli}(\frac{P}{M+1})$}{
    \textbf{Coordinator} samples $i^-\sim\operatorname{Uniform}(\mathcal{P})$\;
    \textbf{Coordinator} requests $(\widetilde{\boldsymbol W}_{i^-},\boldsymbol F_{i^-})$ from client $i^-$ and $(\widetilde{\boldsymbol W}_{i^+},
    \boldsymbol F_{i^+})$ from client $i^+$\;
    \textbf{Coordinator} updates $\mathcal{P}\gets\mathcal{P}\backslash\{i^-\}\cup\{i^+\}$\;

    \textbf{Coordinator} updates $\boldsymbol{\bar{W}}_{\mathcal{P}}$, $\boldsymbol{F}_{\mathcal{P}}$ and $\boldsymbol{b}_{\mathcal{P}}$:
    \[
    \boldsymbol{\bar{W}}_{\mathcal{P}}
    \gets
    \boldsymbol{\bar{W}}_{\mathcal{P}}+\frac{1}{P}\big(\boldsymbol{\tilde{W}}_{i^+}-\boldsymbol{\tilde{W}}_{i^-}\big),
    \]
    \[\boldsymbol{F}_{\mathcal{P}}
    \gets
    \boldsymbol{F}_{\mathcal{P}}+\big(\boldsymbol{F}_{i^+}-\boldsymbol{F}_{i^-}\big),
    \]
    \[\boldsymbol{b}_{\mathcal{P}}
    \gets
    \boldsymbol{b}_{\mathcal{P}}+\big(\boldsymbol{F}_{i^+}\boldsymbol{\tilde{W}}_{i^+}-\boldsymbol{F}_{i^-}\boldsymbol{\tilde{W}}_{i^-}\big);\]

    \textbf{Coordinator} solves the linear system for $\boldsymbol{\Lambda}\in\mathbb{R}^d$:
    $\boldsymbol{F}_{\mathcal{P}}^{2}\boldsymbol{\Lambda}
    =2\big(\boldsymbol{F}_{\mathcal{P}}\boldsymbol{\bar{W}}_{\mathcal{P}}-\boldsymbol{b}_{\mathcal{P}}\big)$\;

    \textbf{Coordinator} updates the global model:
    $\boldsymbol{\bar{W}}_{\mathcal{P}}^{\text{F}}
    \gets
    \boldsymbol{\bar{W}}_{\mathcal{P}}-\tfrac{1}{2}\boldsymbol{F}_{\mathcal{P}}\boldsymbol{\Lambda}$\;
}

\KwOut{$\boldsymbol{\bar{W}}_{\mathcal{P}}^{\text{F}}$}
\end{algorithm}

\subsection{Efficient Editing Algorithms: F-SOSA-U and F-SOSA-I}
\label{sec:editing}

For a fixed subset size $P$, F-SOSA aggregates a uniformly sampled client subset $\mathcal P$. When the client population changes, the support of the uniform size-$P$ subset distribution changes accordingly. Exact model editing therefore reduces to constructing a coupling between these sampling distributions while retaining $P$.
For a current population of $M$ clients, the deletion coupling retains the subset when the departing client is absent and otherwise replaces that client. The insertion coupling includes the new client with probability $P/(M+1)$ by replacing a uniformly selected member; otherwise, it retains the subset. These transitions are maximal couplings of the respective subset distributions, with trigger probabilities $P/M$ and $P/(M+1)$. Thus, the trigger probabilities depend on the current population even though $P$ remains fixed. Whenever the subset is retained, its contributions, aggregate statistics, and fused model can be reused. We develop the corresponding procedures below:
\textsc{F-SOSA-U} for client unlearning (Algorithm~\ref{algo-FSOSAU}) and 
\textsc{F-SOSA-I} for client incremental learning (Algorithm~\ref{algo-FSOSAI}).
The coordinator persistently maintains only the sampled subset $\mathcal P$ and the aggregate statistics $(\bar{\boldsymbol W}_{\mathcal P}, \boldsymbol F_{\mathcal P}, \boldsymbol b_{\mathcal P})$. Individual client contributions $(\widetilde{\boldsymbol W}_i,\boldsymbol F_i)$ need not be retained at the coordinator. While connected, each active client keeps its own contribution and transmits it only when required by an editing operation; in particular, a departing client sends its contribution together with the unlearning request before disconnecting. The coordinator then updates the aggregate statistics by subtracting the departing contribution and adding the incoming one. Hence, its persistent state does not grow with the number of clients that have participated historically.

\begin{itemize}[leftmargin=*]
    \item \noindent\textbf{Client unlearning (F-SOSA-U).}
        When the coordinator receives an unlearning request from client $i^-$, the coordinator first checks whether it belongs to the sampled subset $\mathcal{P}$. If $i^- \notin \mathcal{P}$, then under the maximal coupling the subset remains unchanged, and no editing is needed. Otherwise, we must replace the removed client to preserve the subset size and sampling distribution. The coordinator therefore draws a replacement client $i^+$ uniformly from $[M]\setminus \mathcal{P}$ and updates 
        $\mathcal{P}\gets\mathcal{P}\setminus\{i^-\}\cup\{i^+\}$. 
        The new client performs local training and Fisher computation, after which the coordinator updates the sufficient statistics via client-additive adjustments:
        $\boldsymbol{\bar{W}}_{\mathcal{P}},
        \boldsymbol{F}_{\mathcal{P}},
        \boldsymbol{b}_{\mathcal{P}}$.
        The edited global model is then obtained by solving the same projection system as in F-SOSA. Since only one client contribution changes, the update cost is independent of the total number of clients.
    \item \noindent\textbf{Client incremental learning (F-SOSA-I).}
        When a new client $i^+$ joins, we determine whether it would have been selected under the updated F-SOSA sampling distribution. Conditional on the current client set in the learning objective containing $M+1$ clients, F-SOSA samples a size-$P$ subset uniformly at random. Hence, each client in the updated set, including the newly arrived client $i^+$, has inclusion probability $\Pr(i^+\in\mathcal P')=P/(M+1)$. This probability is determined by the conditional subset-sampling rule and is independent of the external arrival distribution. If the Bernoulli trial fails, the sampled subset remains unchanged under the coupling. Otherwise, the coordinator replaces a uniformly chosen client $i^-\in\mathcal P$ with $i^+$ and updates the sufficient statistics using the same adjustments as in unlearning.
\end{itemize}
Both editing algorithms modify only a single selected-client contribution
while reusing the current aggregate fusion statistics. Consequently, each triggered editing operation updates only one client contribution to the sufficient statistics and performs one projection solve. Its coordinator-side cost depends on the chosen Fisher representation, as discussed above, but does not require recomputing the contributions of all previously sampled clients. Moreover, since editing is triggered with probability at most $O(\rho)$, the expected editing cost per client change is $O(\rho)$ relative to full retraining.
Together, these procedures enable provably exact and communication-efficient model adaptation to dynamic client participation in one-shot federated learning. For sequential requests, F-SOSA-U/I update the model after each request, preserving fixed-\(P\) exactness by induction.

Finally, we consider simultaneous requests. A batch specifies deletions $\mathcal D\subseteq\mathcal M$ and additions $\mathcal A\cap\mathcal M=\varnothing$, independently of the sampling randomness. Let $M=|\mathcal M|$, $d=|\mathcal D|$, $a=|\mathcal A|$, $\mathcal M'=(\mathcal M\setminus\mathcal D)\cup\mathcal A$, and $M'=M-d+a$, with $1\le P\le\min\{M,M'\}$. We condition on the shared initialization and fixed per-client training randomness, independent of subset sampling. Each identifier denotes an immutable local contribution under this configuration; a changed contribution uses a new identifier.
Algorithm~\ref{algo-FSOSABatch} pairs $\min\{d,a\}$ departing and arriving clients independently of the realized subset. Each selected departing identifier is replaced by its paired arriving identifier. Unmatched arrivals are handled by the fixed-$P$ insertion coupling; unmatched departures are removed together and the subset is uniformly replenished. Pairing operates only on identifiers: each added client supplies its own model and Fisher information. Once the final subset $\mathcal P'$ is obtained, the coordinator applies only the net changes $S_-=\mathcal P\setminus\mathcal P'$ and $S_+=\mathcal P'\setminus\mathcal P$ to the original aggregate statistics. Departing selected clients provide their original contributions before disconnecting, with temporary buffering until batch completion. The model is reused if the final subset is unchanged; otherwise, one projection is performed. No intermediate model is released.
Proposition~\ref{prop:paired-batch} establishes exactness and shows that this coupling minimizes both the probability of a subset change and the expected number of replaced members.

\begin{algorithm}[t]
\small
\caption{\texttt{F-SOSA-Batch}: Paired Joint Editing}
\label{algo-FSOSABatch}
\KwIn{Current client set $\mathcal M$, size-$P$ subset $\mathcal P$, aggregate statistics and fused model; deletion set $\mathcal D$ and addition set $\mathcal A$}
Set $\mathcal M'\gets(\mathcal M\setminus\mathcal D)\cup\mathcal A$\;
Buffer original contributions of $\mathcal P\cap\mathcal D$ before their clients disconnect\;
Choose $\mathcal D_0\subseteq\mathcal D$, $\mathcal A_0\subseteq\mathcal A$ of size $\min\{|\mathcal D|,|\mathcal A|\}$ and a bijection $\pi:\mathcal D_0\to\mathcal A_0$, independently of $\mathcal P$\;
Set $\widehat{\mathcal M}\gets(\mathcal M\setminus\mathcal D_0)\cup\mathcal A_0$ and $\mathcal P'\gets(\mathcal P\setminus\mathcal D_0)\cup\pi(\mathcal P\cap\mathcal D_0)$\;
\uIf{$|\mathcal A|>|\mathcal D|$}{
    \ForEach{$i\in\mathcal A\setminus\mathcal A_0$ in a fixed order}{
        \If{$\operatorname{Bernoulli}(P/(|\widehat{\mathcal M}|+1))=1$}{
            Sample $j\sim\operatorname{Uniform}(\mathcal P')$ and set $\mathcal P'\gets(\mathcal P'\setminus\{j\})\cup\{i\}$\;
        }
        $\widehat{\mathcal M}\gets\widehat{\mathcal M}\cup\{i\}$\;
    }
}\Else{
    $\mathcal P'\gets\mathcal P'\setminus(\mathcal D\setminus\mathcal D_0)$\;
    Sample $\mathcal B$ uniformly among size-$(P-|\mathcal P'|)$ subsets of $\mathcal M'\setminus\mathcal P'$; set $\mathcal P'\gets\mathcal P'\cup\mathcal B$\;
}
Set $S_-\gets\mathcal P\setminus\mathcal P'$ and $S_+\gets\mathcal P'\setminus\mathcal P$\;
\If{$S_-\cup S_+\ne\varnothing$}{
    Retrieve original contributions for $S_-$ from the buffer or connected clients; obtain contributions for $S_+$ using the fixed F-SOSA training configuration\;
    From the unchanged original statistics, compute
    \[
    \begin{aligned}
    \bar{\boldsymbol W}'&\gets\bar{\boldsymbol W}_{\mathcal P}
      +\frac{1}{P}\left(\sum_{i\in S_+}\widetilde{\boldsymbol W}_i-\sum_{i\in S_-}\widetilde{\boldsymbol W}_i\right),\\
    \boldsymbol F'&\gets\boldsymbol F_{\mathcal P}
      +\sum_{i\in S_+}\boldsymbol F_i-\sum_{i\in S_-}\boldsymbol F_i,\\
    \boldsymbol b'&\gets\boldsymbol b_{\mathcal P}
      +\sum_{i\in S_+}\boldsymbol F_i\widetilde{\boldsymbol W}_i
      -\sum_{i\in S_-}\boldsymbol F_i\widetilde{\boldsymbol W}_i.
    \end{aligned}
    \]
    Solve $(\boldsymbol F')^2\boldsymbol\Lambda=2(\boldsymbol F'\bar{\boldsymbol W}'-\boldsymbol b')$ and set $\bar{\boldsymbol W}^{\prime\mathrm F}\gets\bar{\boldsymbol W}'-\tfrac12\boldsymbol F'\boldsymbol\Lambda$\;
}\Else{
    Reuse the original statistics and fused model\;
}
Commit $\mathcal M\gets\mathcal M'$, $\mathcal P\gets\mathcal P'$, the statistics and model; discard temporary per-client buffers\;
\KwOut{Updated model, client set, sampled subset, and aggregate statistics}
\end{algorithm}

\section{Theoretical Results}

\subsection{Stability of F-SOSA}

We consider a general device set $\mathcal{M}$ with $|\mathcal{M}|=M$, and note that all results apply to any round $t$ with device set $\mathcal{M}(t)$. We first prove that F-SOSA is $\rho$-TV-stable for any given stability level $\rho$, and then analyze its learning utility. Our stability analysis uses the concept of the push-forward measure: for a measurable map $f:\mathcal{X}\to\mathcal{Y}$ and a measure $\mu$ on $\mathcal{X}$, the push-forward $f\#\mu$ is defined by $(f\#\mu)(\mathcal{Z}) = \mu(f^{-1}(\mathcal{Z}))$ for any measurable set $\mathcal{Z} \subseteq \mathcal{Y}$.

\begin{theorem}\label{th-stability}
    F-SOSA is $\rho$-TV-stable for any $\rho \in (0, 1)$.
\end{theorem}

\begin{proof}
    Let $\mathcal{M}$ be original device set and $\mathcal{M}^\prime$ be the update device set such that $|\mathcal{M}\oplus\mathcal{M}^\prime|=1$. We first consider $\mathcal{M}$ contains $M-1$ devices, where w.l.o.g. we assume device $M$ leaves, i.e., $\mathcal{M}=\{1,2,\cdots,M\}$ and $\mathcal{M}^\prime=\{1,2,\cdots,M-1\}$. Let the sigma-algebra on these sets be the power sets of $\mathcal{M}$ and $\mathcal{M}^\prime$ respectively, denoted by $\operatorname{Pow}(\mathcal{M})$ and $\operatorname{Pow}(\mathcal{M}^\prime)$ respectively. Let $\mu_{(P,M)}$ be the sub-sampling probability measure on $M$ devices in $\mathcal{M}$, i.e., sub-samples a subset $\mathcal{P}$ of size $P$ out of $M$ elements in $\mathcal{M}$ uniformly at random. Similarly we can define $\mu_{(P,M-1)}$ for $\mathcal{M}^\prime$. We first extend the sigma-algebra for the probability spaces so that the random variables $\mu_{(P,M)}$ and $\mu_{(P,M-1)}$ are defined on a common probability space. For this, we will just add an event where the device $M$ can be sampled under $\mu_{(P,M-1)}$ with probability $0$. We define $\mu_{(P,M)}^\prime$ as follows: for any set $\mathcal{P}\in\operatorname{Pow}(\mathcal{M})$, 
    $\mu_{(P,M)}^\prime(\mathcal{P}) \coloneqq
    \begin{cases}
        \mu_{(P,M-1)}(\mathcal{P}), & \text{if } M \notin \mathcal{P} \\
        0,         & \text{otherwise}
    \end{cases}$.
    Note that, for fixed initialization $\boldsymbol{W}^{(0)}$ and other parameters, Algorithm $\mathcal{A}(\mathcal{M})$ and $\mathcal{A}(\mathcal{M}^\prime)$ is the same deterministic map from $\mathcal{P}\in[M]^P$ to $\mathcal{W}$. They only differ because of different measure on the input space. Hence TV distance between $\mathcal{A}(\mathcal{M})$ and $\mathcal{A}(\mathcal{M}^\prime)$ is just the TV distance between the push-forward measures $\mathcal{A}(\mathcal{M})\#\mu_{(P,M)}$ and $\mathcal{A}(\mathcal{M}^\prime)\#\mu_{(P,M)}^\prime$, which by using the fact that $\mathcal{A}(\mathcal{M})\equiv\mathcal{A}(\mathcal{M}^\prime)$ and data processing inequality is at most the TV distance between $\mu_{(P,M)}$ and $\mu_{(P,M)}^\prime$. Therefore, 
    \begin{equation*}
        \begin{aligned}
            &\TV(\mathcal{A}(\mathcal{M}), \mathcal{A}(\mathcal{M}^\prime))
            \le\TV(\mu_{(P,M)}, \mu_{(P,M)}^\prime)\\
            &\le\mu_{(P,M)}(S \text{ containing } M)
            =\binom{M-1}{P-1}/\binom{M}{P}=P/M.
        \end{aligned}
    \end{equation*}
    A similar argument works when $\mathcal{M}^\prime=\{1,2,\cdots,M,M+1\}$, yielding a TV bound of $\frac{P}{M+1}\le\frac{P}{M}$. Taking a uniform bound over all neighbouring set $\mathcal{M}^\prime$, we get that $\sup_{|\mathcal{M}\oplus\mathcal{M}^\prime|=1}\TV(\mathcal{A}(\mathcal{M}), \mathcal{A}(\mathcal{M}^\prime))\le P/M$. Since in F-SOSA, $P$ is set to $\rho M$, thus $\TV(\mathcal{A}(\mathcal{M}),  \mathcal{A}(\mathcal{M}^\prime))\le\rho$. 
\end{proof}

\subsection{Learning Utility of F-SOSA}

We analyze the utility guarantee of $\boldsymbol{\bar{W}}_{\mathcal{P}}^{\text{F}}$ obtained by  F-SOSA under the following standard data assumptions:

\begin{assumption}[Data Normalization]\label{asp-normalization}
For any $(\boldsymbol{x}, y) \in \mathcal{D}$, $|\boldsymbol{x}|_2 = 1$ and $|y| \leq C$ for some constant $C > 0$.
\end{assumption}

\begin{assumption}[Data Uniqueness]\label{asp-uniqueness}
Any two samples $(\boldsymbol{x}, y)$ and $(\boldsymbol{x}', y')$ in $\mathcal{D}$ satisfy $|\boldsymbol{x} - \boldsymbol{x}'|_2 \ge \phi > 0$.
\end{assumption}

\begin{assumption}[Data Non-Degeneracy]\label{asp-degen}
The minimum eigenvalue of the Gram matrix $\mathbf{H}^{\infty}\!\in\!\mathbb{R}^{N \times N}$ defined by $(\mathbf{H}^{\infty})_{k,l}\!=\!\mathbb{E}_{\boldsymbol{w} \sim \mathcal{N}(\boldsymbol{0},\mathbf{I})}[\boldsymbol{x}_{k}^{\top}\boldsymbol{x}_{l}\mathds{1}\{ \boldsymbol{w}^{\top}\boldsymbol{x}_{k} \!\geq\! 0\}\!\mathds{1}\{\boldsymbol{w}^{\top}\boldsymbol{x}_{l} \!\geq\! 0\} ]$, is bounded below: $\lambda_{\min}(\mathbf{H}^\infty) = \lambda_0 > 0$.
\end{assumption}
No bounded heterogeneity assumption is made, so our results hold under arbitrary non-i.i.d. client data. 
We consider a two-layer ReLU neural network (NN) where the first-layer consists of $m$ trainable weights $\{\boldsymbol{w}_r\}_{r=1}^m$ and the second-layer weights $\boldsymbol{a} = \{a_r\}_{r=1}^m$ are drawn i.i.d. from $\{+1, -1\}$ at random and then fixed~\cite{arora2019fine,jhunjhunwala2024fedfisher}. Let $\boldsymbol{W} = \operatorname{vec}(\boldsymbol{w}_1,\boldsymbol{w}_2,\dots,\boldsymbol{w}_m)$.  The NN is modeled as 
\begin{equation}
    f_{\boldsymbol{a}}(\boldsymbol{W},\boldsymbol{x})=\frac{1}{\sqrt{m}}\sum_{r=1}^m a_r\gamma(\boldsymbol{x}^\top\boldsymbol{w}_r),
\end{equation}
where $\gamma(x) = \max\{x, 0\}$ is the ReLU activation function. This model can be equivalently written in a linearized form: 
\begin{equation}
    f_{\boldsymbol{a}}(\boldsymbol{W},\!\boldsymbol{x})\!=\!\phi_{\boldsymbol{a}}(\boldsymbol{W},\! \boldsymbol{x})^\top\!\boldsymbol{W},
\end{equation}
where 
\begin{equation}
\phi_{\boldsymbol{a}}(\boldsymbol{W}, \boldsymbol{x})\coloneqq\frac{1}{\sqrt{m}}
\begin{bmatrix}
    a_1 \boldsymbol{x} \mathds{1} \{ \boldsymbol{x}^\top \boldsymbol{w}_1 \geq 0 \} \\
    a_2 \boldsymbol{x} \mathds{1} \{ \boldsymbol{x}^\top \boldsymbol{w}_2 \geq 0 \} \\
    \vdots \\
    a_m \boldsymbol{x} \mathds{1} \{ \boldsymbol{x}^\top \boldsymbol{w}_m \geq 0 \}
\end{bmatrix}\in\mathbb{R}^{mp}.
\end{equation}
The neural-network objective is non-convex. We analyze it using the two-layer ReLU model, a standard canonical setting for NTK-style analysis of over-parameterized neural networks, where the training dynamics can be characterized through activation stability and the associated Gram matrix. In this regime, gradient descent remains within a controlled neighborhood of the random initialization, the relevant Gram matrices remain well conditioned, and the empirical fitting error decreases geometrically with the number of local iterations.

We first record a basic bound on the feature representation that will be used repeatedly in the subsequent analysis. It follows directly from the definition
of $\phi_{\boldsymbol{a}}(\boldsymbol{W}, \boldsymbol{x})$, since $a_r\in\{-1,+1\}$ and each activation indicator is at
most one.

\begin{lemma}\label{le-bounded_phi}
    For $\forall\boldsymbol{W}\in\mathbb{R}^d$ and $\|x\|_2\le 1$, $\|\phi_{\boldsymbol{a}}(\boldsymbol{W}, \boldsymbol{x})\|_2\le 1$.
\end{lemma}

We initialize $\boldsymbol{w}_r^{(0)} \overset{\text{i.i.d.}}{\sim} \mathcal{N}(\mathbf{0}, \sigma^2 \mathbf{I})$. To quantify activation pattern stability, given a real $R\ge 0$, define event $\mathcal{E}_{ijr}(R) \coloneqq \{ \exists \boldsymbol{w} : \|\boldsymbol{w} - \boldsymbol{w}^{(0)}_{r}\|_2 \le R, \mathds{1}\{\boldsymbol{x}_{ij}^\top\boldsymbol{w}^{(0)}_{r} \ge 0\} \neq \mathds{1}\{\boldsymbol{x}_{ij}^\top\boldsymbol{w} \ge 0\}\}.$ We next characterize the stability of ReLU activation patterns around the random initialization.

\begin{lemma}[{\cite[Lemma~3]{jhunjhunwala2024fedfisher}}]\label{le-learnability}
The random initialization of $\boldsymbol{w}_r^{(0)} \overset{\text{i.i.d.}}{\sim} \mathcal{N}(\mathbf{0}, \sigma^2 \mathbf{I})$ ensures that $\mathbb{P}(\mathcal{E}_{ijr}(R))\le 2R/(\sqrt{2\pi}\sigma)$.
\end{lemma}

We then invoke a standard over-parameterized training result to control both the local fitting error and the distance traveled by each local model.

\begin{lemma}[{\cite[Lemma~2]{jhunjhunwala2024fedfisher}}]\label{le-loc-err}
    Let $m = \Omega(\frac{N^6}{\lambda_0^4 \sigma^2 \delta^3})$ and $\eta = O(\frac{\lambda_0}{N^2})$. With probability at least $1 - \delta$ over the initialization of $\boldsymbol{W}^{(0)}$ and $\boldsymbol{a}$, we have simultaneously for all $i \in [M]$: 
    \begin{enumerate}
        \item $\sum_{j=1}^n(y_{ij}-f_{\boldsymbol{a}}(\boldsymbol{W}^{(0)},\boldsymbol{x}_{ij}))^2=O(\frac{N}{\delta})$; 
        \item $\sum_{j=1}^n(y_{ij}-f_{\boldsymbol{a}}(\boldsymbol{\tilde{W}}_i, \boldsymbol{x}_{ij}))^2\le(1-\frac{\eta\lambda_0}{2})^K\cdot O(\frac{N}{\delta})$; 
        \item For $\forall r\in[m]$, $\|\boldsymbol{\tilde{w}}_{i,r}-\boldsymbol{w}^{(0)}_r\|_2\le R_0$, where $R_0\!=\!\frac{4}{\lambda_0}\!\sqrt{\frac{N}{m}}\!(\!1\!-\!(\!1\!-\!\frac{\eta\lambda_0}{4}\!)^K\!)\!\sqrt{\sum_{j=1}^n(\!y_{ij}\!-\!f_{\boldsymbol{a}}\!(\!\boldsymbol{W}^{(0)},\boldsymbol{x}_{ij}\!)\!)^2}$.
    \end{enumerate} 
    Combining 1) and 3), we have $R_0\le O(\frac{N}{\lambda_0\sqrt{\delta m}})$.
\end{lemma}

Building on these local training properties, we next bound the deviations of the trained local models and their subsampled average from the common initialization. Proofs of all auxiliary lemmas established in this paper are deferred to Appendix~\ref{app:aux_proofs}.

\begin{lemma}\label{le-model-dist}
Let $m = \frac{\operatorname{poly}(N, \delta^{-1}, \lambda_0^{-1})}{\sigma^2}$, $\eta = O(\frac{\lambda_0}{N^2})$, and $\sigma = O(\frac{\lambda_0 \delta}{N})$. With probability at least $1 - \delta$ over random initialization: 
    \begin{enumerate}
        \item for all $i \in [M]$, $\|\boldsymbol{\tilde{W}}_i - \boldsymbol{W}^{(0)}\|_2 \le \sqrt{\boldsymbol{y}_i^\top (\mathbf{H}_i^{\infty})^{-1} \boldsymbol{y}_i} + O(\frac{N \sigma}{\lambda_0 \delta}) + \frac{\operatorname{poly}(N, \lambda_0^{-1}, \delta^{-1})}{m^{1/4} \sigma^{1/2}}$;
        \item $\|\boldsymbol{\bar{W}}_\mathcal{P} - \boldsymbol{W}^{(0)}\|_2 \le \sqrt{\frac{2\boldsymbol{y}^\top(\mathbf{H}^\infty)^{-1}\boldsymbol{y}}{\rho}} + O(\frac{N \sigma}{\lambda_0 \delta}) + \frac{\operatorname{poly}(N, \lambda_0^{-1}, \delta^{-1})}{m^{1/4} \sigma^{1/2}}$.
    \end{enumerate}
\end{lemma}

These parameter-space bounds further allow us to control the prediction discrepancy between the averaged model and the collection of local models.
Let $\boldsymbol{y} = \operatorname{vec}(\boldsymbol{y}_1, \boldsymbol{y}_2, \cdots, \boldsymbol{y}_M)\in\mathbb{R}^N$ be the true label vector and $\boldsymbol{\tilde{y}} = \operatorname{vec}(\boldsymbol{\tilde{y}}_1, \boldsymbol{\tilde{y}}_2, \cdots, \boldsymbol{\tilde{y}}_M) = [f_{\boldsymbol{a}}(\boldsymbol{\tilde{W}}_1, \boldsymbol{x}_{11}),\ f_{\boldsymbol{a}}(\boldsymbol{\tilde{W}}_1, \boldsymbol{x}_{12}),\cdots, f_{\boldsymbol{a}}(\boldsymbol{\tilde{W}}_M, \boldsymbol{x}_{Mn})]^\top \in \mathbb{R}^N$ be the predictions of trained local models. We define the proxy output $\tilde{f}(\boldsymbol{W},\boldsymbol{x}_{ij})=\phi_{\boldsymbol{a}}(\boldsymbol{\tilde{W}}_i,\boldsymbol{x}_{ij})^\top\boldsymbol{W}$ for $\forall \boldsymbol{W}\in\mathbb{R}^d$ and $\boldsymbol{x}_{ij}$. 

\begin{lemma}\label{le-lse}
Let $m = \operatorname{poly}(N, \delta^{-1}, \lambda_0^{-1})/\sigma^2$, $\eta = O(\lambda_0/N^2)$, and $\sigma = O(\lambda_0 \delta / N)$. With probability at least $1 - \delta$ over the random initialization, we have
\begin{equation*}
    \left\|\tilde{f}(\boldsymbol{\bar{W}}_{\mathcal{P}},\boldsymbol{x}_{ij})-\boldsymbol{\tilde{y}}\right\|_2^2\le O\left(\frac{N}{\rho}\boldsymbol{y}^\top(\mathbf{H}^\infty)^{-1}\boldsymbol{y}\right).
\end{equation*}
\end{lemma}

We define four matrices associated with the initial model $\boldsymbol{W}^{(0)}$ and participant set $\mathcal{P} = \{i_1, \dots, i_P\}$: $\mathbf{A}^{(0)}_{\mathcal{P}}, \mathbf{\tilde{A}}_{\mathcal{P}} \in \mathbb{R}^{(Pn) \times (mp)}$ and $\mathbf{H}_{\mathcal{P}}^{(0)}, \mathbf{\tilde{H}}_{\mathcal{P}} \in \mathbb{R}^{(Pn) \times (Pn)}$:
\[
\begin{alignedat}{2}
    &\mathbf{A}^{(0)}_{\mathcal{P}} \coloneqq
    \begin{bmatrix}
        \phi_{\boldsymbol{a}}(\boldsymbol{W}^{(0)}, \boldsymbol{x}_{i_1 1})^\top \\
        \phi_{\boldsymbol{a}}(\boldsymbol{W}^{(0)}, \boldsymbol{x}_{i_1 2})^\top \\
        \vdots \\
        \phi_{\boldsymbol{a}}(\boldsymbol{W}^{(0)}, \boldsymbol{x}_{i_P n})^\top
    \end{bmatrix},
    \quad
    &&\mathbf{\tilde{A}}_{\mathcal{P}} \coloneqq
    \begin{bmatrix}
        \phi_{\boldsymbol{a}}(\boldsymbol{\tilde{W}}_1, \boldsymbol{x}_{i_1 1})^\top \\
        \phi_{\boldsymbol{a}}(\boldsymbol{\tilde{W}}_1, \boldsymbol{x}_{i_1 2})^\top \\
        \vdots \\
        \phi_{\boldsymbol{a}}(\boldsymbol{\tilde{W}}_M, \boldsymbol{x}_{i_P n})^\top
    \end{bmatrix}, \\
    &\mathbf{H}_{\mathcal{P}}^{(0)} \coloneqq \mathbf{A}^{(0)}_{\mathcal{P}} \mathbf{A}^{(0)\top}_{\mathcal{P}},
    \quad
    &&\mathbf{\tilde{H}}_{\mathcal{P}} \coloneqq \mathbf{\tilde{A}}_{\mathcal{P}} \mathbf{\tilde{A}}_{\mathcal{P}}^\top.
\end{alignedat}
\]

We next study the conditioning of the Gram matrices associated with the sampled clients, which is needed to control the projection step.

\begin{lemma}\label{le-H0}
    Let $\mathbf{H}^{\infty}_{\mathcal{P}} \in \mathbb{R}^{(Pn)\times (Pn)}$ be the principal submatrix of $\mathbf{H}^{\infty}$ formed by  $\mathcal{P}$. If $m = \Omega(\frac{(\rho N)^2 \log(\rho N/\delta)}{\lambda_0^2})$, with probability at least $1 - \delta$, $\|\mathbf{H}_\mathcal{P}^{(0)}\!-\!\mathbf{H}^{\infty}_{\mathcal{P}}\|\!\le\!\lambda_0 / 4$, $\lambda_{\min}(\mathbf{H}_\mathcal{P}^{(0)})\!\ge\!3\lambda_0 / 4$.
\end{lemma}

\begin{lemma}\label{le-eigen}
If $m = \Omega(\frac{N^6}{\lambda_0^4\sigma^2\delta^3})$, then with probability at least $1 - \delta$, we have $\lambda_{\min}(\mathbf{\tilde{H}}_{\mathcal{P}}) \ge \lambda_0 / 2$ and $\lambda_{\max}(\mathbf{\tilde{H}}_{\mathcal{P}}) \le \rho N$.
\end{lemma}

With these spectral bounds in place, we can quantify the additional displacement introduced by Fisher-based projection.

\begin{lemma}\label{le-fishdev}
    Let $m\ge\frac{\sigma^2}{\operatorname{poly}(\rho, N, \delta^{-1}, \lambda_0^{-1})}$, $\sigma=O(\frac{\lambda_0\delta}{N})$. With probability at least $1-\delta$ over the initialization, we have
    \begin{equation*}
        \left\|\boldsymbol{\bar{W}}_{\mathcal{P}}^{\text{F}}-\boldsymbol{\bar{W}}_{\mathcal{P}}\right\|_2=O\left(\sqrt{\frac{N}{\rho\lambda_0}\boldsymbol{y}^\top(\mathbf{H}^{\infty})^{-1}\boldsymbol{y}}\right).
    \end{equation*}
\end{lemma}

\begin{lemma}\label{le-fishopt}
    Let $m = \operatorname{poly}(N, \delta^{-1}, \lambda_0^{-1})/\sigma^2$, $\eta = O(\lambda_0/N^2)$, and $\sigma = O(\lambda_0 \delta / N)$. With probability at least $1-\delta$ over the random initialization, we have
    \begin{equation*}
        \|\boldsymbol{\bar{w}}_{\mathcal{P},r}^{\text{F}}-\boldsymbol{\bar{w}}_{\mathcal{P},r}\|_2\le R_1\le O\left(\frac{N}{\sqrt{m}\lambda_0}\sqrt{\boldsymbol{y}^\top(\mathbf{H}^\infty)^{-1}\boldsymbol{y}}\right).
    \end{equation*}    
\end{lemma}

The resulting neuron-wise displacement bound allows us to control activation-pattern mismatches between the fused and local models.
Given any $(i,j)$ pair, we define $S_{ij}^{\text{F}} \coloneqq \{ r \in [m] : \mathds{1}\{\boldsymbol{x}_{ij}^\top \boldsymbol{\bar{w}}_{\mathcal{P},r}^{\text{F}} \ge 0\} = \mathds{1}\{\boldsymbol{x}_{ij}^\top \boldsymbol{\tilde{w}}_{i,r} \ge 0\} \}$ as the set of neurons with matching activation signs between the global and local models. Let $S_{ij}^{\text{F}\perp} \coloneqq [m] \setminus S_{ij}^{\text{F}}$ denote the mismatch set.

\begin{lemma}\label{le-mismatch-size}
    With probability at least $1-\delta$ over the initialization, $|S_{ij}^{\text{F}\perp}|=O(\frac{mN(R_0+R_1)}{\delta\sigma})$ for all $i,j$ simultaneously.
\end{lemma}

Note that the mismatch size $|S_{ij}^{\text{F}\perp}|$ is independent of the client subsampling rate $\rho$, since $R_0$ bounds the deviation of each local trained model $\boldsymbol{\tilde{w}}_{i,r}$ from initial $\boldsymbol{w}^{(0)}_r$ uniformly across $i \in [M]$, regardless of which clients are sampled.

We are now ready to establish the empirical fitting guarantee of F-SOSA. Lemma~\ref{le-loc-err} controls the local fitting error, while Lemmas~\ref{le-fishdev}--\ref{le-mismatch-size} bound the additional error induced by Fisher-based fusion through the resulting activation-pattern mismatches. Combining these two components yields the following training-error guarantee.

\begin{theorem}[Training Error for Fitting the Data]\label{th-training-err-fish}
Under Assumptions~\ref{asp-normalization} -- \ref{asp-degen}, let $m = \Omega(N^6 / (\lambda_0^4 \sigma^2 \delta^3))$, $\eta = \lambda_0 / N^2$, $\boldsymbol{w}_r^{(0)} \sim \mathcal{N}(0, \sigma^2 \mathbf{I})$, $a_r \sim \text{Uniform}(\{-1, +1\})$, then with probability at least $1 - \delta$, we have 
\begin{equation*}
    L_{\mathcal{D}}^{f,\ell}(\boldsymbol{\bar{W}}_{\mathcal{P}}^{\text{F}}) \le O((1 - \frac{\eta \lambda_0}{2})^K \cdot \frac{N}{\delta} + \frac{N^6 (\boldsymbol{y}^\top (\mathbf{H}^\infty)^{-1} \boldsymbol{y})^2}{m \delta^4 \sigma^2 \lambda_0^4}).
\end{equation*}
\end{theorem}

\begin{proof}
We decompose the loss $L_{\mathcal{D}}^{f,\ell}(\boldsymbol{\bar{W}}_{\mathcal{P}}^{\text{F}})$ as follows
\begin{flalign*}
    &L(\boldsymbol{\bar{W}}_{\mathcal{P}}^{\text{F}})
    =\frac{1}{N}\sum_{i=1}^M\sum_{j=1}^n(f(\boldsymbol{\bar{W}}_{\mathcal{P}}^{\text{F}}, \boldsymbol{x}_{ij})-y_{ij})^2\\
    =&\frac{1}{N}\sum_{i=1}^M\sum_{j=1}^n(\phi(\boldsymbol{\bar{W}}_{\mathcal{P}}^{\text{F}}, \boldsymbol{x}_{ij})^\top\boldsymbol{\bar{W}}_{\mathcal{P}}^{\text{F}}-\tilde{y}_{ij}+\tilde{y}_{ij}-y_{ij})^2\\
    \le& \underbrace{\frac{2}{N}\!\sum_{i=1}^M\!\sum_{j=1}^n(\phi(\boldsymbol{\bar{W}}_{\mathcal{P}}^{\text{F}},\!\boldsymbol{x}_{ij})^\top\!\boldsymbol{\bar{W}}_{\mathcal{P}}^{\text{F}}\!-\!\tilde{y}_{ij})^2}_{T_1}\!+\!\underbrace{\frac{2}{N}\!\sum_{i=1}^M\!\sum_{j=1}^n(\tilde{y}_{ij}\!-\!y_{ij})^2}_{T_2}.
\end{flalign*}
$T_2$ measures how well the local models fit their local data and can be bounded via Lemma~\ref{le-loc-err} as $T_2=O\left(\left(1-\frac{\eta\lambda_0}{2}\right)^K\frac{N}{\delta}\right)$.
By taking $t\to\infty$ in \eqref{eq-predict}, $\tilde{y}_{ij} = \phi_{\boldsymbol{a}}(\boldsymbol{\tilde{W}}_i, \boldsymbol{x}_{ij})^\top \boldsymbol{\bar{W}}_{\mathcal{P}}^{\text{F}}$, we have
\begin{flalign*}
    &T_1
    =\frac{2}{N}\sum_{i=1}^M\sum_{j=1}^n(\phi(\boldsymbol{\bar{W}}_{\mathcal{P}}^{\text{F}}, \boldsymbol{x}_{ij})^\top\boldsymbol{\bar{W}}_{\mathcal{P}}^{\text{F}}-\phi(\boldsymbol{\tilde{W}}_{i}, \boldsymbol{x}_{ij})^\top\boldsymbol{\bar{W}}_{\mathcal{P}}^{\text{F}})^2\\
    &=\frac{2}{N}\sum_{i=1}^M\sum_{j=1}^n\\
    &\frac{(\sum\nolimits_{r\in S_{ij}^{\text{F}\perp}} a_r\boldsymbol{x}_{ij}^\top\boldsymbol{\bar{w}}_{\mathcal{P},r}^{\text{F}}(\mathds{1}\{\boldsymbol{x}_{ij}^\top\boldsymbol{\bar{w}}_{\mathcal{P},r}^{\text{F}}\!\ge\! 0\}\!-\!\mathds{1}\{\boldsymbol{x}_{ij}^\top\boldsymbol{\tilde{w}}_{i,r}\!\ge\! 0\}))^2}{m}\\
    &\le \frac{2}{N}\sum_{i=1}^M\sum_{j=1}^n\frac{|S_{ij}^{\text{F}\perp}|}{m}\sum_{r\in S_{ij}^\perp}\left(a_r\boldsymbol{x}_{ij}^\top\boldsymbol{\bar{w}}_{\mathcal{P}, r}^{\text{F}}\right)^2\\
    &\le \frac{2}{N}\sum_{i=1}^M\sum_{j=1}^n\frac{|S_{ij}^{\text{F}\perp}|}{m}\sum_{r\in S_{ij}^{\text{F}\perp}}\left(\boldsymbol{x}_{ij}^\top\boldsymbol{\bar{w}}_{\mathcal{P}, r}^{\text{F}}-\boldsymbol{x}_{ij}^\top\boldsymbol{\tilde{w}}_{i, r}\right)^2\\
    &\le \frac{2}{N}\sum_{i=1}^M\sum_{j=1}^n\frac{|S_{ij}^{\text{F}\perp}|^2}{m}\max_{r\in[m]}\left\|\boldsymbol{\bar{w}}_{\mathcal{P},r}^{\text{F}}-\boldsymbol{\tilde{w}}_{i,r}\right\|_2^2\\
    &= O\left(\frac{mN^2(R_0^2+R_1^2)}{\delta^2\sigma^2}(R_0^2+R_1^2)\right)\le O\left(\frac{mN^2(R_0^4+R_1^4)}{\delta^2\sigma^2}\right)\\
    &\le O\left(\frac{mN^2}{\delta^2\sigma^2}\left(\frac{N^4}{\lambda_0^4\delta^2 m^2}+\frac{N^4(\boldsymbol{y}^\top(\mathbf{H}^\infty)^{-1}\boldsymbol{y})^2}{\lambda_0^4 m^2}\right)\right)\\
    &\le O\left(\frac{N^{6}\left(\boldsymbol{y}^\top(\mathbf{H}^\infty)^{-1}\boldsymbol{y}\right)^2}{m\delta^4\sigma^2\lambda_0^4}\right)
\end{flalign*}
where the second-to-the-last inequality uses upper bounds on $R_0$ and $R_1$ from Lemma~\ref{le-loc-err} and Lemma~\ref{le-fishopt}.
\end{proof}

\begin{corollary}\label{co-training-err-fish}
Under the setting of Theorem~\ref{th-training-err-fish}, given a target training error $\epsilon$, if we set $m = \Omega(\operatorname{poly}(N, \delta^{-1}, \lambda_0^{-1})/\sigma^2)$, $\eta = \lambda_0 / N^2$, $\sigma = O(\lambda_0 \delta / N)$, and $K = \Omega((1/\eta \lambda_0) \log(N / \delta \epsilon))$, then with probability at least $1 - \delta$, we have $L_{\mathcal{D}}^{f,\ell}(\boldsymbol{\bar{W}}_{\mathcal{P}}^{\text{F}}) \le \epsilon$.
\end{corollary}

Corollary~\ref{co-training-err-fish} shows that the empirical risk can be made vanishingly small in the over-parameterized regime by choosing the network width $m$ and the number of local iterations $K$ sufficiently large. 

Having established the empirical fitting guarantee, we next study generalization to unseen data. For this purpose, we adopt a two-level data distribution framework~\cite{hu2023generalization}. Let $\mathcal{Z} \coloneqq \mathcal{X} \times \mathcal{Y}$ denote the data space, $\mathfrak{D}$ the set of distributions over $\mathcal{Z}$, and $\mathcal{T}$ a meta-distribution over $\mathfrak{D}$. Each client $i$ samples a distribution $D_i \sim \mathcal{T}$ and obtains $n$ i.i.d. samples $\mathcal{D}_i \sim D_i^n$. Let $\mathcal{D} = \bigcup_{i=1}^M \mathcal{D}_i$ be the union of local datasets.
We consider the squared loss $\ell(z, y) = \frac{1}{2}(z - y)^2$. The empirical risk over training data is 
\begin{equation*}
    L_{\mathcal{D}}^{f,\ell}(\boldsymbol{W}) = \frac{1}{M} \sum_{i=1}^M \frac{1}{n} \sum_{j=1}^n \frac{1}{2}(f_{\boldsymbol{a}}(\boldsymbol{W}, \boldsymbol{x}_{ij}) - y_{ij})^2.
\end{equation*}
The population risk is defined as 
\begin{equation*}
    L_{\mathcal{T}}^{f,\ell}(\boldsymbol{W}) = \mathbb{E}_{D \sim \mathcal{T}} \mathbb{E}_{\boldsymbol{z} \sim D} [ \frac{1}{2}(f_{\boldsymbol{a}}(\boldsymbol{W}, \boldsymbol{x}) - y)^2 ].
\end{equation*}
We also define the semi-population risk:
\begin{equation*}
    L_S^{f,\ell}(\boldsymbol{W}) = \frac{1}{M} \sum_{i=1}^M \mathbb{E}_{\boldsymbol{z} \sim D_i} [ \frac{1}{2}(f_{\boldsymbol{a}}(\boldsymbol{W}, \boldsymbol{x}) - y)^2 ].
\end{equation*}

To extend the empirical analysis to unseen data, we adopt the following federated Rademacher complexity.

\begin{definition}[Federated Rademacher complexity]\label{def-rademacher}
Let $\mathcal{F}$ be a class of functions $g: \mathcal{X} \to \mathbb{R}$. For $M$ clients with $n$ samples each, the federated Rademacher complexity is 
\begin{equation*}
    \mathcal{R}_n^M(\mathcal{F}) = \mathbb{E}_{\{\iota_{ij}\}} [ \sup_{g \in \mathcal{F}} \frac{1}{M} \sum_{i=1}^M \frac{1}{n} \sum_{j=1}^n \iota_{ij} \cdot g(\boldsymbol{x}_{ij})],
\end{equation*}
where $\{\iota_{ij}\}$ are i.i.d. Rademacher variables.
\end{definition}
When $M=1$, this reduces to the classical Rademacher complexity $\mathbb{E}_{\{\iota_j\}} [ \sup_{g \in \mathcal{F}} \frac{1}{n} \sum_{j=1}^n \iota_j g(\boldsymbol{x}_j)]$.
We define a neural network class with weight constraints relative to initialization: $\mathcal{F}_{\boldsymbol{W}^{(0)}, \boldsymbol{a}, R, B} \coloneqq \{ f_{\boldsymbol{a}}(\boldsymbol{W}, \cdot) : \|\boldsymbol{w}_r - \boldsymbol{w}^{(0)}_r\| \le R,\ \forall r \in [m],\ \|\boldsymbol{W} - \boldsymbol{W}^{(0)}\|_2 \le B \}$.
The following existing result bounds the Rademacher complexity of the network class determined by the parameter neighborhood above.

\begin{lemma}[{\cite[Lemma~8]{jhunjhunwala2024fedfisher}}]\label{le-rademacher}
Given $R > 0$, with probability at least $1 - \delta$ over the initialization, uniformly over all $B > 0$, 
\begin{equation*}
\begin{aligned}
    &\mathcal{R}_n^M(\mathcal{F}_{\boldsymbol{W}^{(0)}, \boldsymbol{a}, R, B}) \\
    &\le \frac{B}{\sqrt{2N}}(1 + (\frac{2\log(2/\delta)}{m})^{1/4}) + \frac{2R^2\sqrt{m}}{\sigma} + R\sqrt{2\log(2/\delta)}.
\end{aligned}
\end{equation*}
\end{lemma}

We next recall the empirical-to-semi-population generalization bound under the two-level federated data model.

\begin{lemma}[{\cite[Lemma~1]{hu2023generalization}}]\label{le-general-err}
Let the loss function $\ell(f(\boldsymbol{W}, \boldsymbol{x}), y)$ be bounded in $[0,c]$. Under the two-level distribution setting, with probability at least $1 - \delta$, we have
\begin{equation*}
    \sup_{f \in \mathcal{F}} \{ L_S^{f,\ell} - L_{\mathcal{D}}^{f,\ell} \} \le 2 \mathcal{R}_n^M(\mathcal{F}) + 3c \sqrt{\frac{\log(2/\delta)}{2N}}.
\end{equation*}
\end{lemma}

The remaining step is to control the gap between the semi-population and population risks.

\begin{lemma}\label{le-general-err2}
Let the loss $\ell(f(\boldsymbol{W}, \boldsymbol{x}), y)$ be bounded in $[0, c]$. With probability at least $1 - \delta$,
\begin{equation*}
    \sup_{f \in \mathcal{F}} \{ L_{\mathcal{T}}^{f,\ell} - L_S^{f,\ell} \} \le 2\mathcal{R}_n^M(\mathcal{F}) + 2c\sqrt{\frac{\log(2/\delta)}{2N}} + c\sqrt{\frac{\log(2/\delta)}{2M}}.
\end{equation*}
\end{lemma}

Lemmas~\ref{le-model-dist} and \ref{le-fishdev}--\ref{le-fishopt} control the fused model around the initialization, while Lemmas~\ref{le-rademacher}--\ref{le-general-err2} bound the generalization gaps. Together with Theorem~\ref{th-training-err-fish}, these results yield the following population-risk bound.

\begin{theorem}[Generalization Error for Unseen Data]\label{th-gen-err}
    Under Assumptions~\ref{asp-normalization},~\ref{asp-uniqueness} and~\ref{asp-degen}, let $\sigma=O(\lambda_0 \delta/N)$, $m=\Omega(\operatorname{poly}(N, \delta^{-1}, \lambda_0^{-1})/\sigma^2)$, $\eta=\lambda_0/N^2$,  $\boldsymbol{w}_r^{(0)}\sim\mathcal{N}(0, \sigma^2 \mathbf{I})$, $a_r\sim\text{Uniform}(\{-1, +1\})$ and $K=\Omega((1/\eta \lambda_0)\log(\rho N / (\delta(1+\rho)\boldsymbol{y}^\top (\mathbf{H}^\infty)^{-1} \boldsymbol{y})))$, with probability at least $1 - \delta$, we have 
    \begin{equation*}
        L_{\mathcal{T}}^{f,\ell}(\boldsymbol{\bar{W}}_{\mathcal{P}}^{\text{F}}) \le O(\frac{1}{\rho}\sqrt{\frac{\boldsymbol{y}^\top (\mathbf{H}^\infty)^{-1} \boldsymbol{y}}{\lambda_0}} + \sqrt{\frac{\log\frac{1}{\delta}}{N}} + \sqrt{\frac{\log \frac{1}{\delta}}{M}}).
    \end{equation*}
\end{theorem}

\begin{proof}
We first bound $\|\boldsymbol{\bar{W}}_{\mathcal{P}}^{\text{F}} - \boldsymbol{\bar{W}}_{\mathcal{P}}\|_2$ via Lemmas~\ref{le-fishdev} and Lemma~\ref{le-model-dist}, which gives
\begin{equation}\label{eq-model-dist-fish}
    \|\boldsymbol{\bar{W}}_{\mathcal{P}}^{\text{F}} - \boldsymbol{\bar{W}}_{\mathcal{P}}\|_2 \le O(\sqrt{N/(\rho \lambda_0)} \cdot \sqrt{\boldsymbol{y}^\top (\mathbf{H}^\infty)^{-1} \boldsymbol{y}}).
\end{equation}
Similarly, we bound $\|\boldsymbol{\bar{w}}_{\mathcal{P}, r}^{\text{F}} - \boldsymbol{\bar{w}}_{\mathcal{P}, r}\|_2$ via Lemmas~\ref{le-loc-err} and Lemma~\ref{le-fishopt}, which gives
\begin{equation}\label{eq-weight-dist-fish}
    \|\boldsymbol{\bar{w}}_{\mathcal{P}, r}^{\text{F}} - \boldsymbol{w}_r^{(0)}\|_2 \le O(\operatorname{poly}(N, \delta^{-1}, \lambda_0^{-1})/\sqrt{m})
\end{equation}
Using above bounds, we can conclude that $f_{\boldsymbol{a}}(\boldsymbol{\bar{W}}_{\mathcal{P}}^{\text{F}}, \cdot) \in \mathcal{F}_{\boldsymbol{W}^{(0)}, \boldsymbol{a}, R, B}$ with $R$ and $B$ set by Equation~\eqref{eq-weight-dist-fish} and~\eqref{eq-model-dist-fish} respectively. Using Lemma~\ref{le-general-err} and Lemma~\ref{le-general-err2}, we have
\begin{flalign*}
    &\sup_{f\in\mathcal{F}_{\boldsymbol{W}^{(0)},\boldsymbol{a},R, B}}\{L_{\mathcal{T}}^{f,\ell}(\boldsymbol{\bar{W}}_{\mathcal{P}}^{\text{F}})-L_{\mathcal{D}}^{f, \ell}(\boldsymbol{\bar{W}}_{\mathcal{P}}^{\text{F}})\} \\
    \le &\sup_{f\in\mathcal{F}_{\boldsymbol{W}^{(0)},\boldsymbol{a},R, B}}\{L_{\mathcal{T}}^{f,\ell}(\boldsymbol{\bar{W}}_{\mathcal{P}}^{\text{F}})\!-\!L_{\mathcal{S}}^{f, \ell}(\boldsymbol{\bar{W}}_{\mathcal{P}}^{\text{F}})\}\\
    &+\sup_{f\in\mathcal{F}_{\boldsymbol{W}^{(0)},\boldsymbol{a},R, B}}\{L_{S}^{f,\ell}(\boldsymbol{\bar{W}}_{\mathcal{P}}^{\text{F}})\!-\!L_{\mathcal{D}}^{f, \ell}(\boldsymbol{\bar{W}}_{\mathcal{P}}^{\text{F}})\}\\
    \le& O(\mathcal{R}_n^M(\mathcal{F}_{\boldsymbol{W}^{(0)},\boldsymbol{a},R, B})+\sqrt{\log(1/\delta)/N}+\sqrt{\log(1/\delta)/M}).
\end{flalign*}
By Lemma~\ref{le-rademacher}, the Rademacher complexity satisfies $\mathcal{R}_n^M(\mathcal{F}_{\boldsymbol{W}^{(0)}, \boldsymbol{a}, R, B}) \le O(\frac{1}{\rho} \sqrt{\boldsymbol{y}^\top (\mathbf{H}^\infty)^{-1} \boldsymbol{y} / \lambda_0})$. Since $L_{\mathcal{T}}^{f,\ell}(\boldsymbol{\bar{W}}_{\mathcal{P}}^{\text{F}})=L_{\mathcal{T}}^{f,\ell}(\boldsymbol{\bar{W}}_{\mathcal{P}}^{\text{F}})-L_{\mathcal{D}}^{f, \ell}(\boldsymbol{\bar{W}}_{\mathcal{P}}^{\text{F}})+L_{\mathcal{D}}^{f, \ell}(\boldsymbol{\bar{W}}_{\mathcal{P}}^{\text{F}})$, by combing with the training error established in Theorem~\ref{th-training-err-fish}, we obtain the result by union bound.
\end{proof}

Theorem~\ref{th-gen-err} shows that, with sufficiently many clients and local samples, the generalization of F-SOSA is governed by a Rademacher complexity term of order $\frac{1}{\rho} \sqrt{\boldsymbol{y}^\top (\mathbf{H}^\infty)^{-1} \boldsymbol{y} / \lambda_0}$. As shown in \cite{arora2019fine}, the quantity $\boldsymbol{y}^\top (\mathbf{H}^\infty)^{-1} \boldsymbol{y}$ is bounded by a constant independent of sample size for a broad class of target functions, including linear, quadratic, and cosine models. The factor $1/\rho$ reflects a trade-off between editing efficiency and generalization: larger $\rho$ includes more clients and improves the utility/generalization bound, but weakens TV stability and increases the expected editing cost. 

\subsection{Exactness and Efficiency of Model Editing}

F-SOSA-U/I are valid transports and efficient for editing.

\begin{theorem}\label{th-editing}
    (F-SOSA, F-SOSA-U/I) satisfies exact federated model editing. For $R$ editing requests, F-SOSA-U/I triggers at most $O(\rho R)$ recomputation in expectation, with total expected runtime bounded by $O(\max\{\rho R \cdot \text{Training Time}, R\})$.
\end{theorem}

\begin{proof}
Let $\mu_{(P,M)}$ be the uniform distribution over size-$P$ subsets of $[M]$. Let $\mathcal{P} \sim \mu_{(P,M)}$ be a sampled client subset. Assume $M$ leaves or $M+1$ joins, and consider a deletion map $\operatorname{DEL} : \mathcal{P} \mapsto \mathcal{P}' \subseteq [M{-}1]$ (for client removal), and an insertion map $\operatorname{INS} : \mathcal{P} \mapsto \mathcal{P}' \subseteq [M{+}1]$ (for new client addition). These induce pushforward distributions $\mu_{(P,M)}^{\operatorname{DEL}} := \operatorname{DEL}\#\mu_{(P,M)}$ and $\mu_{(P,M)}^{\operatorname{INS}} := \operatorname{INS}\#\mu_{(P,M)}$. To show F-SOSA-U/I are valid transports, we prove that for any $\mathcal{P}\in[M]^P$, $\mu_{(P,M)}^{\operatorname{DEL}}(\mathcal{P})=\mu_{P,M-1}(\mathcal{P})$ and for any $\mathcal{P}\in[M+1]^P$, $\mu_{(P,M)}^{\operatorname{INS}}(\mathcal{P})=\mu_{(P,M+1)}(\mathcal{P})$.

For F-SOSA-U, we consider two cases based on whether the membership verification succeeds. \textbf{Case 1 (verification succeeds).} The pushforward measure $\mu_{(\!P,M\!)}^{\operatorname{DEL}}$ is simply $\mu_{(P,M)}$ conditioned on the event $M\!\notin\!\mathcal{P}$,
\begin{equation*}
    \mu_{(\!P,M\!)}^{\operatorname{DEL}}(\mathcal{P})\!=\!\mu_{(\!P,M\!)}(\mathcal{P}|M\!\notin\!\mathcal{P})\!=\!\frac{\mu_{(\!P,M\!)}(\mathcal{P}\cap\{M\!\notin\!\mathcal{P}\})}{\mu_{(\!P,M\!)}(\{M\!\notin\!\mathcal{P}\})}.
\end{equation*}
The denominator $\mu_{(\!P,M\!)}(\{M\!\notin\!\mathcal{P}\})\!=\!1\!-\!\mu_{(\!P,M\!)}(\{M\!\in\!\mathcal{P}\})=1\!-\!\binom{M\!-\!1}{P\!-\!1}/\binom{M}{P}=1\!-\!P/M$. For $\mu_{(P,M)}(\mathcal{P}\cap\{M\notin\mathcal{P}\})$, it equals $0$ if $M\in\mathcal{P}$, and $\mu_{(P,M)}(\mathcal{P})$ otherwise. In both branches, we have $\mu_{(P,M)}^{\operatorname{DEL}}(\mathcal{P})=\mu_{P,M-1}(\mathcal{P})$. 
\textbf{Case 2 (verification fails).} In this case, $M$ would have been replaced by some point in $\mathcal{P}$ and there are $M$ possibilities. Let $\upsilon$ be the uniform probability measure over $[M]$. For each $k\in[M]$, let $\mathcal{E}_k$ be the event that $M$ was replaced by $i_k$. The events $\mathcal{E}_k$'s are disjoint and $\mathcal{P}$ is $\cup_{k=1}^{M}\mathcal{E}_k$. Then $\mu_{(\!P,M\!)}^{\operatorname{DEL}}(\mathcal{P})=\mu_{(\!P,M\!)}^{\operatorname{DEL}}(\sum_{k=1}^M\!\mathcal{E}_k)\!=\!\sum_{k=1}^M\mu_{(\!P,M\!)}(\mathcal{P}\backslash\{\!i_k\!\}\cup\{\!M\!\})\upsilon(\{\!i_k\!\})\!=\!1/\binom{M\!-\!1}{P}\!=\!\mu_{(\!P,M\!-\!1\!)}(\mathcal{P})$.

For F-SOSA-I, given $\mathcal{P}$, we consider two cases based on whether the inserted $M+1$ lies in $\mathcal{P}$ or not. \textbf{Case 1 ($M\!+\!1\in\mathcal{P}$).} In this case, $\operatorname{Bernoulli}(P/(M\!+\!1)$ must be $1$, $M\!+\!1$ would have replaced some other points in $[M\!+\!1]\backslash\mathcal{P}$ and there are $M\!+\!1\!-\!P$ possibilities. Let $\nu$ be a uniform probability measure over $[M\!+\!1\!-\!P]$. For each $k\in[M\!+\!1\!-\!P]$, let $\mathcal{E}_k$ be the event that $M\!+\!1$ replaced $i_k$, then $\mathcal{E}_k$'s are disjoint and the event $\mathcal{P}$ is $\cup_{k=1}^{M\!+\!1\!-\!P}\mathcal{E}_k$. Then $\mu_{(\!P,M\!)}^{\operatorname{INS}}(\mathcal{E}_k)$ equals
\begin{flalign*}
    &\mu_{(\!P,M\!)}^{\operatorname{INS}}(\mathcal{P}\backslash\{\!M\!+\!1\!\}\!\cup\!\{\!i_k\!\}| i_k \,\text{replaced})\mu_{(\!P,M\!)}^{\operatorname{INS}}(i_k\,\text{replaced})\\
    &=\mu_{(\!P,M\!)}(\mathcal{P}\backslash\{\!M\!+\!1\!\}\!\cup\!\{\!i_k\!\}|\{\!i_k\!\})\nu(i_k)=\frac{1}{\binom{M}{P\!-\!1}(M\!+\!1\!-\!P)}.
\end{flalign*}
Thus, $\mu_{(\!P,M\!)}^{\operatorname{INS}}(\mathcal{P})=\frac{P}{M+1}\mu_{(\!P,M\!)}^{\operatorname{INS}}(\cup_{k=1}^{M\!+\!1\!-\!P}\mathcal{E}_k)=\mu_{(\!P,M+1\!)}(\mathcal{P})$. 
\textbf{Case 2 ($M\!+\!1\!\notin\!\mathcal{P}$).} $\operatorname{Bernoulli}(P/(M+1))=0$, there is no replacement, $\mu_{(P,M)}^{\operatorname{INS}}(\mathcal{P})=(1-\frac{P}{M+1})/\binom{M}{P}=1/\binom{M+1}{P}=\mu_{(P,M+1)}(\mathcal{P})$.

Therefore, F-SOSA-U/I are both valid transports. Let $\mathcal{P}^{\operatorname{EDT}}$ be the post-editing sampled subset. By $\rho$-TV-stability of F-SOSA, $\mathbb{P}(\mathcal{P}\!\ne\! \mathcal{P}^{\operatorname{EDT}})\!\le\!\rho$, and $M/2\le M(t)\le 2M$, so each request triggers recomputation with probability at most $2\rho$. For $R$ total edits, the expected number of recomputations is $O(\rho R)$, with total expected runtime bounded by $O(\max\{\rho R \cdot \text{Training Time}, R\})$.
\end{proof}

\begin{proposition}[Exact and Optimal Batch Coupling]
\label{prop:paired-batch}
In the batch setting of Section~\ref{sec:editing}, Algorithm~\ref{algo-FSOSABatch} satisfies exact model editing in Definition~\ref{def-data-unlearning} at batch completion.
Moreover, $\Pr(\mathcal P'\ne\mathcal P)=1-\binom{M-d}{P}/\max\{\binom MP,\binom{M'}P\}$,
and
$\mathbb E[|\mathcal P\setminus\mathcal P'|]
=\max\left\{\frac{Pd}{M},\frac{Pa}{M'}\right\}$.
The algorithm simultaneously minimizes these two quantities over all couplings of the initial and final uniform size-$P$ subset distributions, where $\binom{M-d}{P}=0$ if $M-d<P$.
\end{proposition}

\begin{proof}
Condition on the common training randomness and an independently chosen pairing. Paired replacement is a bijective relabeling and preserves uniformity. Remaining insertions preserve uniformity by the fixed-$P$ insertion coupling. For remaining deletions, conditional on the number of survivors, the surviving selected clients form a uniform subset of $\mathcal M'$; uniform replenishment therefore produces a uniform size-$P$ subset. Thus $\mathcal P'$ has the retraining subset distribution. The net updates recover its aggregate statistics exactly, and the same fusion rule produces its model. After temporary buffers are discarded, the model and internal state are the same functions of $\mathcal P'$ as in retraining. Averaging over the common training randomness establishes Definition~\ref{def-data-unlearning}. For any coupling $(S,T)$ of the initial and final uniform subset distributions, equality is possible only on the $\binom{M-d}{P}$ subsets containing no deleted or added clients. Each such subset contributes at most the smaller of its two marginal probabilities, giving
$\Pr(S=T) \le {\binom{M-d}{P}}/{\max\{\binom MP,\binom{M'}P\}}$.
Moreover, every selected deleted client must leave, and every selected added client must enter. Since both subsets have size $P$,
$|S\setminus T|\ge |S\cap\mathcal D|,
|S\setminus T|\ge |T\cap\mathcal A|$.
Taking expectations yields 
$\mathbb E[|S\setminus T|]
\ge \max\left\{\frac{Pd}{M},\frac{Pa}{M'}\right\}$.
The algorithm attains both bounds. If $a\ge d$, it adds no previously unselected common client. Hence
$\mathcal P'=\mathcal P$ exactly when
$\mathcal P'\cap\mathcal A=\varnothing$, and
$|\mathcal P\setminus\mathcal P'|
=|\mathcal P'\cap\mathcal A|$.
Uniformity gives agreement probability
$\binom{M-d}{P}/\binom{M'}P$ and expected net replacements $Pa/M'$.
If $d\ge a$, every originally selected common client is retained. Hence equality holds exactly when $\mathcal P\cap\mathcal D=\varnothing$, and $|\mathcal P\setminus\mathcal P'|=|\mathcal P\cap\mathcal D|$, giving agreement probability $\binom{M-d}{P}/\binom MP$ and expectation $Pd/M$.
Since $M'-M=a-d$ and $\frac{a}{M'}-\frac{d}{M}=\frac{(M-d)(a-d)}{MM'}$, both branches attain the two bounds simultaneously.
\end{proof}

The batch updates $|\mathcal P\setminus\mathcal P'|$
contribution pairs and performs one projection if and only
if $\mathcal P'\ne\mathcal P$.
For $k=d+a$, let $M_0$ be the initial population size.
If $P\le\rho M_0$ and $\min\{M,M'\}\ge M_0/2$,
Proposition~\ref{prop:paired-batch} gives
$\mathbb E[|\mathcal P\setminus\mathcal P'|]
\le 2\rho\max\{d,a\}\le 2\rho k$.
Since a subset change requires at least one net replacement,
the projection is triggered with probability at most
$\min\{1,2\rho k\}$.

\renewcommand{\arraystretch}{1.2}
\setlength{\arrayrulewidth}{0.3mm}  
\begin{table*}[ht]
\centering
\caption{Performance comparison of F-SOSA and baselines across MNIST, SVHN, and CIFAR-10 under varying Dirichlet heterogeneity $\alpha$. Best results are in bold; second-best results are underlined.}
\label{tab:my-table1}
\resizebox{\textwidth}{!}{%
\begin{tabular}{cccccccccccccccc}
\hline
\multirow{2}{*}{Dataset} & \multirow{2}{*}{Metrics} & \multicolumn{4}{c}{$\alpha=0.1$} &  & \multicolumn{4}{c}{$\alpha=0.5$} &  & \multicolumn{4}{c}{$\alpha=0.8$} \\ \cline{3-6} \cline{8-11} \cline{13-16}
&  & FedFisher & FedLPA & Dense & F-SOSA &  & FedFisher & FedLPA & Dense & F-SOSA &  & FedFisher & FedLPA & Dense & F-SOSA \\ \hline
\multirow{5}{*}{MNIST} & Accuracy (\%) & {\ul 90.35} & 87.90 & 85.92 & \textbf{91.08} &  & {\ul 95.98} & 94.21 & \textbf{96.06} & 95.43 &  & {\ul 96.08} & 92.82 & \textbf{96.93} & 95.80 \\
& Time (s) & 43.11 & \textbf{24.59} & 187.96 & {\ul 34.63} &  & 43.76 & \textbf{26.61} & 190.54 & {\ul 34.59} &  & 43.15 & \textbf{29.02} & 193.58 & {\ul 35.56} \\
& Comm. (MB) & 35.30 & 793.31 & {\ul 33.76} & \textbf{18.46} &  & 35.30 & 793.31 & {\ul 33.76} & \textbf{18.46} &  & 35.30 & 793.31 & {\ul 33.76} & \textbf{18.46} \\
& Time Efficiency & 2.10 & \textbf{3.57} & 0.52 & {\ul 2.63} &  & 2.19 & \textbf{3.54} & 0.57 & {\ul 2.76} &  & 2.23 & \textbf{3.20} & 0.58 & {\ul 2.69} \\
& Comm. Efficiency & {\ul 2.56} & 0.11 & 2.55 & \textbf{4.93} &  & 2.72 & 0.12 & {\ul 2.85} & \textbf{5.17} &  & 2.72 & 0.12 & {\ul 2.87} & \textbf{5.19} \\ \hline
\multirow{5}{*}{SVHN} & Accuracy (\%) & \textbf{64.89} & 38.81 & 62.78 & {\ul 64.88} &  & \textbf{80.52} & 76.24 & 70.63 & {\ul 80.33} &  & \textbf{79.64} & 77.75 & 78.76 & {\ul 79.40} \\
& Time (s) & {\ul 29.15} & 50.50 & 198.85 & \textbf{23.54} &  & {\ul 28.98} & 51.61 & 184.83 & \textbf{21.29} &  & {\ul 28.53} & 51.33 & 188.70 & \textbf{22.14} \\
& Comm. (MB) & 77.04 & 362.74 & {\ul 75.27} & \textbf{40.31} &  & 77.04 & 362.74 & {\ul 75.27} & \textbf{40.31} &  & 77.04 & 362.74 & {\ul 75.27} & \textbf{40.31} \\
& Time Efficiency & {\ul 2.23} & 0.77 & 0.36 & \textbf{2.76} &  & {\ul 2.78} & 1.48 & 0.40 & \textbf{3.77} &  & {\ul 2.79} & 1.51 & 0.46 & \textbf{3.59} \\
& Comm. Efficiency & {\ul 0.84} & 0.11 & 0.83 & \textbf{1.61} &  & {\ul 1.05} & 0.21 & 0.94 & \textbf{1.99} &  & 1.03 & 0.21 & {\ul 1.05} & \textbf{1.97} \\ \hline
\multirow{5}{*}{CIFAR10} & Accuracy (\%) & \textbf{39.99} & 30.49 & {\ul 36.57} & 32.38 &  & \textbf{51.86} & 40.36 & 50.09 & {\ul 51.69} &  & \textbf{55.18} & 44.32 & 52.58 & {\ul 52.65} \\
& Time (s) & {\ul 22.76} & 63.79 & 202.91 & \textbf{16.55} &  & {\ul 23.16} & 56.50 & 197.96 & \textbf{16.26} &  & {\ul 23.22} & 55.92 & 197.70 & \textbf{16.22} \\
& Comm. (MB) & 77.04 & 362.74 & {\ul 75.27} & \textbf{40.31} &  & 77.04 & 362.74 & {\ul 75.27} & \textbf{40.31} &  & 77.04 & 362.74 & {\ul 75.27} & \textbf{40.31} \\
& Time Efficiency & {\ul 1.76} & 0.48 & 0.19 & \textbf{1.96} &  & {\ul 2.24} & 0.72 & 0.26 & \textbf{3.18} &  & {\ul 2.38} & 0.79 & 0.28 & \textbf{3.25} \\
& Comm. Efficiency & {\ul 0.52} & 0.08 & 0.49 & \textbf{0.80} &  & {\ul 0.67} & 0.11 & 0.67 & \textbf{1.28} &  & {\ul 0.72} & 0.12 & 0.70 & \textbf{1.31} \\ \hline
\end{tabular}%
}
\end{table*}

\renewcommand{\arraystretch}{1.2}
\setlength{\arrayrulewidth}{0.3mm}  
\begin{table}[t]
\centering
\caption{Performance comparison on FEMNIST under a real-world non-IID setting. Best results are in bold; second-best results are underlined.}
\label{tab:my-table2}
\begin{tabular}{ccccc}
\hline
\multirow{2}{*}{Metrics} & \multirow{2}{*}{FedFisher} & \multirow{2}{*}{FedLPA} & \multirow{2}{*}{Dense} & \multirow{2}{*}{F-SOSA} \\
&  &  &  &  \\ \hline
Accuracy (\%) & \textbf{76.94} & 70.18 & {\ul 75.97} & 75.71 \\
Time (s) & 115.89 & \textbf{18.69} & 138.57 & {\ul 67.07} \\
Comm. (MB) & 554.12 & 1968.70 & {\ul 528.83} & \textbf{289.68} \\
Time Efficiency & 0.66 & \textbf{3.75} & 0.64 & {\ul 1.13} \\
Comm. Efficiency & 0.139 & 0.036 & {\ul 0.144} & \textbf{0.261} \\ \hline
\end{tabular}%
\end{table}

\renewcommand{\arraystretch}{1.2}
\setlength{\arrayrulewidth}{0.3mm}  
\begin{table*}[htbp]
\centering
\caption{Membership Inference Attack (MIA) results on unlearned global models from different methods on FEMNIST and CIFAR-10.}
\label{tab:my-table3}
\begin{tabular}{cccccccccccccc}
\hline
\multirow{2}{*}{Method} & \multicolumn{1}{l}{} & \multicolumn{5}{c}{FEMNIST} &  & \multicolumn{5}{c}{CIFAR-10} & \multicolumn{1}{l}{} \\ \cline{3-7} \cline{9-13}
 &  & Accuracy (\%) &  & Precision (\%) &  & AUC (\%) &  & Accuracy (\%) &  & Precision (\%) &  & AUC (\%) &  \\ \hline
FedErasesr &  & $50.01_{\pm0.02}$ &  & $51.08_{\pm3.23}$ &  & $50.01_{\pm0.02}$ &  & $48.99_{\pm0.78}$ &  & $40.75_{\pm7.83}$ &  & $48.99_{\pm0.78}$ &  \\
FedRR &  & $50.45_{\pm0.49}$ &  & $50.37_{\pm0.41}$ &  & $50.45_{\pm0.49}$ &  & $45.17_{\pm2.33}$ &  & $42.29_{\pm2.87}$ &  & $45.17_{\pm2.33}$ &  \\
FATS-CU &  & $50.47_{\pm0.25}$ &  & $50.33_{\pm0.18}$ &  & $50.47_{\pm0.25}$ &  & $53.54_{\pm0.88}$ &  & $61.29_{\pm2.99}$ &  & $53.54_{\pm0.88}$ &  \\
Retrain &  & $50.45_{\pm0.21}$ &  & $50.31_{\pm0.15}$ &  & $50.45_{\pm0.21}$ &  & $51.64_{\pm1.67}$ &  & $52.04_{\pm2.09}$ &  & $51.64_{\pm1.67}$ &  \\
F-SOSA-U &  & $50.92_{\pm0.33}$ &  & $50.63_{\pm0.23}$ &  & $50.92_{\pm0.33}$ &  & $51.63_{\pm1.80}$ &  & $52.17_{\pm2.72}$ &  & $51.63_{\pm1.80}$ &  \\ \hline
\end{tabular}%
\end{table*}

\begin{figure*}[t]
    \centering
    \begin{minipage}{0.26\textwidth}
        \centering
        \includegraphics[width=\textwidth]{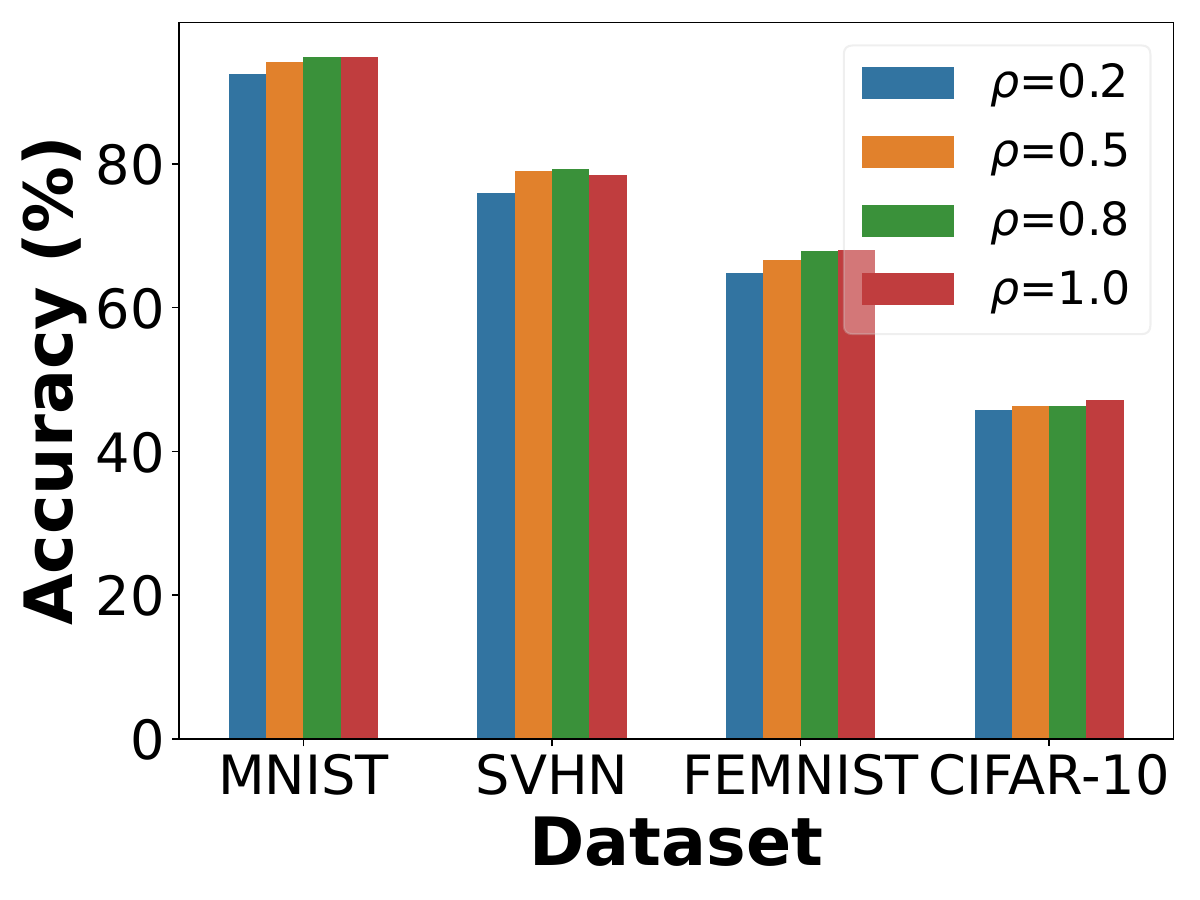}
    \end{minipage}%
    \begin{minipage}{0.26\textwidth}
        \centering
        \includegraphics[width=\textwidth]{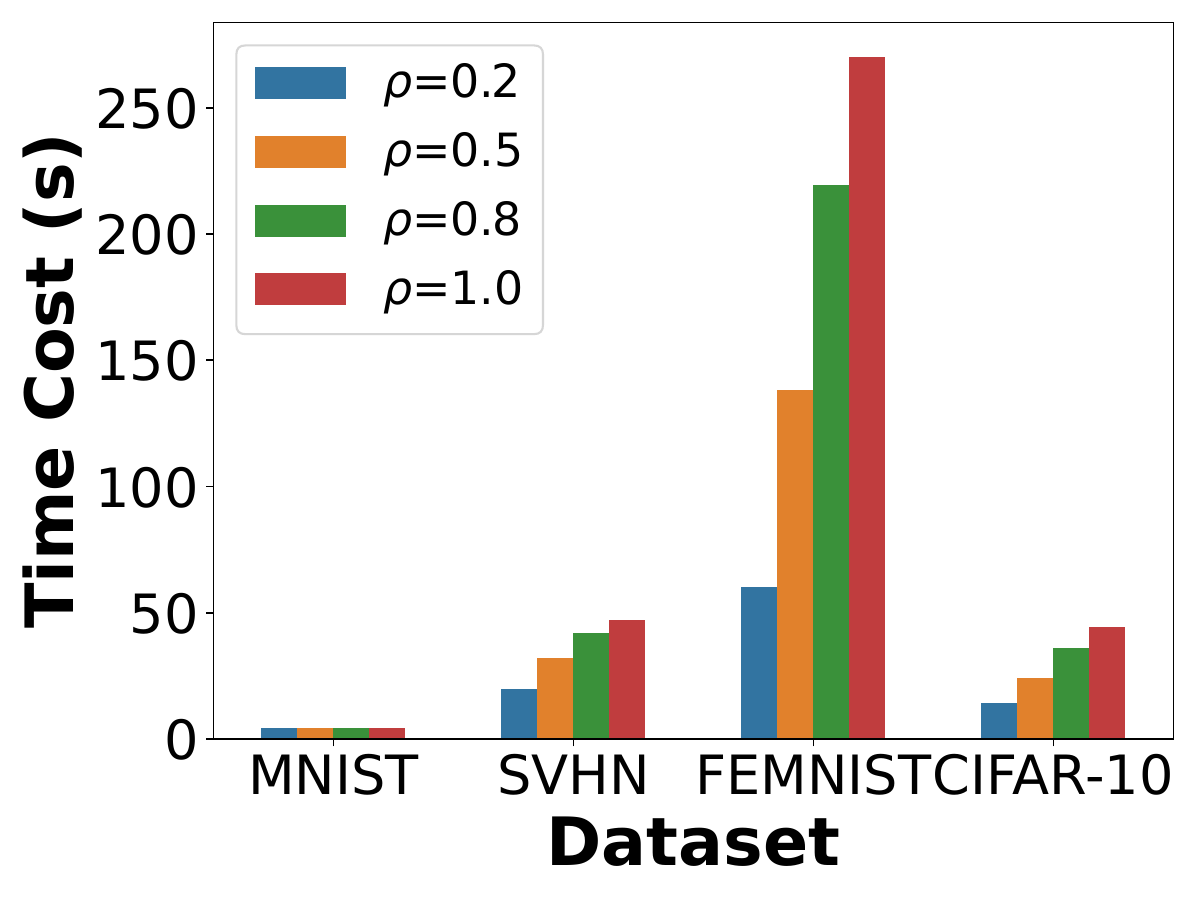}
    \end{minipage}%
    \begin{minipage}{0.26\textwidth}
        \centering
        \includegraphics[width=\textwidth]{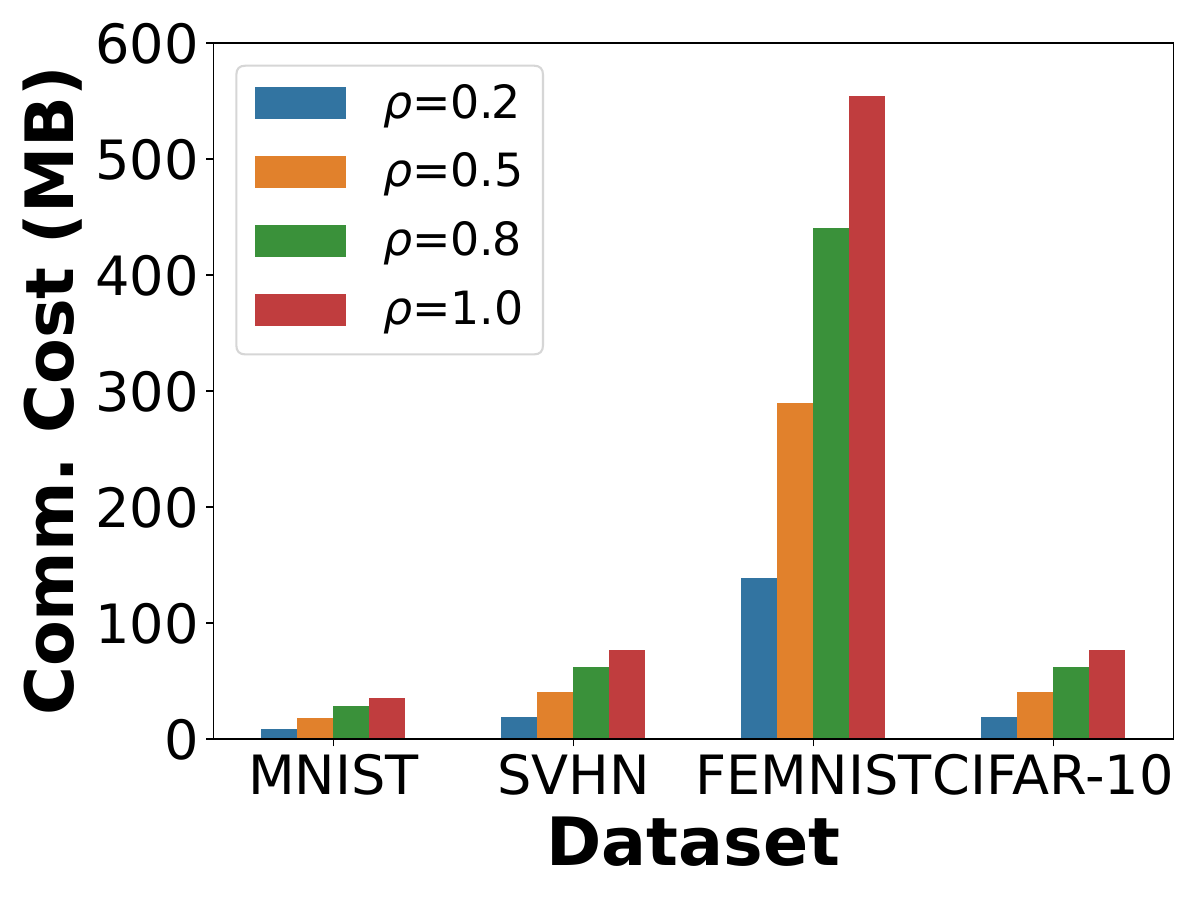}
    \end{minipage}
    
    \caption{Effect of the stability level $\rho$ on F-SOSA’s accuracy and time/ communication overhead across datasets.}
    \label{fig:my-figure2}
\end{figure*}

\begin{figure}[t]
    \centering
    \includegraphics[width=0.23\textwidth]{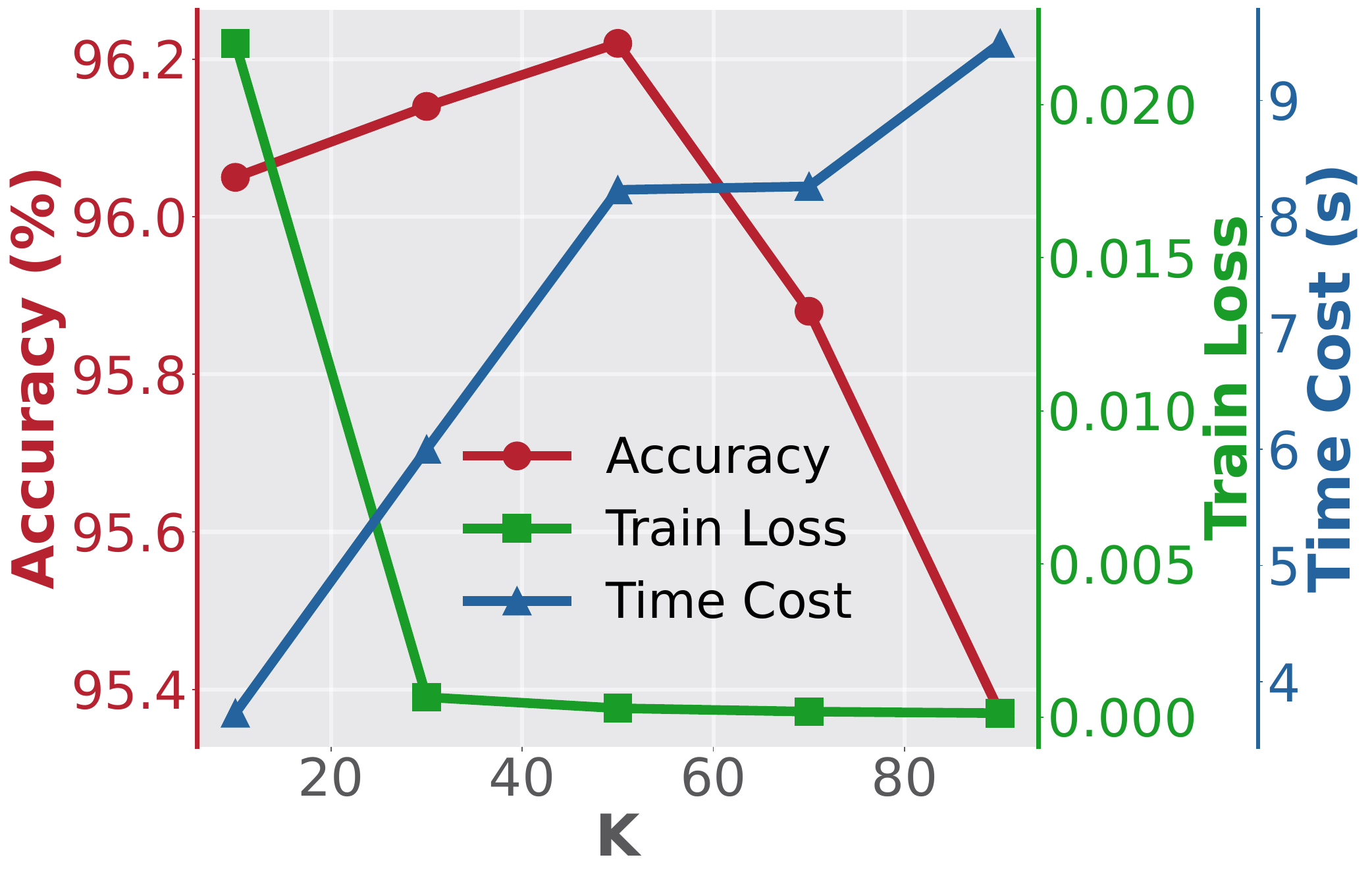}
    \hfill
    \includegraphics[width=0.23\textwidth]{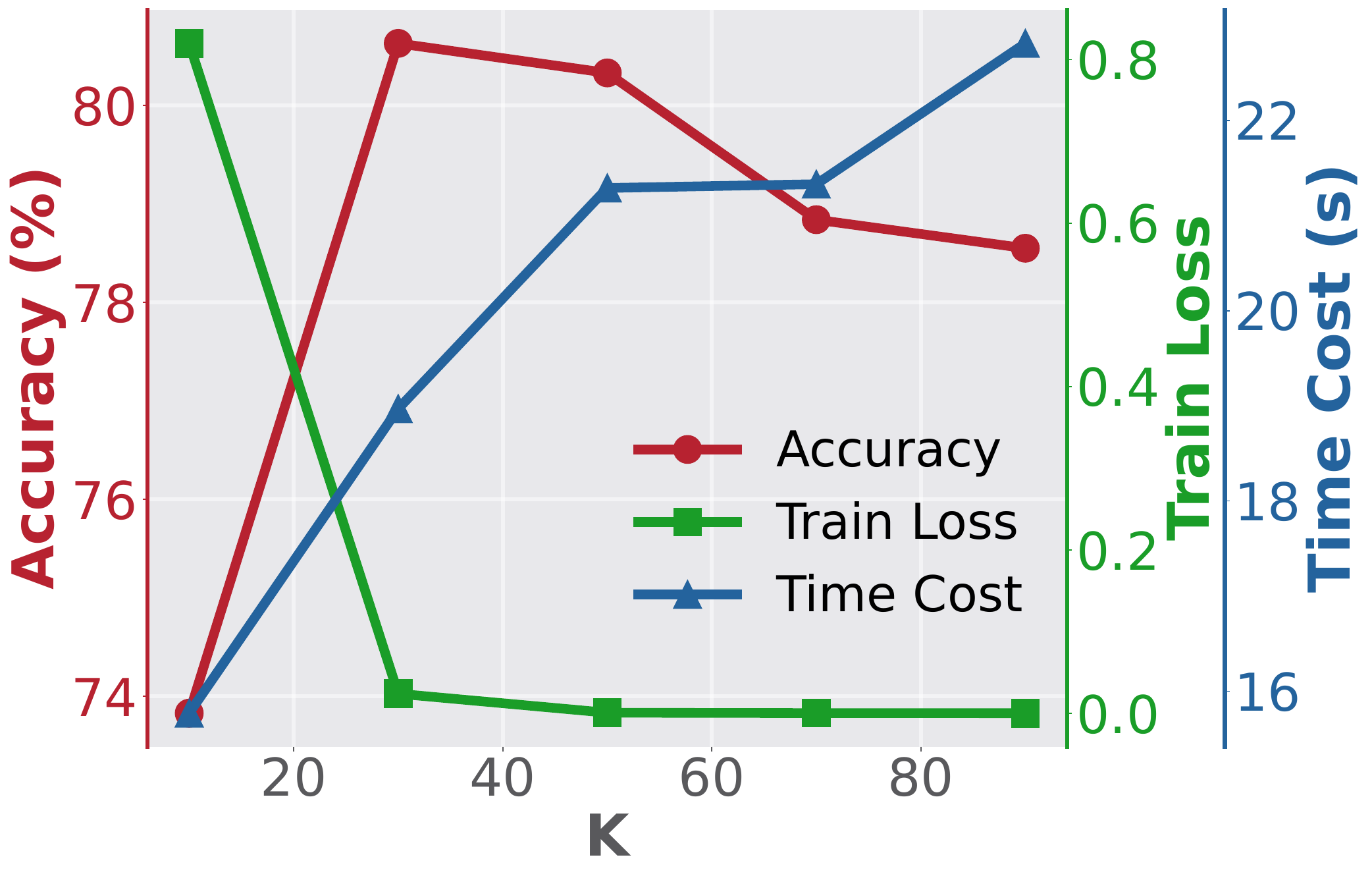}
    \\
    \includegraphics[width=0.23\textwidth]{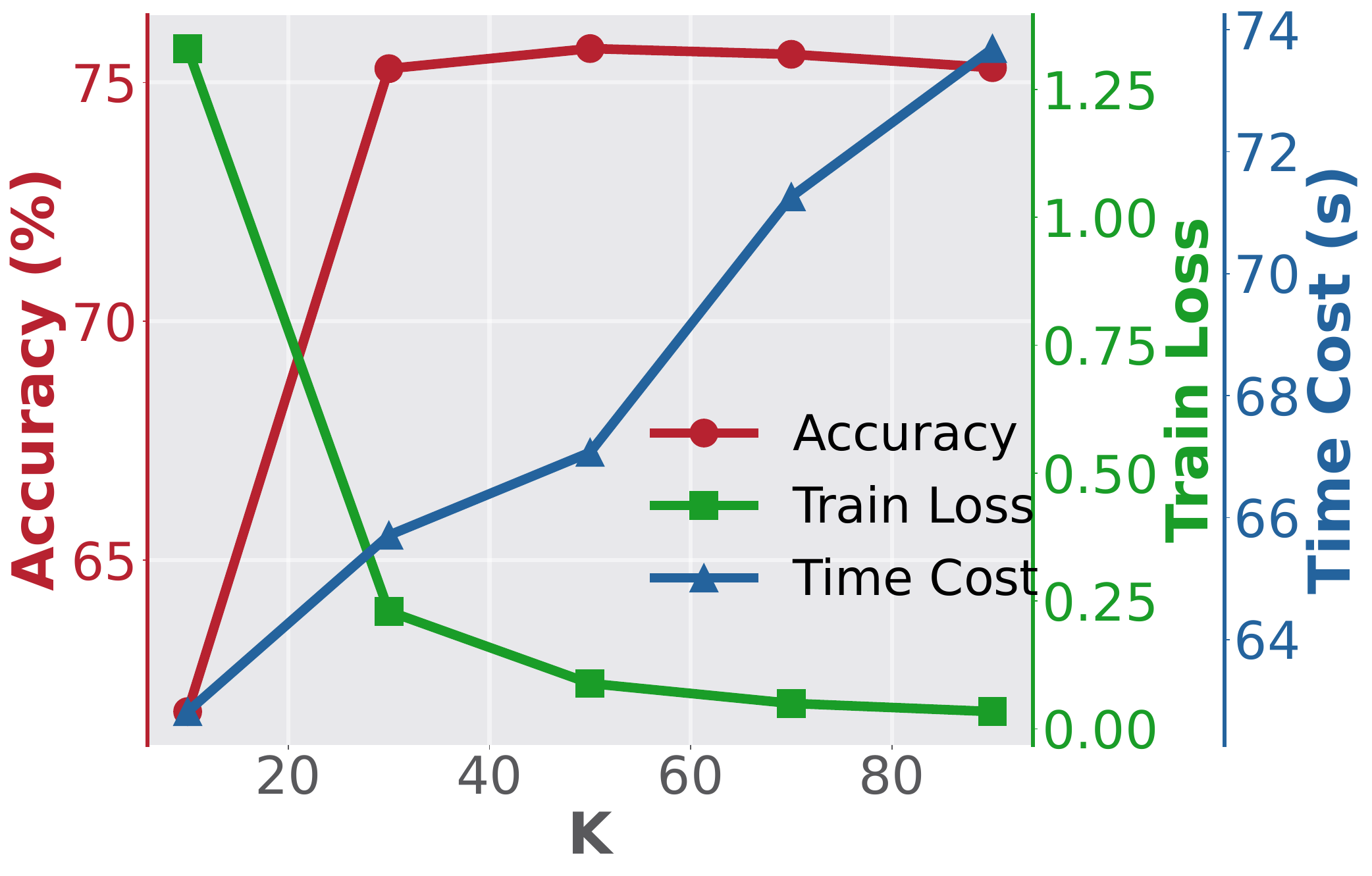}
    \hfill
    \includegraphics[width=0.23\textwidth]{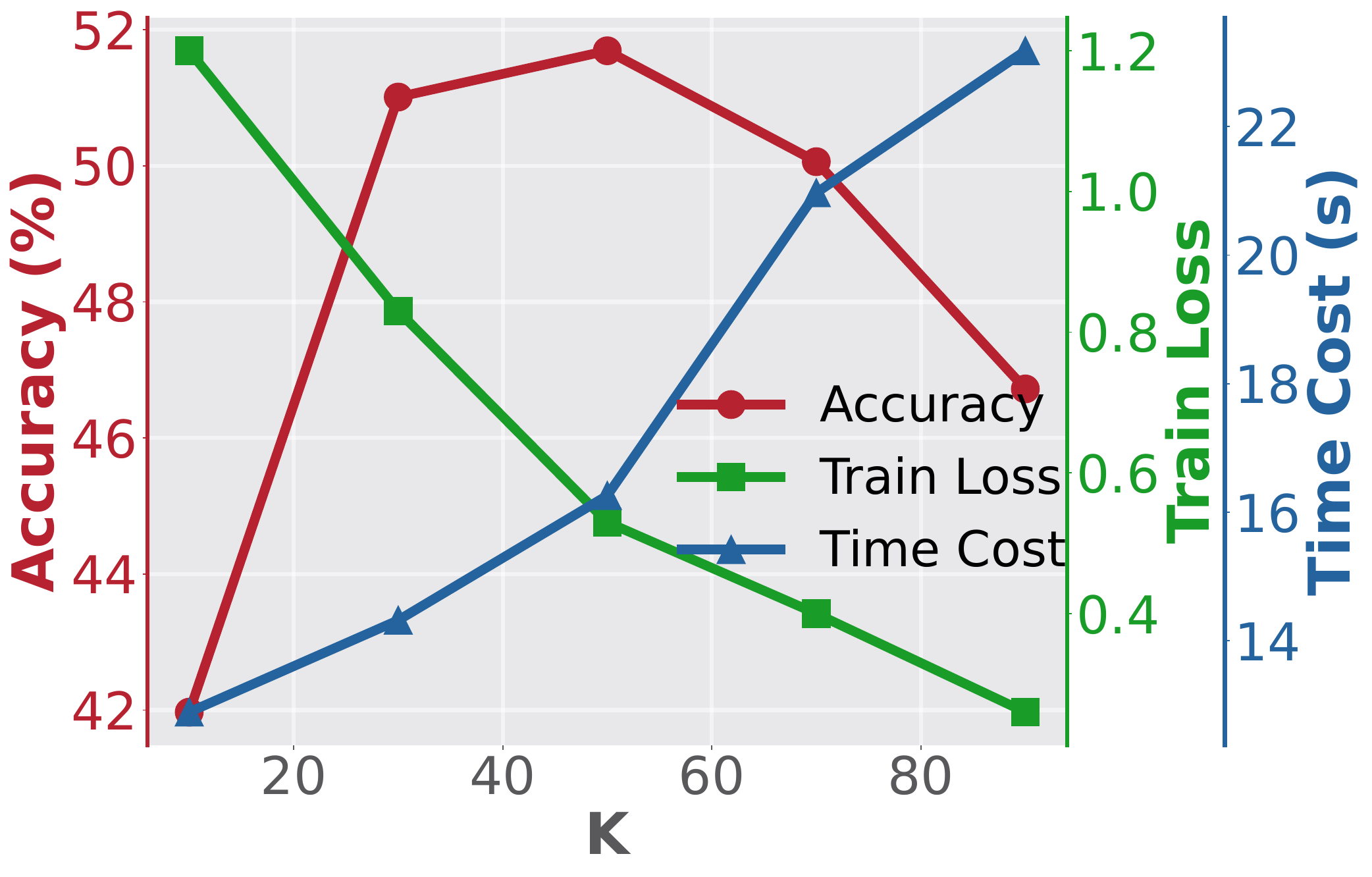}

    \caption{Effect of local epoch number $K$ on performance. Top-left: MNIST, top-right: SVHN, bottom-left: FEMNIST, bottom-right: CIFAR-10.}
    \label{fig:my-figure3}
\end{figure}

\begin{figure}[t]
    \centering
    \includegraphics[width=0.455\linewidth]{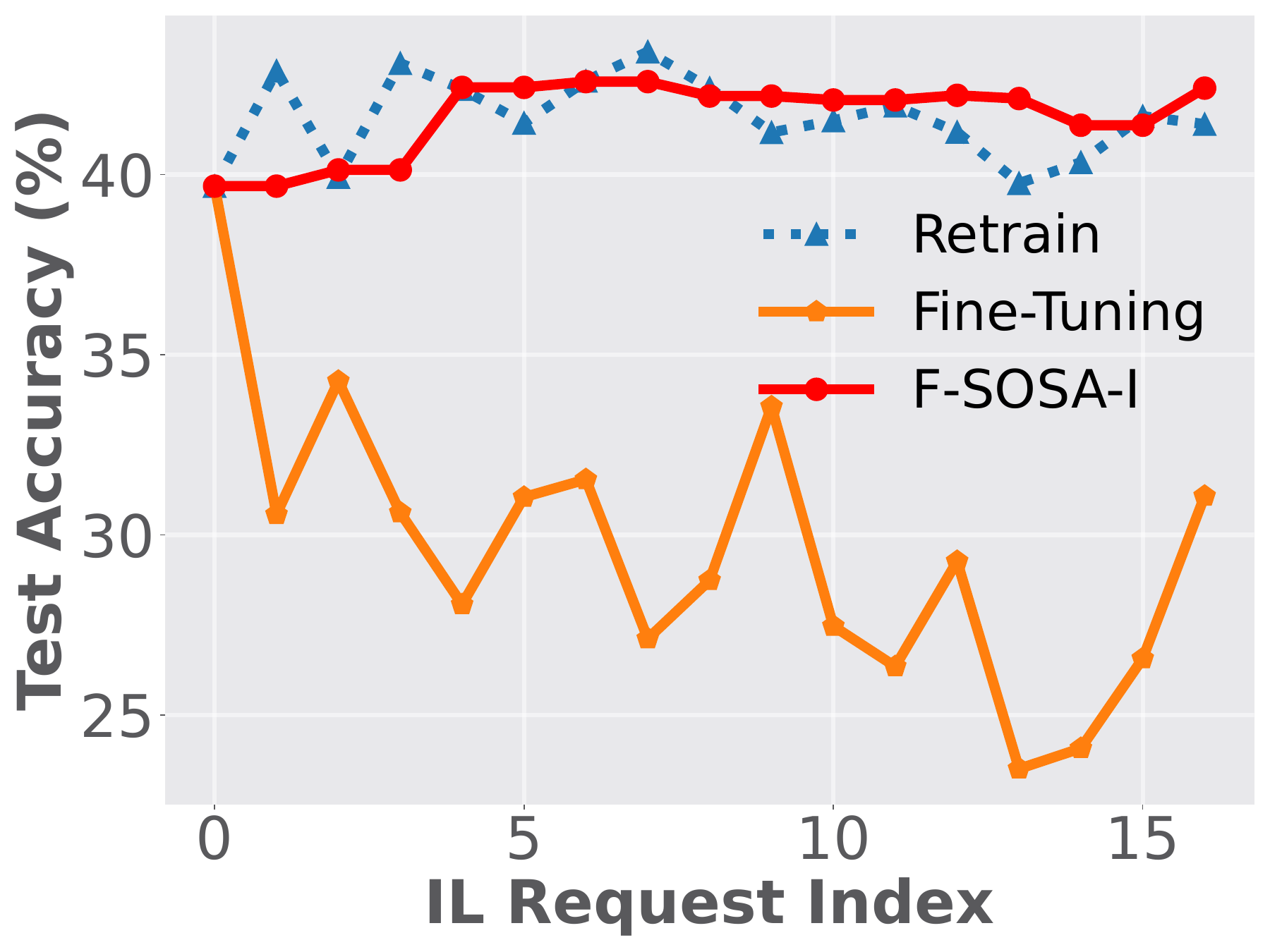}
    \hfill
    \includegraphics[width=0.455\linewidth]{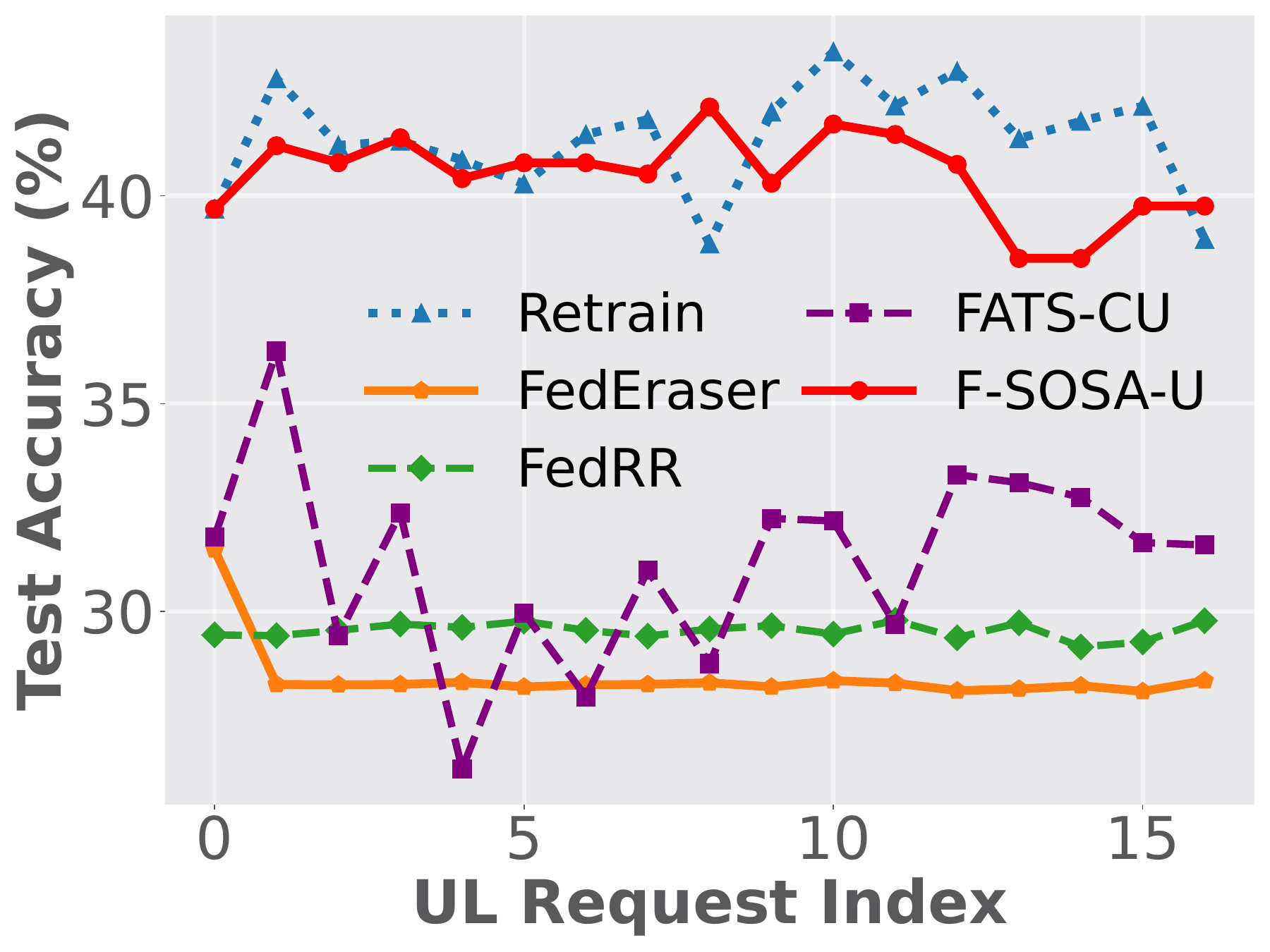}
    \\
    \includegraphics[width=0.455\linewidth]{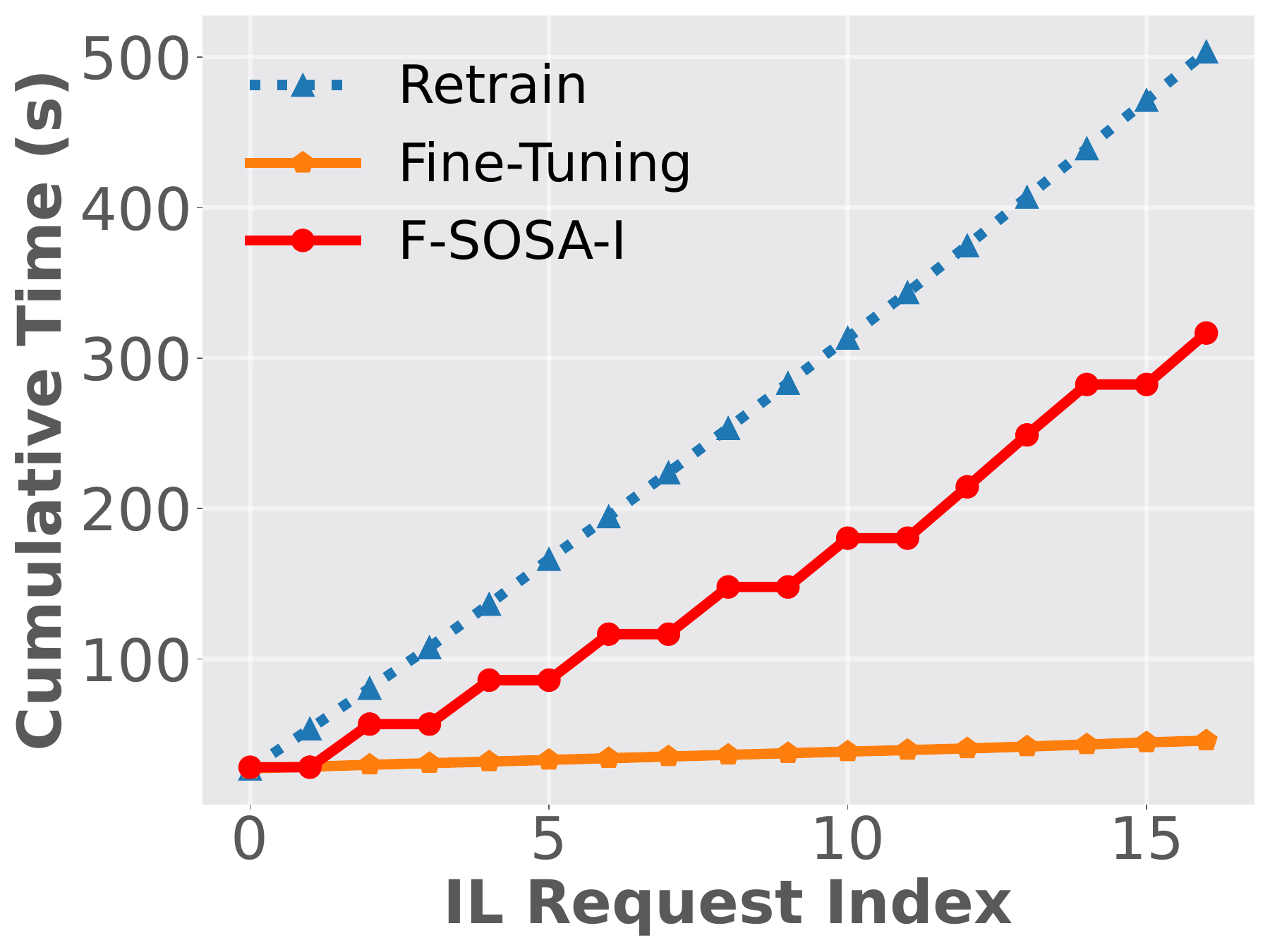}
    \hfill
    \includegraphics[width=0.455\linewidth]{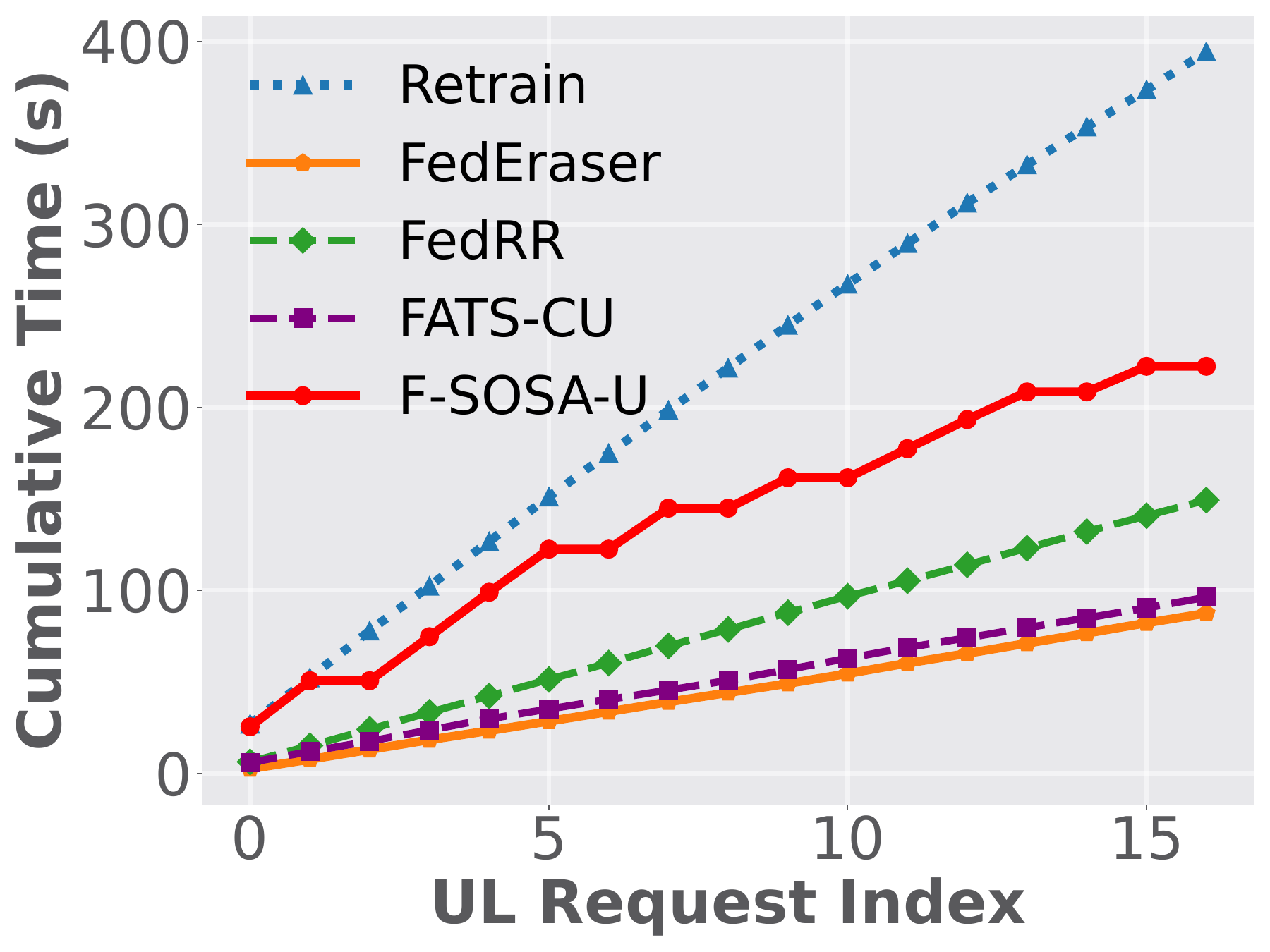}
    \\
    \includegraphics[width=0.455\linewidth]{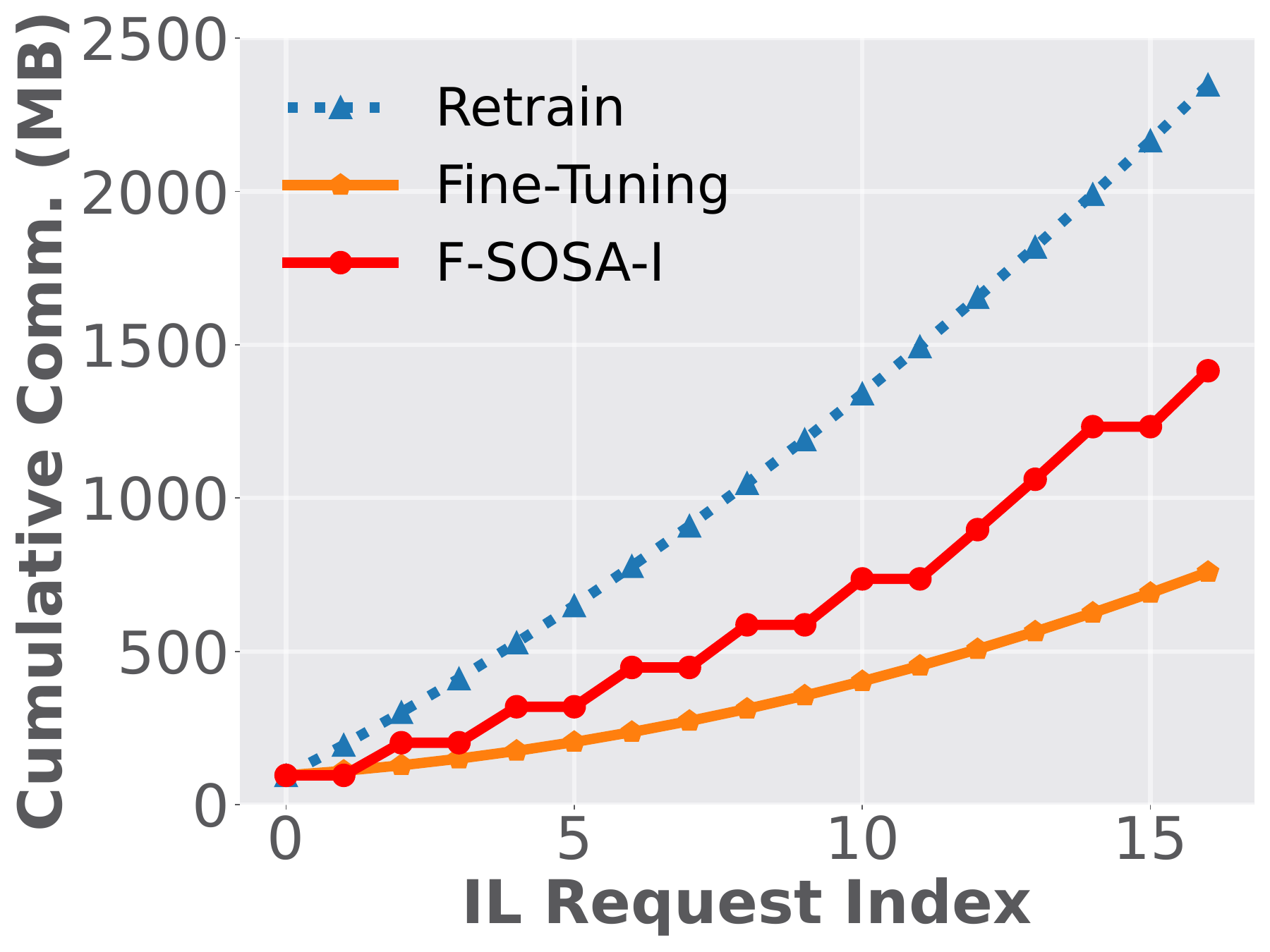}
    \hfill
    \includegraphics[width=0.455\linewidth]{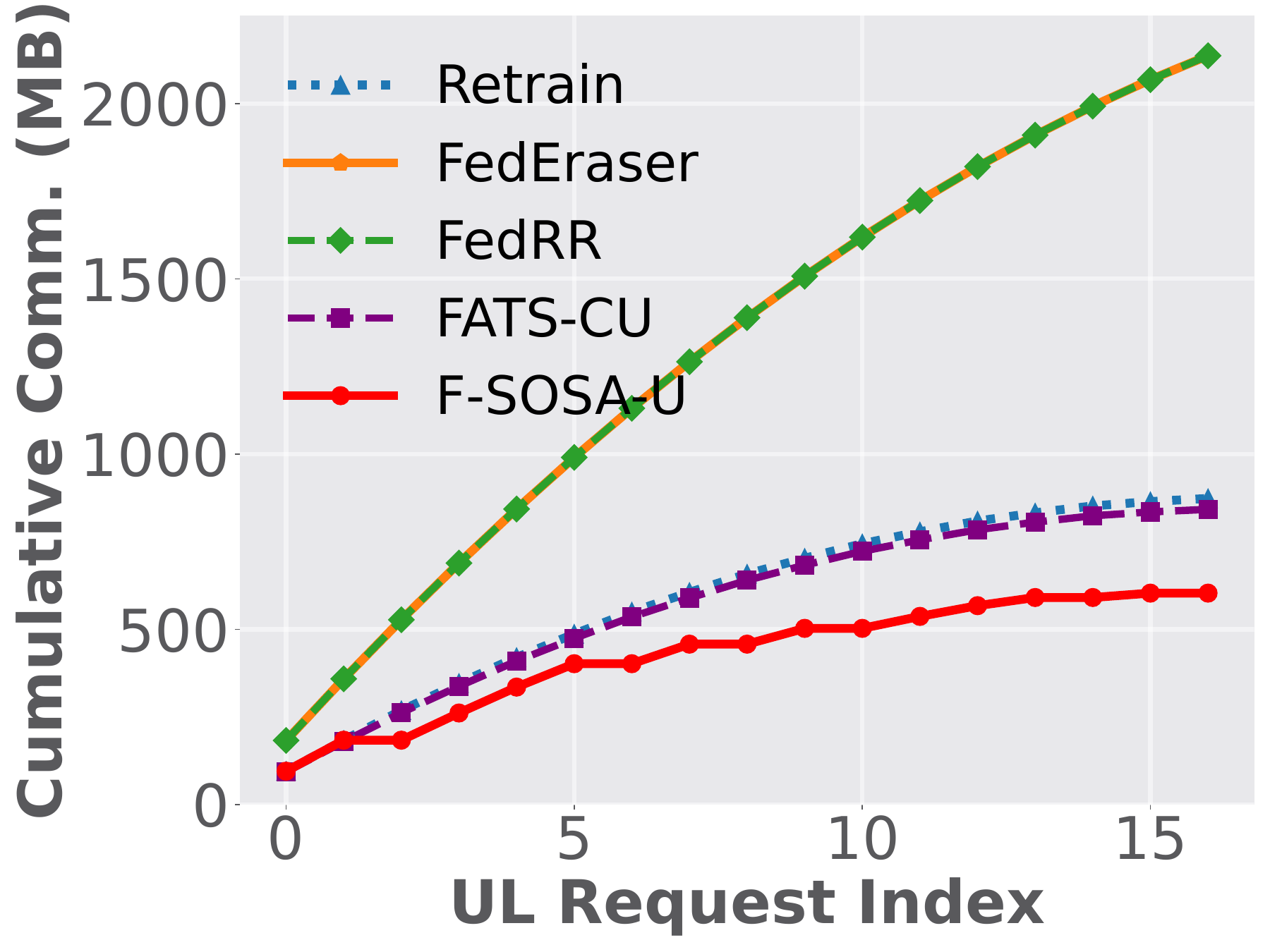}

    \caption{Performance of F-SOSA-I/U vs. baselines on CIFAR-10. Columns: IL (left), UL (right). Rows: test accuracy (top), time cost (middle), communication cost (bottom).}
    \label{fig:my-figure-ulil}
\end{figure}

\begin{figure}[t]
    \centering
    \includegraphics[width=0.46\linewidth]{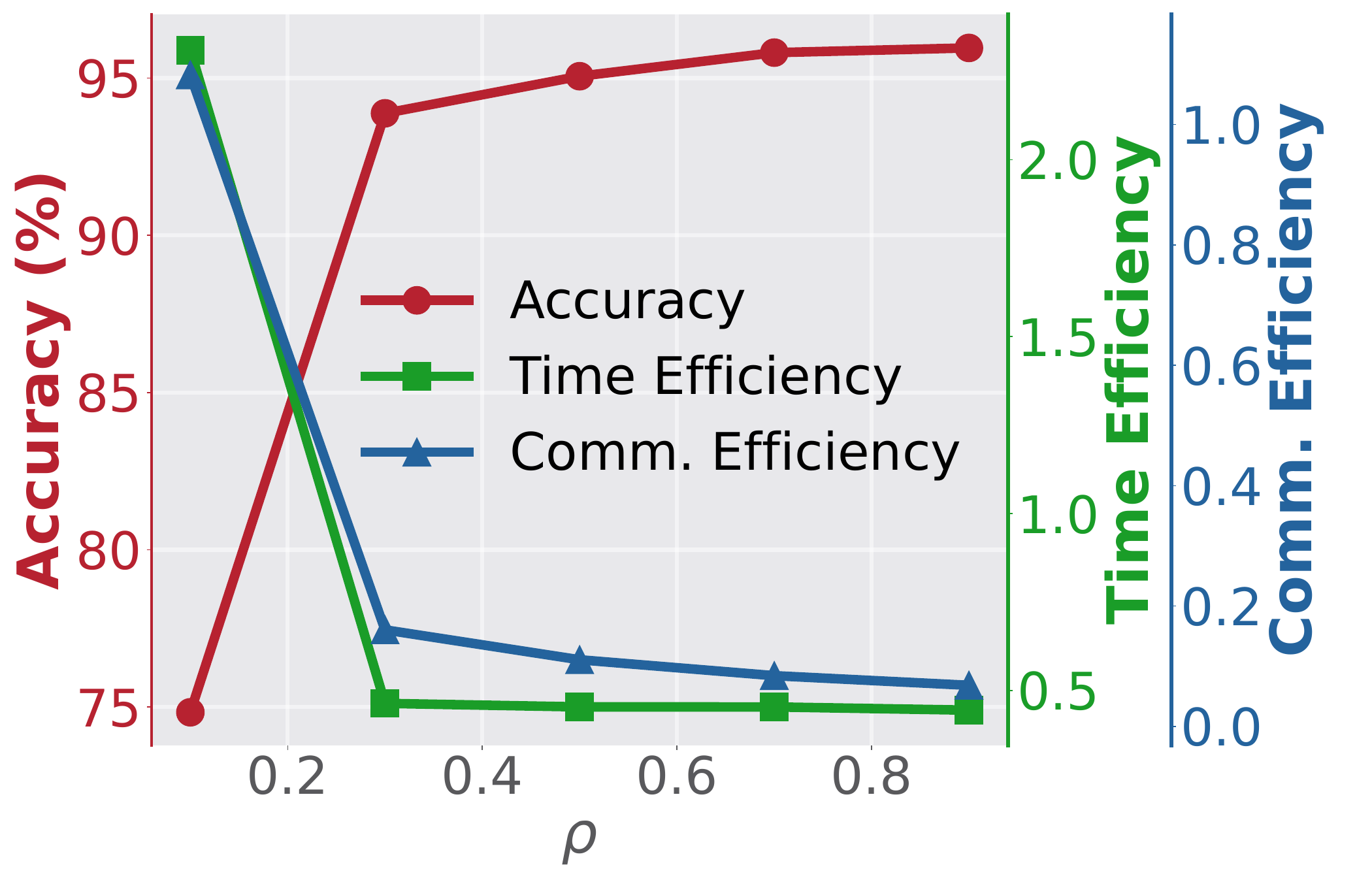}
    \hfill
    \includegraphics[width=0.46\linewidth]{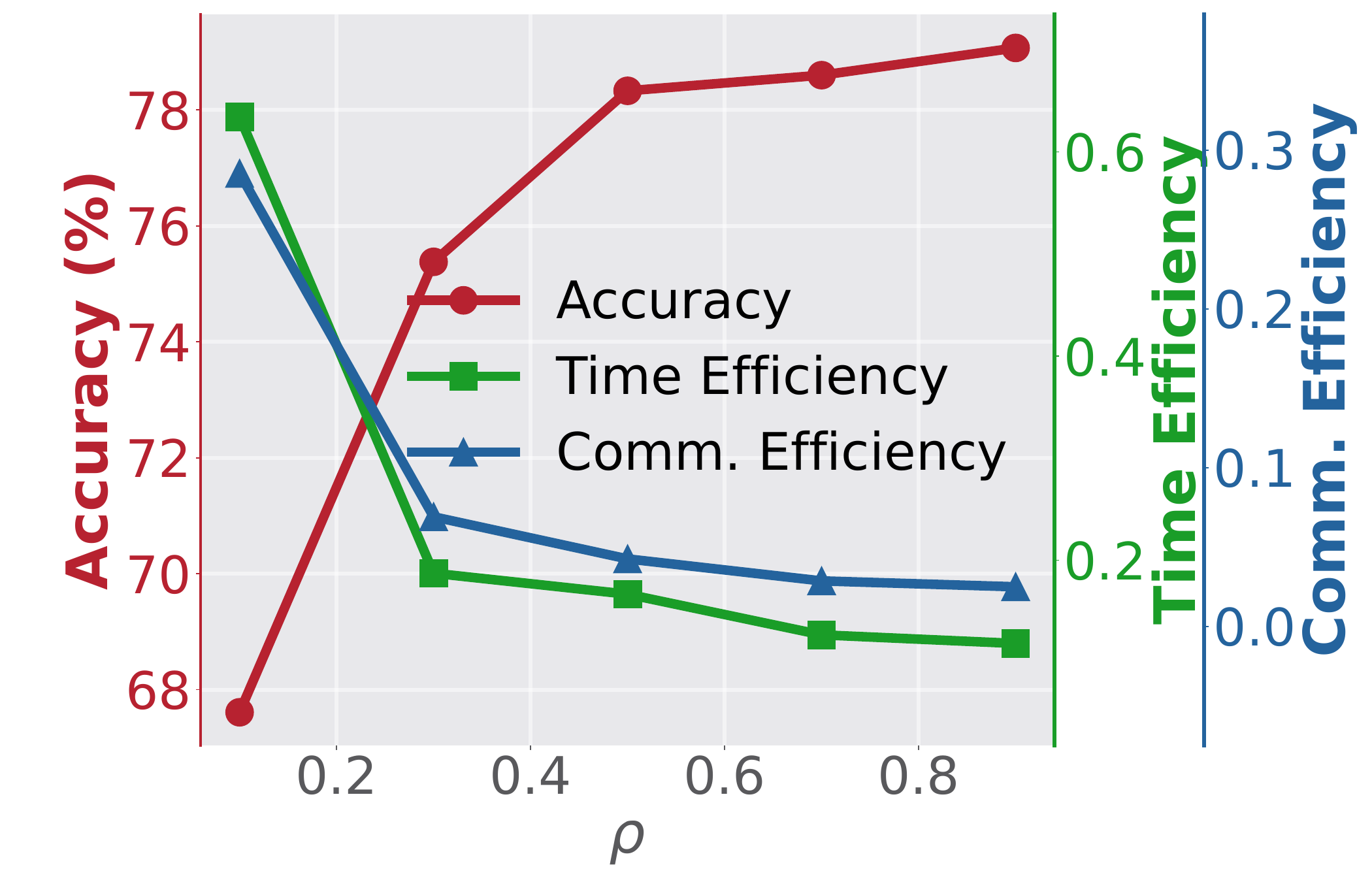}
    \\
    \includegraphics[width=0.46\linewidth]{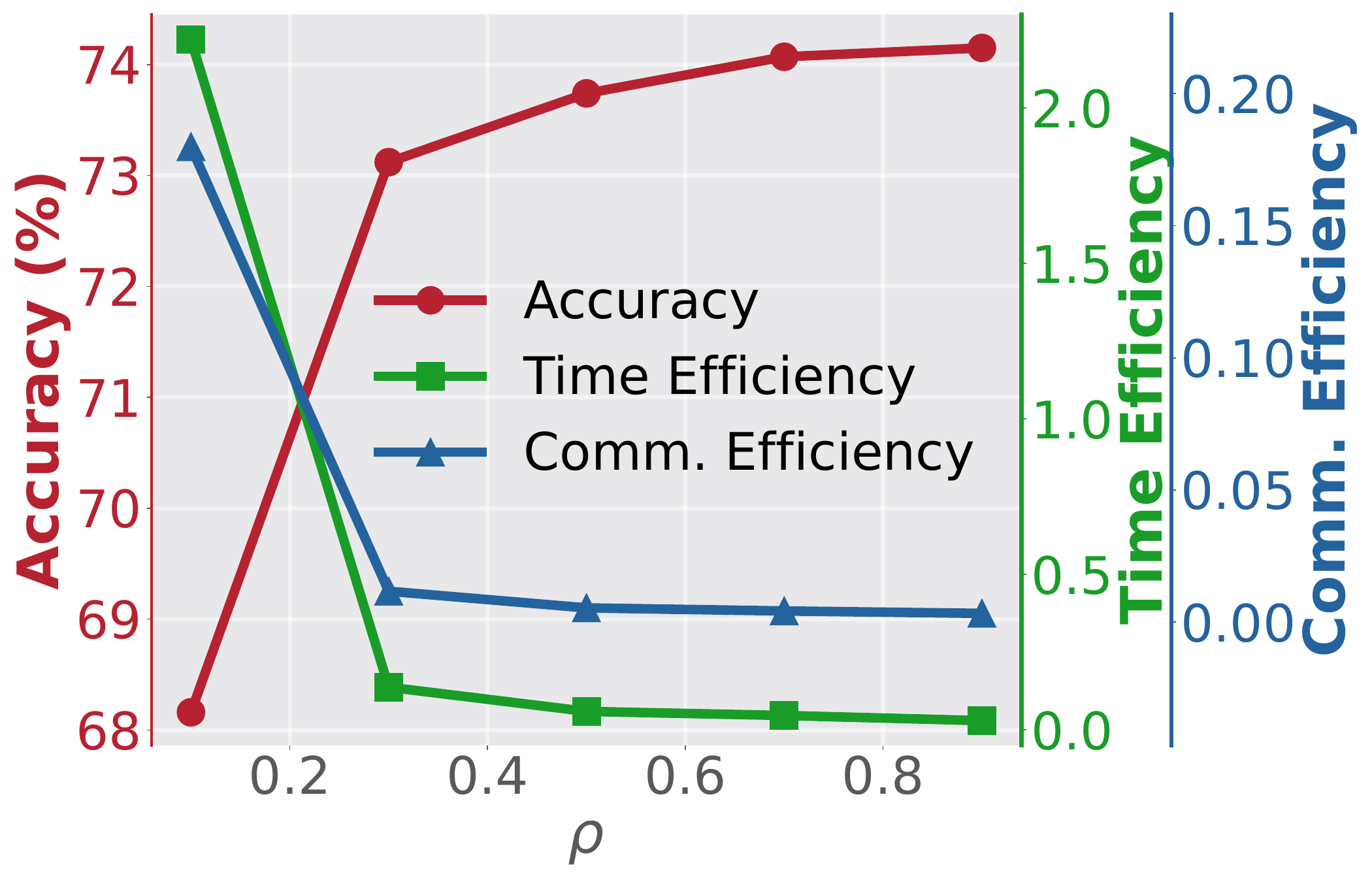}
    \hfill
    \includegraphics[width=0.46\linewidth]{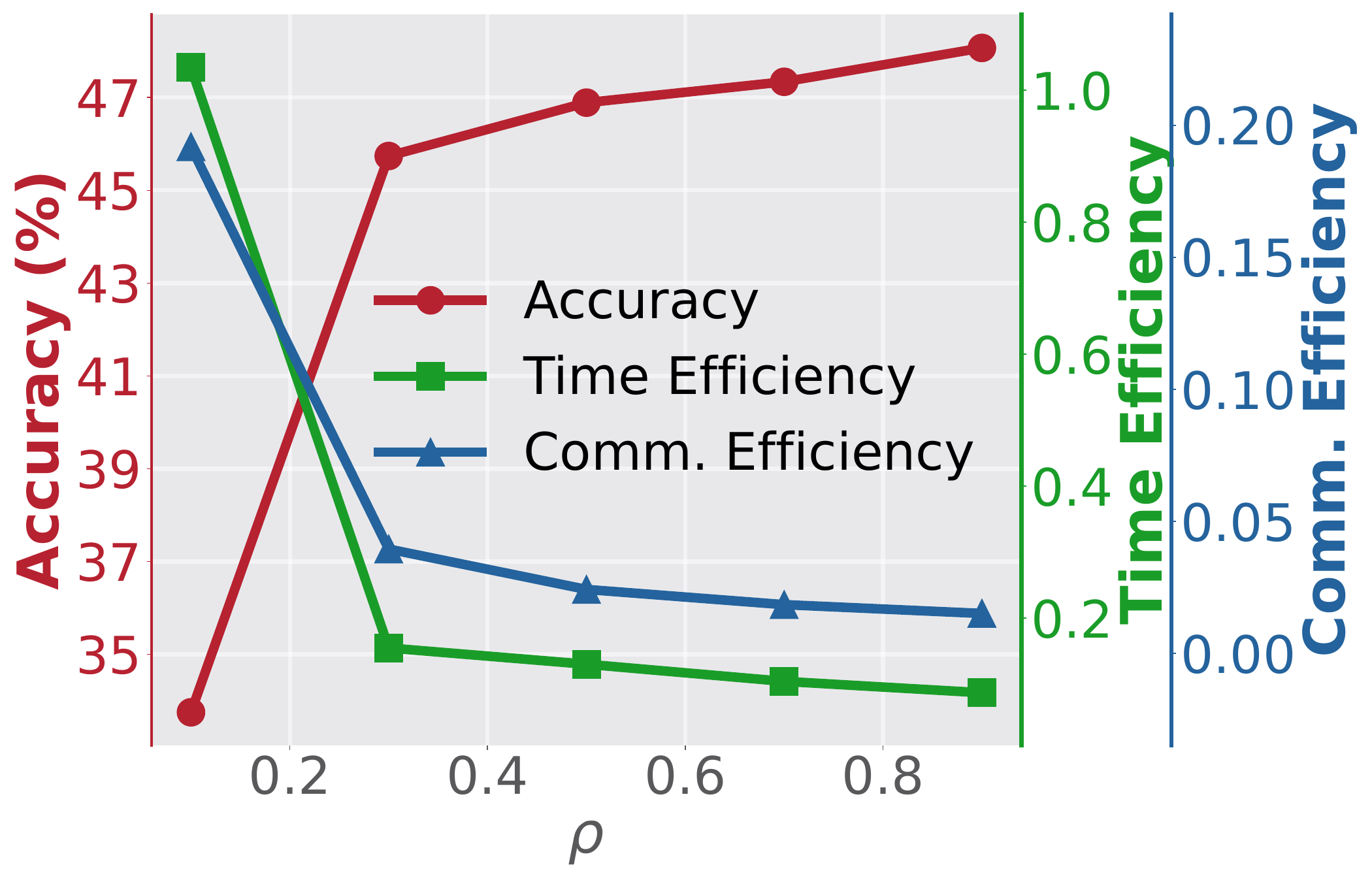}

    \caption{Effect of stability parameter $\rho$ on accuracy and efficiency. Top-left: MNIST, top-right: SVHN, bottom-left: FEMNIST, bottom-right: CIFAR-10.}
    \label{fig:my-figure5}
\end{figure}

\begin{figure}[t]
    \centering
    \includegraphics[width=0.45\textwidth]{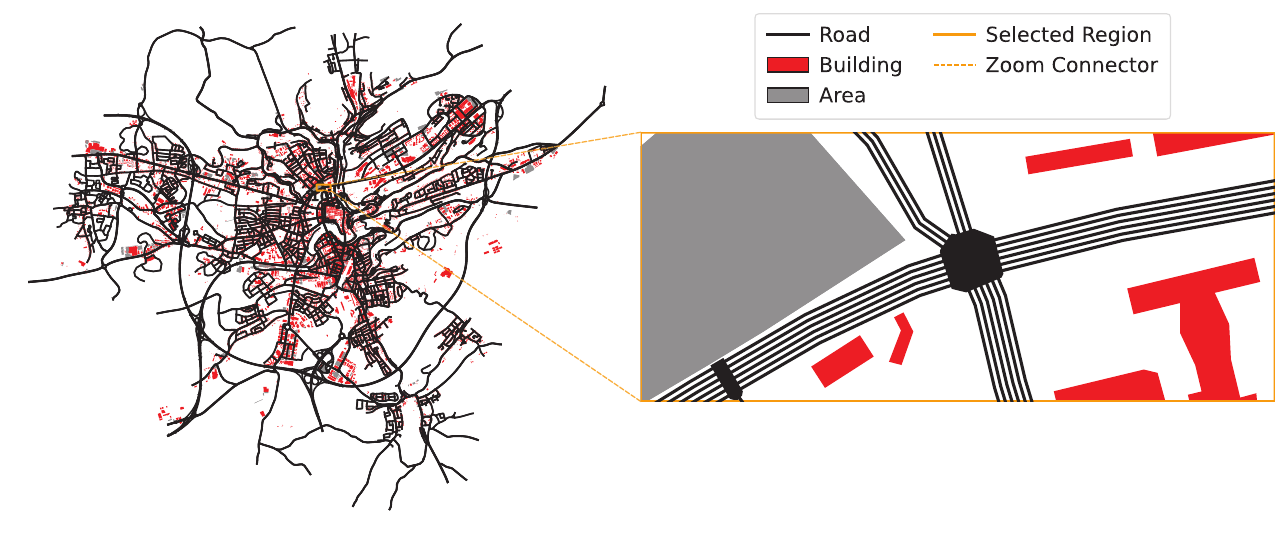}
    \caption{The Luxembourg map and the region considered.}
    \label{fig:simulation}
\end{figure}

\begin{figure}[t]
    \centering
    \includegraphics[width=0.45\linewidth]{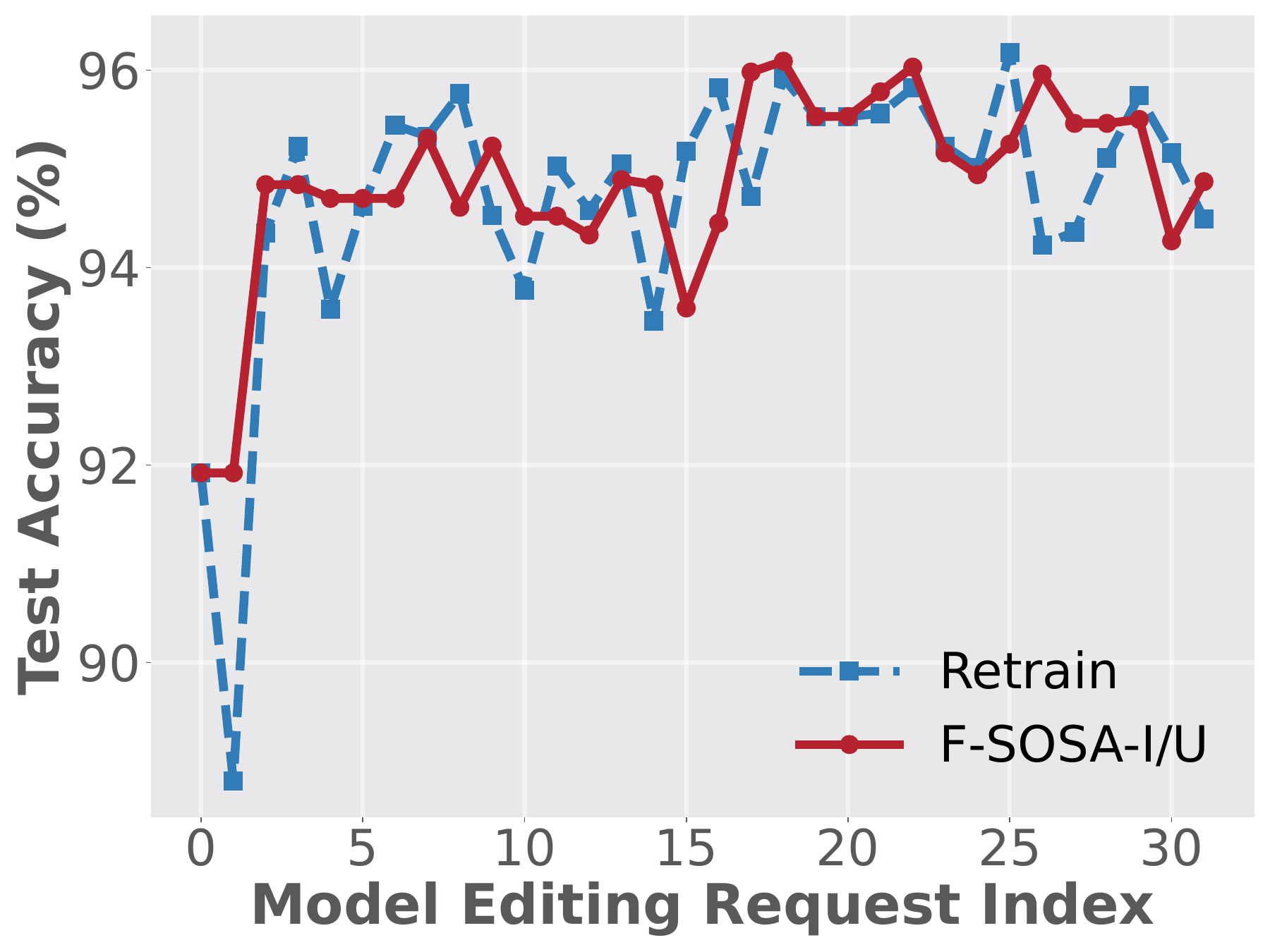}
    \hfill    
    \includegraphics[width=0.48\linewidth]{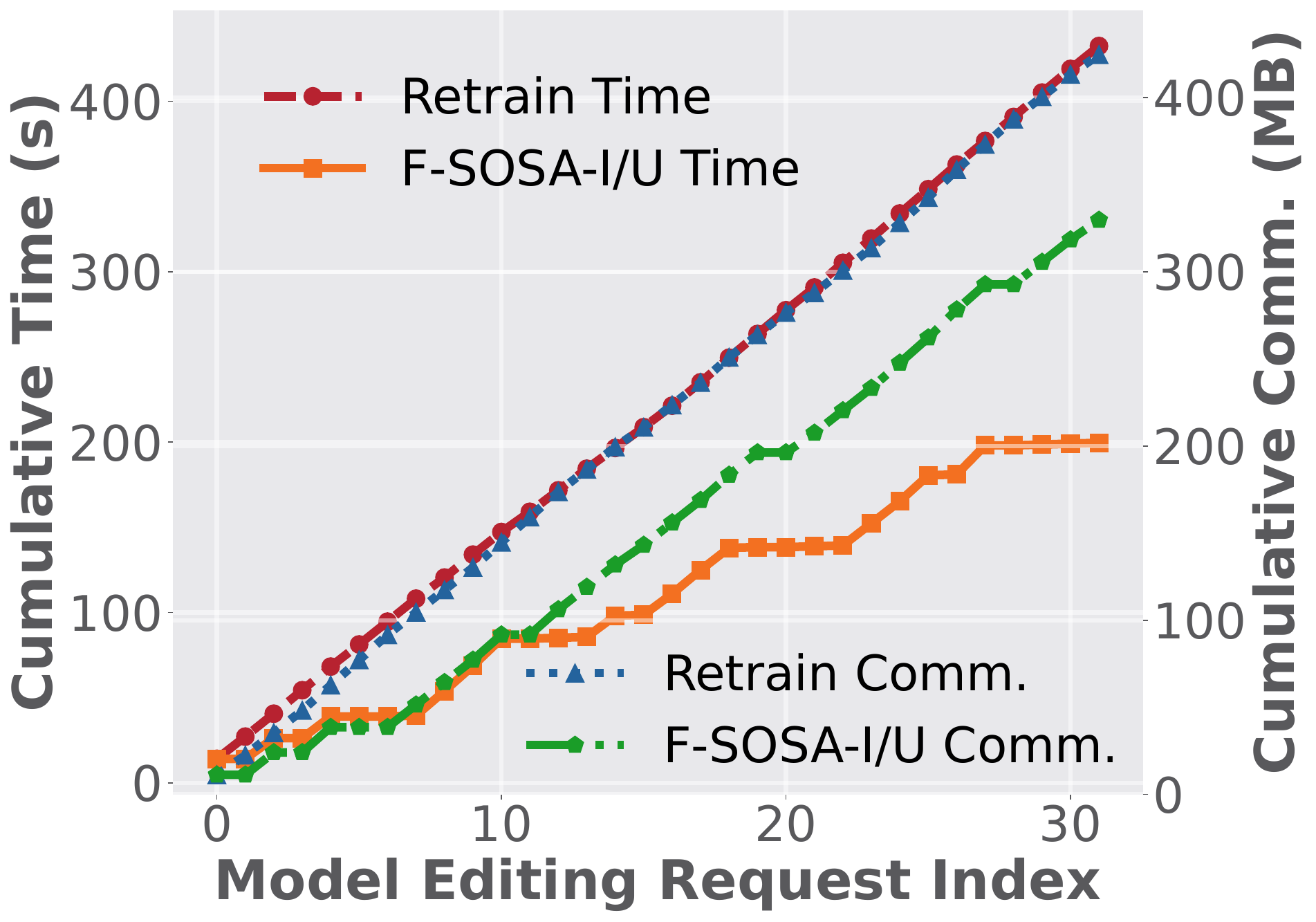}

    \includegraphics[width=0.45\linewidth]{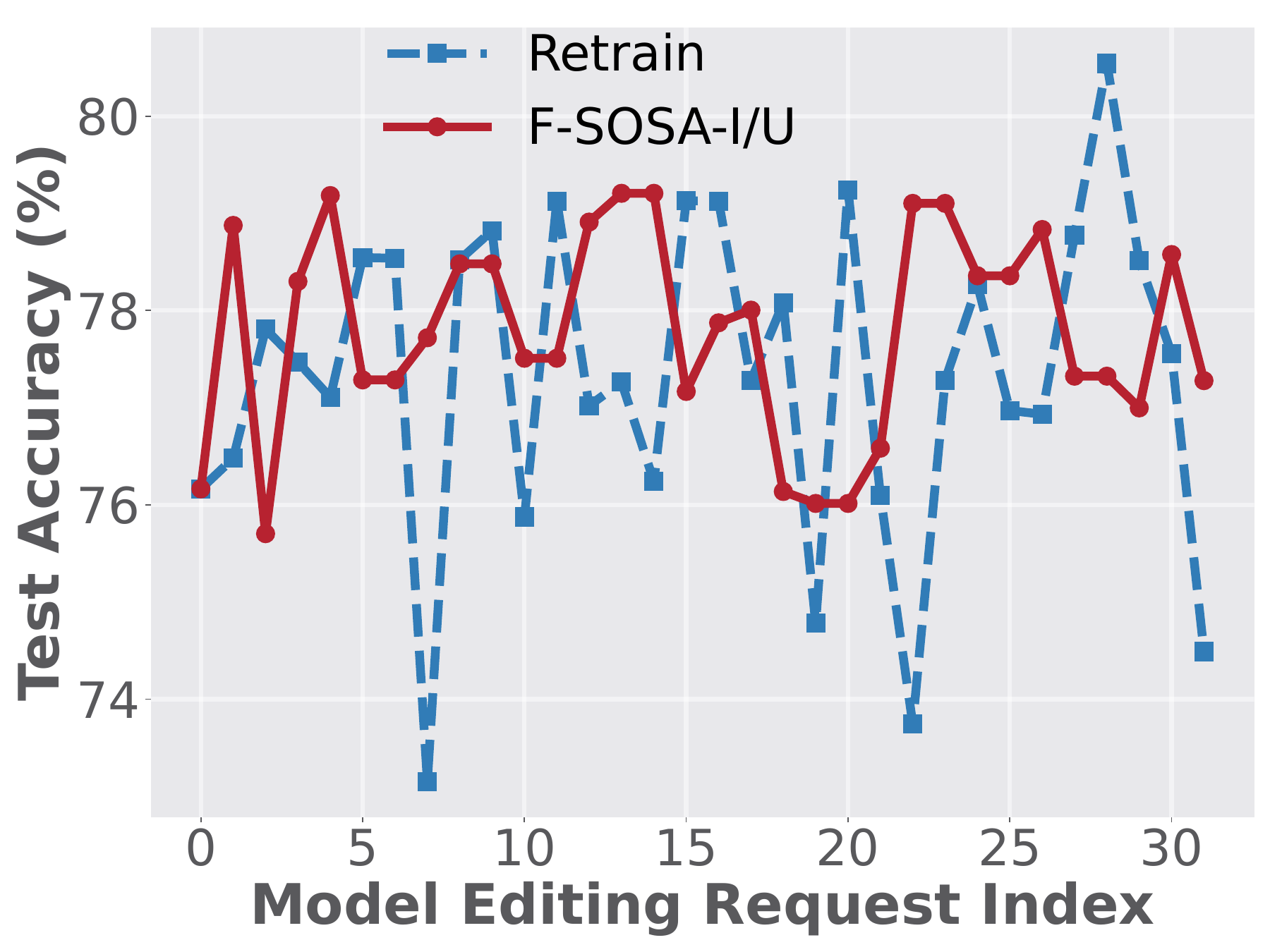}
    \hfill
    \includegraphics[width=0.48\linewidth]{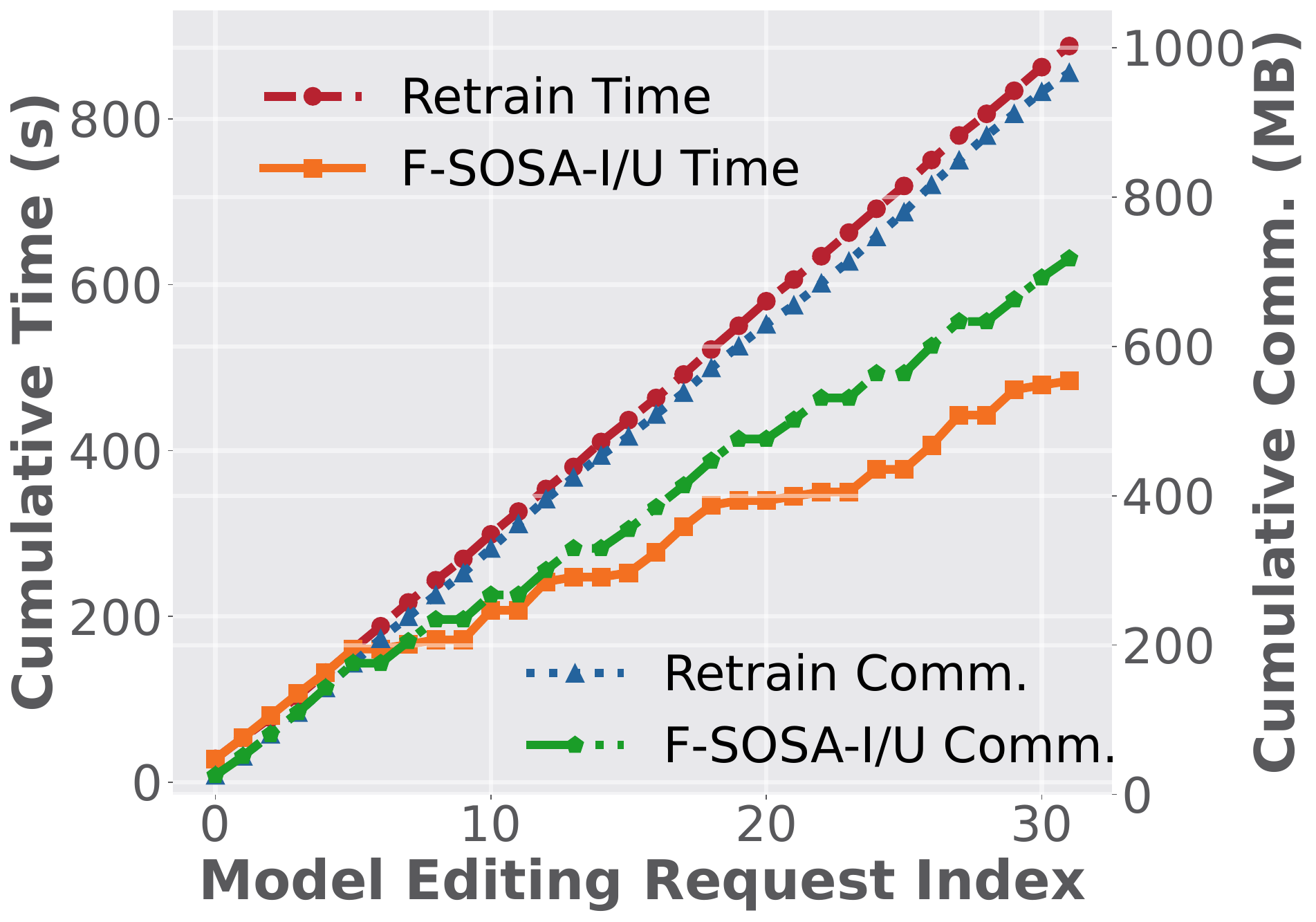}

    \includegraphics[width=0.45\linewidth]{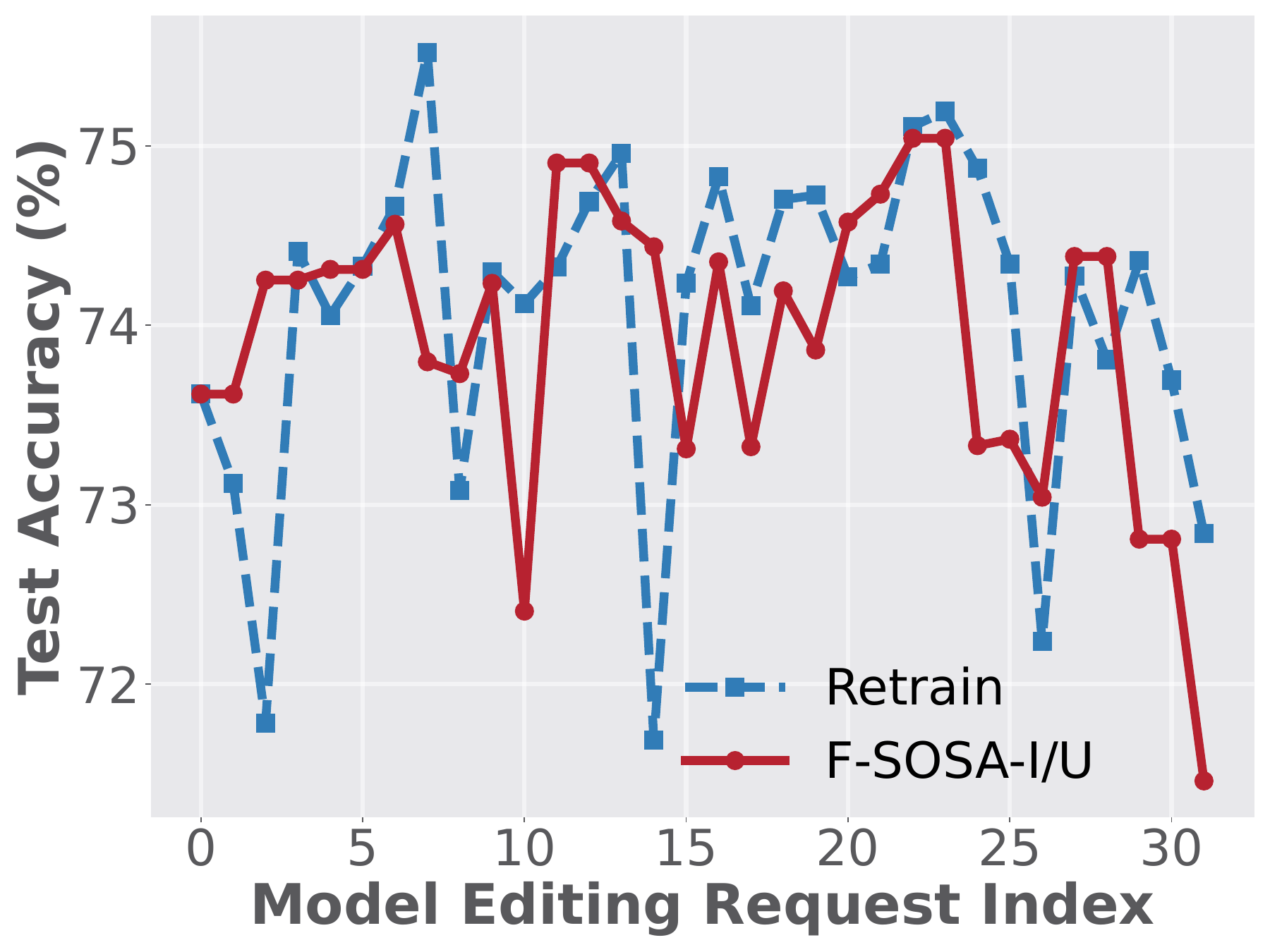}
    \hfill
    \includegraphics[width=0.48\linewidth]{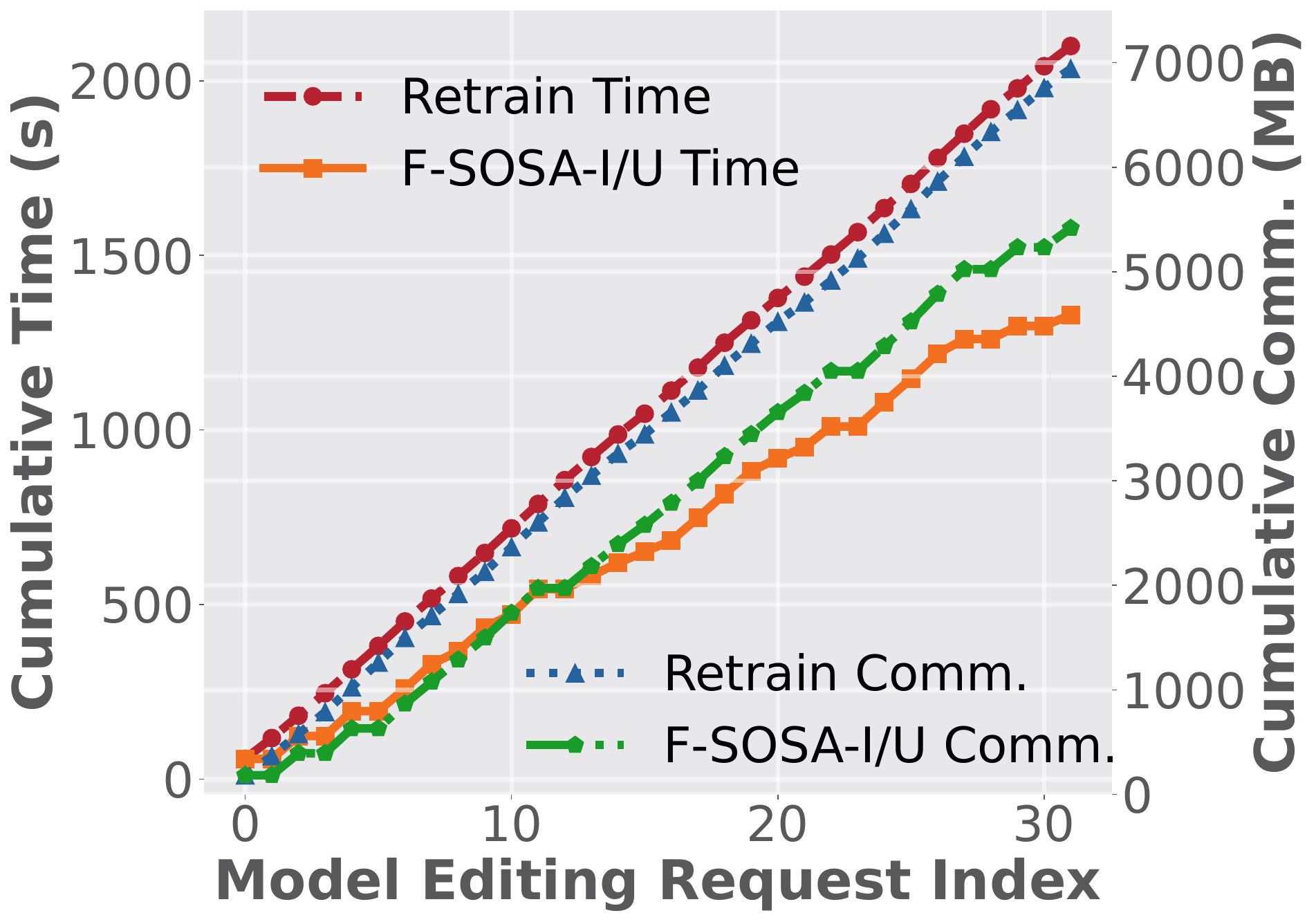}

    \includegraphics[width=0.45\linewidth]{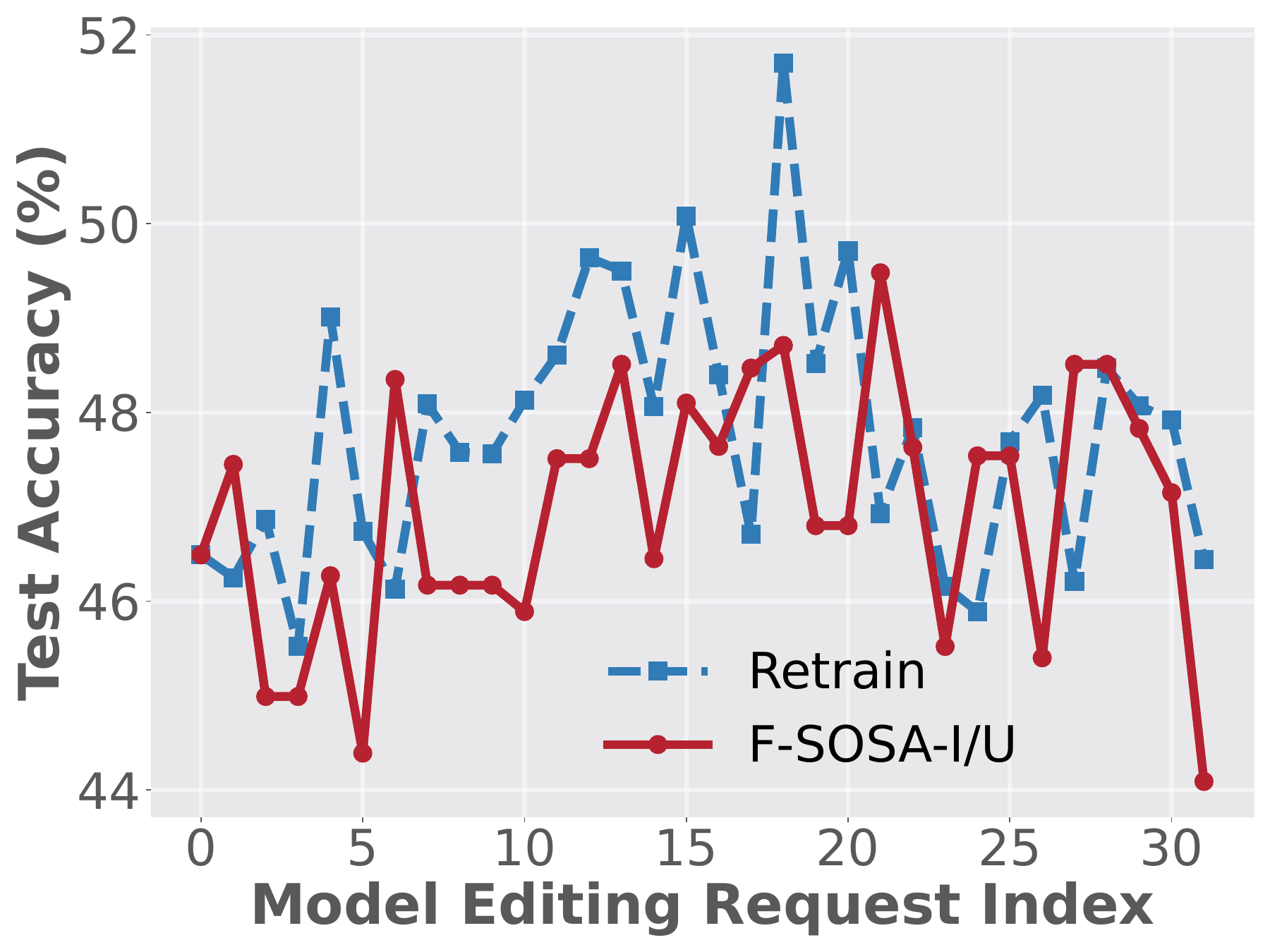}
    \hfill
    \includegraphics[width=0.48\linewidth]{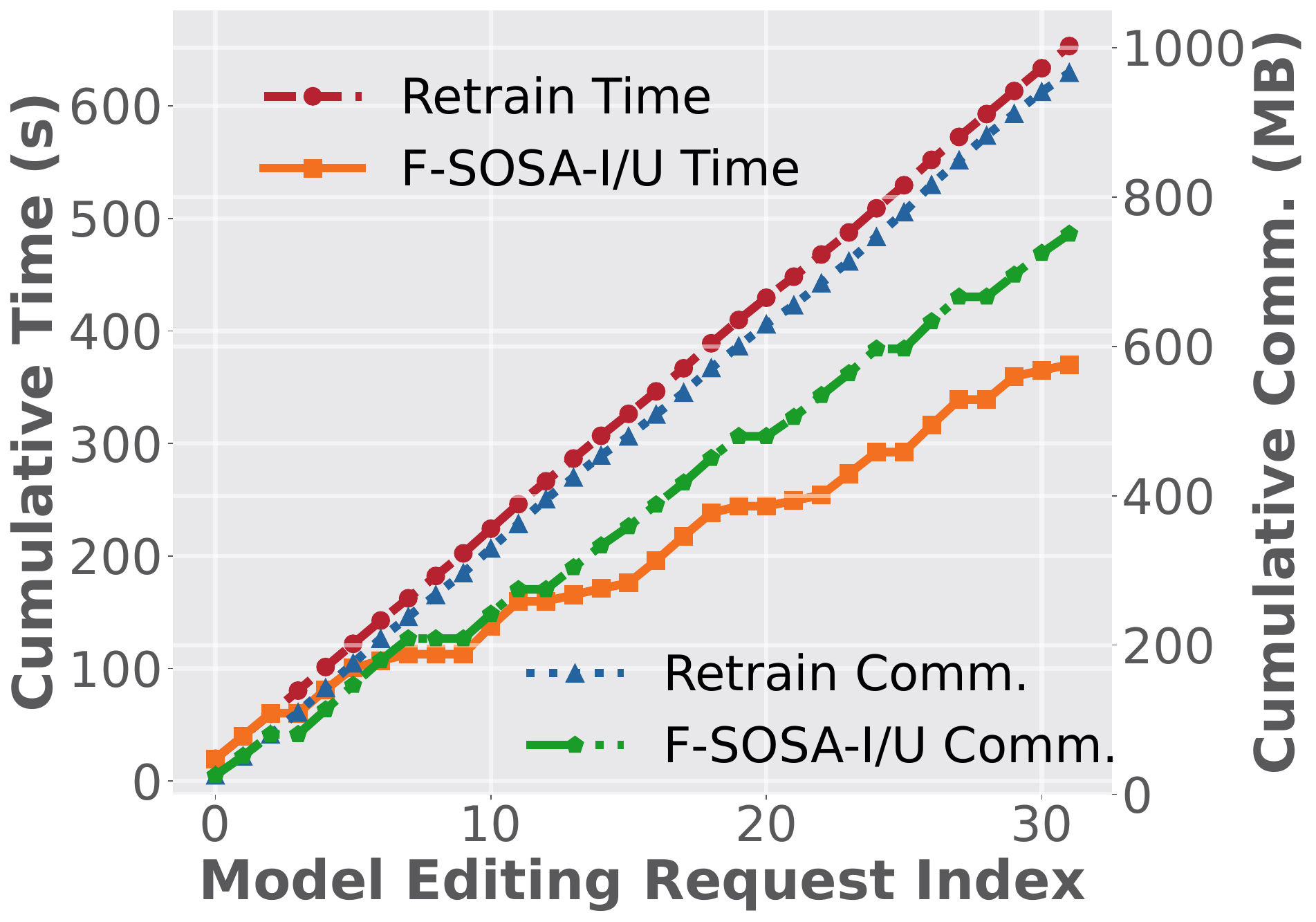}

    \caption{Test accuracy (left) and cumulative cost (right) of F-SOSA-I/U vs. Retrain under simulator-generated model editing requests. Rows: MNIST, SVHN, FEMNIST, CIFAR-10 (top to bottom).}
    \label{fig:my-figure4}
\end{figure}

\section{Experiments}\label{sec-experiments}
We empirically evaluate our framework with focus on: (1) one-shot learning performance, (2) model editing accuracy and efficiency, (3) trade-offs controlled by the stability level $\rho$, and (4) resilience to membership inference attacks. We also include a realistic case study based on vehicle mobility simulation.

\subsection{Experimental Setup}

\subsubsection{Datasets and Models}
We use MNIST~\cite{lecun1998gradient}, SVHN~\cite{netzer2011reading}, CIFAR-10~\cite{krizhevsky2009learning}, and FEMNIST from LEAF~\cite{caldas2018leaf}. For MNIST/SVHN/CIFAR-10, we simulate client heterogeneity using Dirichlet partitioning ($\alpha\in\{0.1, 0.5, 0.8\}$; default $\alpha{=}0.5$). FEMNIST, as a FL dataset, naturally exhibits non-IID characteristics; we follow~\cite{liu2024fedlpa}, sampling 10\% of users and grouping into $M$ clients. Unless otherwise specified, we use the two-layer CNN from~\cite{jhunjhunwala2024fedfisher} with K-FAC Fisher approximation. To evaluate F-SOSA beyond the default classification setting, we additionally consider ResNet18 on CIFAR-10 and YOLOv8s on the real-world driving dataset BDD100K; detailed results are reported in the Appendix.

\subsubsection{Baselines}
We compare F-SOSA with three one-shot FL baselines: FedFisher~\cite{jhunjhunwala2024fedfisher}, which uses client-specific Fisher information for model fusion; FedLPA~\cite{liu2024fedlpa}, which constructs local posteriors via layer-wise Laplace approximation; and Dense~\cite{zhang2022dense}, which performs data-free distillation using GAN-generated samples.
For dynamic-client experiments, we use a fixed-ratio implementation of F-SOSA-I/U to accommodate population changes that may extend beyond the bounded range assumed in the theoretical analysis.
Specifically, we set $P_t=\lfloor\rho M(t)\rfloor$, where $M(t)$ is the current active client population. This keeps the participation fraction approximately constant and adapts the aggregation size to population changes. 
For both unlearning and incremental learning, Retrain reruns F-SOSA on the updated client set using the same current subset size $P_t$.
For unlearning, we additionally compare with FedEraser~\cite{liu2021federaser}, which reconstructs the model using calibrated gradient updates; FedRR~\cite{liu2022right}, which uses diagonal Fisher information with momentum correction; and FATS-CU~\cite{tao2024communication}, which stabilizes FedAvg for rapid unlearning. For incremental learning, we also compare with Fine-Tuning, which adapts the global model on the new client's data using a tuned number of epochs selected from $\{5,10,\dots,50\}$. All methods are evaluated under the one-shot constraint for fair comparison.

\subsubsection{Hyperparameters}
We follow the settings in~\cite{jhunjhunwala2024fedfisher}, using learning rate $\eta=0.01$ and default local epoch number $K=50$. 
For learning experiments, we set $M=20$ clients; for model editing, we use at most $100$ clients and initialize with $50$ active clients, while the scalability experiments in the Appendix consider client populations from $20$ to $500$. The stability level is set to $\rho=0.5$ unless otherwise specified.

\subsubsection{Evaluation Metrics}
We report test accuracy, training loss, total wall-clock time, and total communication bytes, including global-model broadcasts, local-model uploads, and Fisher-information uploads. Time and communication efficiency measure the average accuracy gain per unit of time and communication, respectively.

\subsection{Experimental Results}

\subsubsection{Performance of F-SOSA in Heterogeneous Settings}

We evaluate F-SOSA under varying data heterogeneity levels ($\alpha \in \{0.1, 0.5, 0.8\}$) on MNIST/SVHN/CIFAR-10, and include FEMNIST as an inherently heterogeneous FL dataset. We compare against three state-of-the-art one-shot FL baselines. As shown in Table~\ref{tab:my-table1} and Table~\ref{tab:my-table2}, F-SOSA consistently achieves the lowest communication cost across all datasets and heterogeneity levels, leading to the highest communication efficiency. This benefit arises from its client subsampling mechanism, which reduces aggregation overhead. Time cost is also substantially lower, often by over $50\%$ compared to FedLPA and Dense, while remaining comparable to FedFisher in most cases. In terms of accuracy, F-SOSA matches or exceeds FedFisher under moderate heterogeneity ($\alpha = 0.5, 0.8$), and consistently outperforms Dense and FedLPA. Under high heterogeneity ($\alpha = 0.1$), its performance remains robust with minimal degradation. These results demonstrate that F-SOSA offers strong utility-efficiency trade-offs, maintaining competitive accuracy while substantially reducing communication and computation costs.

\subsubsection{Impact and Selection of Hyperparameters} 

We examine how the stability parameter $\rho$ and local epoch number $K$ affect F-SOSA's utility and efficiency. 
Figure~\ref{fig:my-figure2} shows that increasing $\rho$ improves accuracy, as more clients are aggregated, enhancing generalization. When $\rho = 1.0$, F-SOSA recovers FedFisher and achieves the highest accuracy. However, this comes at a cost of increased time and communication, especially on larger datasets like FEMNIST. A moderate setting ($\rho = 0.5$) provides a strong trade-off between utility and efficiency, and is used by default.
As shown in Figure~\ref{fig:my-figure3}, larger $K$ reduces training loss, indicating better local fitting. However, test accuracy peaks around $K = 30$ or $50$, beyond which overfitting may reduce generalization. Meanwhile, time cost grows linearly with $K$. We adopt $K = 50$ as default, balancing generalization and efficiency across datasets.

\subsubsection{Performance of F-SOSA-I/U}

We evaluate F-SOSA-I/U on CIFAR-10 using a fixed sequence of randomly generated client join and leave requests. All methods are tested on the same sequence for fair comparison. Figure~\ref{fig:my-figure-ulil} presents the results. Note that initial accuracy and overhead differ across unlearning baselines due to their learning methods: FedEraser, FedRR, and FATS-CU use FedAvg, while F-SOSA-U and Retrain use F-SOSA.
F-SOSA-I matches Retrain in accuracy while consistently outperforming Fine-Tuning by up to $18.6\%$. It also exhibits the lowest accuracy fluctuation ($2.28\%$), compared to $3.20\%$ (Retrain) and $9.12\%$ (Fine-Tuning), benefiting from its verification mechanism that avoids unnecessary re-aggregation. In contrast, Fine-Tuning overfits to newly joined clients, degrading generalization. F-SOSA-I reduces the final cumulative communication cost by $39.73\%$ compared with Retrain. Fine-Tuning incurs the lowest communication and time costs, but suffers from significant utility loss. 
F-SOSA-U reaches accuracy comparable to Retrain and surpasses FedEraser/FedRR/FATS-CU by up to $10\%$. This advantage arises from Fisher-aware aggregation. While FedEraser/FedRR/FATS-CU incur slightly lower time costs, their performance drops notably. F-SOSA-U also reduces accuracy fluctuation and maintains high efficiency by avoiding redundant recomputation.

\subsubsection{Trade-off Between Editing Utility and Efficiency w.r.t. Stability Parameter $\rho$}

We study how the stability parameter $\rho$ influences the utility-efficiency trade-off in F-SOSA-I/U across datasets, with results in Figure~\ref{fig:my-figure5}. As $\rho$ increases, accuracy improves due to greater client participation, which enhances generalization. However, this comes at the cost of increased time and communication overhead from larger aggregation sets. Notably, small values of $\rho$ (e.g., 0.1--0.3) yield significant accuracy gains with minimal efficiency loss, suggesting that editing benefits from a small number of well-chosen clients. Around $\rho = 0.5$, F-SOSA-I/U achieves a favorable balance between accuracy and resource usage across all datasets.

\subsubsection{Resistance to Membership Inference Attacks}

We assess unlearning effectiveness via membership inference attacks (MIA) on the global models produced by different methods on FEMNIST and CIFAR-10 (Table~\ref{tab:my-table3}). F-SOSA-U achieves near-random metrics (accuracy, precision, AUC around 50\%), closely matching Retrain across both datasets. This suggests that the model cannot reliably distinguish whether a sample was seen during training, confirming successful unlearning. 
On CIFAR-10, F-SOSA-U shows even smaller deviations from random guessing compared to other baselines, some of which exceed 5\%, indicating a lower risk of memorization leakage. Given its much lower time and communication costs, F-SOSA-U offers an effective and efficient solution for one-shot federated unlearning.

\subsubsection{F-SOSA Performance Across Model Architectures}
Experiments with ResNet18 on CIFAR-10 and YOLOv8s on BDD100K extend the default CNN evaluation to deeper architectures. Across varying data heterogeneity, F-SOSA maintains competitive utility with lower communication and computation costs (Appendix~\ref{app:model}).

\subsubsection{Scalability with the Number of Clients}
We evaluate F-SOSA-I/U with total client populations ranging from $20$ to $500$, using $15$ sequential incremental-learning and $15$ sequential unlearning requests at each scale. Across the evaluated range, F-SOSA-I/U maintain accuracy close to Retrain while consistently reducing communication and computation costs. The efficiency gains remain substantial even with $500$ clients, showing that selective editing continues to provide benefits as the client population grows. Detailed results are reported in the Appendix~\ref{app:scalability}.

\subsection{Case Study Using Simulation of Urban Mobility (SUMO)}
\label{sec:SUMO}
We further evaluate F-SOSA-I/U using the Simulation of Urban MObility (SUMO)~\cite{lopez2018microscopic}, focusing on a selected city center region in the Luxembourg scenario~\cite{codeca2017luxembourg} (see Figure~\ref{fig:simulation}). Editing requests are triggered by vehicle movements during the 08:00 morning rush hour, simulating dynamic client join/leave events. For this case study, we map each vehicle arrival to a new client contribution and each vehicle departure to an unlearning request issued before disconnection, thereby generating a mobility-driven editing sequence.
We apply F-SOSA-I/U and Retrain to the resulting request sequence across four datasets. Figure~\ref{fig:my-figure4} reports the test accuracy after each request and the cumulative time/communication costs. F-SOSA-I/U match the accuracy of Retrain while maintaining smaller fluctuations. On SVHN, for instance, Retrain exhibits swings over 5\%, whereas F-SOSA-I/U keep variations within 3.96\% across all datasets. This robustness stems from client subsampling and Bernoulli sampling, which avoid unnecessary updates when edits do not affect the aggregation set. In terms of efficiency, F-SOSA-I/U reduce cumulative time and communication overhead by $36.66\%$–$53.84\%$ and $21.97\%$–$25.76\%$, respectively, where the reported communication cost includes all model transfers and Fisher-information uploads.
In summary, F-SOSA-I/U deliver model editing with near-optimal accuracy and significantly improved efficiency, showing their practical benefits in dynamic environments.
We further replay the same mobility-driven editing sequence on a YOLOv8s object-detection model trained on BDD100K; the corresponding results are reported in the Appendix~\ref{app:driving_case}.

\section{Conclusions}\label{sec-conclusions}

We present a principled framework for exact model editing in one-shot FL, enabling both client unlearning and incremental learning. At its core is F-SOSA, a stable learning algorithm that supports efficient post-training edits. Together with the editing algorithms F-SOSA-U and F-SOSA-I, our method achieves statistical exactness while preserving model utility and efficiency. We provide theoretical guarantees and comprehensive experiments demonstrating its effectiveness across standard benchmarks and dynamic edge scenarios.


\bibliographystyle{IEEEtran}
\bibliography{references}



\newpage
\onecolumn

\appendices
\section{Proofs of Auxiliary Lemmas}
\label{app:aux_proofs}

\begin{proof}[Proof of Lemma~\ref{le-model-dist}]
    Part 1) is directly from \cite[Lemma~10]{jhunjhunwala2024fedfisher}. For Part 2), we first apply convexity of the squared norm to bound $\|\boldsymbol{\bar{W}}_\mathcal{P} - \boldsymbol{W}^{(0)}\|_2^2$ by the average $\|\boldsymbol{\tilde{W}}_i - \boldsymbol{W}^{(0)}\|_2^2$ over $i \in \mathcal{P}$. 
    \begin{flalign*}
        \left\|\boldsymbol{\overline{W}}_\mathcal{P}-\boldsymbol{W}^{(0)}\right\|_2^2\le\frac{1}{\rho M}\sum_{l\in\mathcal{P}}\left\|\boldsymbol{\tilde{W}}_l-\boldsymbol{W}^{(0)}\right\|_2^2,
    \end{flalign*}
    where each summation term can be bounded using Part~1. Expressing the resulting sum via the block-diagonal Gram matrix $\operatorname{diag}(\mathbf{H}_1^{\infty}, \dots, \mathbf{H}_M^{\infty})$, we have
    \begin{flalign*}
        &\left\|\boldsymbol{\overline{W}}_\mathcal{P}-\boldsymbol{W}^{(0)}\right\|_2^2\\
        \le&\frac{2\sum_{l\in\mathcal{P}}\boldsymbol{y}_l^\top(\mathbf{H}_l^{\infty})^{-1}\boldsymbol{y}_l}{\rho M}\!+\!O\left(\!\frac{N^2\sigma^2}{\lambda_0^2\delta^2}\!\right)\!+\!\frac{\operatorname{poly}(N, \lambda_0^{-1}, \delta^{-1})}{m^\frac{1}{2}\sigma}\\
        \le& \frac{2\sum_{i=1}^M\boldsymbol{y}_l^\top(\mathbf{H}_l^{\infty})^{-1}\boldsymbol{y}_l}{\rho M}\!+\!O\left(\!\frac{N^2\sigma^2}{\lambda_0^2\delta^2}\!\right)\!+\!\frac{\operatorname{poly}(N, \lambda_0^{-1}, \delta^{-1})}{m^\frac{1}{2}\sigma}\\
        =&\frac{2\boldsymbol{y}^\top(M\operatorname{dag}(\mathbf{H}_1^{\infty},\mathbf{H}_2^{\infty},\cdots,\mathbf{H}_M^{\infty}))^{-1}\boldsymbol{y}}{\rho}\\
        &+O\left(\frac{N^2\sigma^2}{\lambda_0^2\delta^2}\right)+\frac{\operatorname{poly}(N, \lambda_0^{-1}, \delta^{-1})}{m^\frac{1}{2}\sigma}\\
        \le&\frac{2\boldsymbol{y}^\top(\mathbf{H}^\infty)^{-1}\boldsymbol{y}}{\rho}+O\left(\frac{N^2\sigma^2}{\lambda_0^2\delta^2}\right)+\frac{\operatorname{poly}(N, \lambda_0^{-1}, \delta^{-1})}{m^\frac{1}{2}\sigma}.
    \end{flalign*}
    The proof concludes by the setting of $m$ and $\sigma$.
\end{proof}

\begin{proof}[Proof of Lemma~\ref{le-lse}]
We bound $\|\tilde{f}(\boldsymbol{\bar{W}}_{\mathcal{P}},\boldsymbol{x}_{ij})-\boldsymbol{\tilde{y}}\|_2^2$ by first expanding the difference as a sum of squared inner products and then applying $\|\phi_{\boldsymbol{a}}(\boldsymbol{W}, \boldsymbol{x})\|_2 \le 1$ from Lemma~\ref{le-bounded_phi}.
\begin{flalign*}
    &\|\tilde{f}(\boldsymbol{\bar{W}}_{\mathcal{P}},\boldsymbol{x}_{ij})\!-\!\boldsymbol{\tilde{y}}\|_2^2    =\sum_{i=1}^M\sum_{j=1}^n(\phi_{\boldsymbol{a}}(\boldsymbol{\tilde{W}}_i, x_{ij})^{\top}(\boldsymbol{\bar{W}}_{\mathcal{P}}\!-\!\boldsymbol{\tilde{W}}_i))^2\\
    \le&\!\sum_{i=1}^M\!\sum_{j=1}^n\|\phi_{\boldsymbol{a}}(\!\boldsymbol{\tilde{W}}_i, x_{ij}\!)\|_2^2\|\boldsymbol{\bar{W}}_{\mathcal{P}}\!-\!\boldsymbol{\tilde{W}}_i\|_2^2 
    \!\le\!\sum_{i=1}^M\!\sum_{j=1}^n\|\boldsymbol{\bar{W}}_{\mathcal{P}}\!-\!\boldsymbol{\tilde{W}}_i\|_2^2
\end{flalign*}
Triangle inequality gives a further bound via $\|\boldsymbol{\tilde{W}}_i - \boldsymbol{W}^{(0)}\|_2$ and $\|\boldsymbol{\bar{W}}_{\mathcal{P}} - \boldsymbol{W}^{(0)}\|_2$. 
\begin{flalign*}
    &\|\tilde{f}(\boldsymbol{\bar{W}}_{\mathcal{P}},\boldsymbol{x}_{ij})\!-\!\boldsymbol{\tilde{y}}\|_2^2\\
    \le&2\sum_{i=1}^M\sum_{j=1}^n(\|\boldsymbol{\bar{W}}_{\mathcal{P}}\!-\!\boldsymbol{W}^{(0)}\|_2^2\!+\!\|\boldsymbol{\tilde{W}}_i\!-\!\boldsymbol{W}^{(0)}\|_2^2)
\end{flalign*}
By Lemma~\ref{le-model-dist} and the setting of $m, \eta, \delta$, we obtain
\begin{flalign*}
    &\|\tilde{f}(\boldsymbol{\bar{W}}_{\mathcal{P}},\boldsymbol{x}_{ij})\!-\!\boldsymbol{\tilde{y}}\|_2^2\\
    \le& \sum_{i=1}^M\sum_{j=1}^n O(\frac{\boldsymbol{y}^\top(\mathbf{H}^\infty)^{-1}\boldsymbol{y}}{\rho}+\boldsymbol{y}_i^\top(\mathbf{H}^{\infty}_i)^{-1}\boldsymbol{y}_i)\\
    =& O(N(\frac{\boldsymbol{y}^\top(\mathbf{H}^\infty)^{-1}\boldsymbol{y}}{\rho}\!+\!\boldsymbol{y}^\top\frac{(\operatorname{dag}(\mathbf{H}_1^{\infty},\mathbf{H}_2^{\infty},\cdots,\mathbf{H}_M^{\infty}))^{-1}}{M}\boldsymbol{y}))\\
    =& O\left(\frac{N}{\rho}\boldsymbol{y}^\top(\mathbf{H}^\infty)^{-1}\boldsymbol{y}\right),
\end{flalign*}
\end{proof}

\begin{proof}[Proof of Lemma~\ref{le-H0}]
By Assumption~\ref{asp-degen} and Cauchy’s interlacing theorem, we have $\lambda_{\min}(\mathbf{H}^{\infty}_{\mathcal{P}}) \ge \lambda_0$. For any entry $(k,l)$, 
\begin{equation*}
    \mathbf{H}^{(0)}_{\mathcal{P}}(k,l)\!=\!\frac{\boldsymbol{x}_{k'k''}^\top\boldsymbol{x}_{l'l''}}{m}\!\sum_{r=1}^m \!\mathds{1} \{ \boldsymbol{x}_{k'k''}^\top \boldsymbol{w}^{(0)}_r \!\geq\! 0 \}\mathds{1} \{ \boldsymbol{x}_{l'l''}^\top \boldsymbol{w}^{(0)}_r \!\geq\! 0 \},
\end{equation*}
where $k'=i_{\lceil k/n \rceil}$, $k''=k-\lceil k/n \rceil+1$, $l'=i_{\lceil l/n \rceil}, l''=l-\lceil l/n \rceil+1$. Thus $\mathbf{H}^{(0)}_{\mathcal{P}}(k,l)$ is an average of $m$ bounded independent variables. Applying Hoeffding’s inequality yields 
\begin{equation*}
    | \mathbf{H}^{(0)}_{\mathcal{P}}(k,l) - \mathbf{H}^{\infty}_{\mathcal{P}}(k,l) | \le 4 \sqrt{ \log(\rho N / \delta) / m }.
\end{equation*}
for all $(k,l)$ with probability at least $1 - \delta$. Therefore,
\begin{equation*}
    \|\mathbf{H}^{(0)}_{\mathcal{P}} - \mathbf{H}^{\infty}_{\mathcal{P}}\|_2^2\le\|\mathbf{H}^{(0)}_{\mathcal{P}} - \mathbf{H}^{\infty}_{\mathcal{P}}\|_F^2 \le = 16 (\rho N)^2 \log(\rho N / \delta)/m.
\end{equation*}
Setting $m=\Omega\left(\frac{(\rho N)^2\log(\rho N/\delta)}{\lambda_0^2}\right)$ completes the proof.
\end{proof} 

\begin{proof}[Proof of Lemma~\ref{le-eigen}]
The $(k,l)$-th entry of $\mathbf{\tilde{H}}_{\mathcal{P}}$ is given by 
\begin{equation*}
    \frac{\boldsymbol{x}_{k^\prime k^{\prime\prime}}^\top\boldsymbol{x}_{k^\prime k^{\prime\prime}}}{m} \sum_{r=1}^m \mathds{1}\{\boldsymbol{x}_{kk''}^\top \boldsymbol{\tilde{w}}_{k',r} \ge 0\} \cdot \mathds{1}\{\boldsymbol{x}_{ll''}^\top \boldsymbol{\tilde{w}}_{l',r} \ge 0\},
\end{equation*}
where $k'=i_{\lceil k/n \rceil}$, $k''=k-\lceil k/n \rceil+1$, $l'=i_{\lceil l/n \rceil}, l''=l-\lceil l/n \rceil+1$. We have
\begin{flalign*}
    &\mathbb{E}[|\mathbf{\tilde{H}}_{\mathcal{P}}(k,l)-\mathbf{H}_\mathcal{P}^{(0)}(k,l)|]\\
    =&\mathbb{E}[\frac{|\boldsymbol{x}_{k^\prime k^{\prime\prime}}^\top\boldsymbol{x}_{k^\prime k^{\prime\prime}}|}{m}
    \!\sum\nolimits_{r=1}^m\!\mathds{1}\{\mathds{1}\!\{\boldsymbol{x}_{k^\prime k^{\prime\prime}}^\top\boldsymbol{\tilde{w}}_{k^\prime,r}\!\ge\! 0\} \mathds{1}\!\{\boldsymbol{x}_{l^\prime l^{\prime\prime}}^\top\boldsymbol{\tilde{w}}_{l^\prime,r}\!\ge\! 0\}\\
    &\neq\mathds{1}\!\{\!\boldsymbol{x}_{k^\prime k^{\prime\prime}}^\top\boldsymbol{w}^{(0)}_{r}\!\ge\! 0\} \mathds{1}\!\{\!\boldsymbol{x}_{l^\prime l^{\prime\prime}}^\top\boldsymbol{w}^{(0)}_{r}\!\ge\! 0\}\!\}]\\
    \le&\mathbb{E}[\frac{1}{m}
    \!\sum\nolimits_{r=1}^m\mathds{1}\{\mathds{1}\!\{\boldsymbol{x}_{k^\prime k^{\prime\prime}}^\top\boldsymbol{\tilde{w}}_{k^\prime,r}\!\ge\! 0\} \mathds{1}\!\{\boldsymbol{x}_{l^\prime l^{\prime\prime}}^\top\boldsymbol{\tilde{w}}_{l^\prime,r}\!\ge\! 0\}\\
    &\neq\mathds{1}\!\{\!\boldsymbol{x}_{k^\prime k^{\prime\prime}}^\top\boldsymbol{w}^{(0)}_{r}\!\ge\! 0\} \mathds{1}\!\{\!\boldsymbol{x}_{l^\prime l^{\prime\prime}}^\top\boldsymbol{w}^{(0)}_{r}\!\ge\! 0\}\!\}]\\
    \le&\mathbb{E}[\frac{1}{m}
    \!\sum\nolimits_{r=1}^m(\mathds{1}\{\mathds{1}\!\{\boldsymbol{x}_{k^\prime k^{\prime\prime}}^\top\boldsymbol{\tilde{w}}_{k^\prime,r}\!\ge\! 0\}\neq\mathds{1}\!\{\boldsymbol{x}_{l^\prime l^{\prime\prime}}^\top\boldsymbol{\tilde{w}}_{l^\prime,r}\!\ge\! 0\}\}\\
    &+\mathds{1}\{\mathds{1}\!\{\!\boldsymbol{x}_{k^\prime k^{\prime\prime}}^\top\boldsymbol{w}^{(0)}_{r}\!\ge\! 0\}\neq\mathds{1}\!\{\!\boldsymbol{x}_{l^\prime l^{\prime\prime}}^\top\boldsymbol{w}^{(0)}_{r}\!\ge\! 0\}\!\})]\\
    \le&\mathbb{E}[\frac{1}{m}\sum\nolimits_{r=1}^m\left(\mathds{1}\left\{\mathcal{E}_{k^\prime k^{\prime\prime}r}(R_0)\right\}+\mathds{1}\left\{\mathcal{E}_{l^\prime l^{\prime\prime}r}(R_0)\right\}\right)]\\
    =&\frac{1}{m}\sum\nolimits_{r=1}^m\mathbb{P}(\mathcal{E}_{k^\prime k^{\prime\prime}r}(R_0))+\mathbb{P}(\mathcal{E}_{l^\prime l^{\prime\prime}r}(R_0))
    \le\frac{4R_0}{\sqrt{2\pi}\sigma}.
\end{flalign*}
Then, we bound the Frobenius norm,
\begin{flalign*}
    \mathbb{E}[\|\mathbf{\tilde{H}}_{\mathcal{P}}-\mathbf{H}_\mathcal{P}^{(0)}\|_{\text{F}}]
    &\le\mathbb{E}[\sum_{k,l\in\mathcal{P}}|\mathbf{\tilde{H}}_{\mathcal{P}}(k,l)-\mathbf{H}_\mathcal{P}^{(0)}(k,l)|]\\
    &\le 4(\rho N)^2R_0/(\sqrt{2\pi}\sigma).
\end{flalign*}
Using Markov’s inequality, with probability $1 - \delta$, we get 
\begin{equation*}
    \|\mathbf{\tilde{H}}_{\mathcal{P}}-\mathbf{H}_\mathcal{P}^{(0)}\|_{\text{F}}\le\mathbb{E}[\|\mathbf{\tilde{H}}_{\mathcal{P}}-\mathbf{H}_\mathcal{P}^{(0)}\|_{\text{F}}]/\delta\le4(\rho N)^2R_0/(\sqrt{2\pi}\delta\sigma).
\end{equation*}
Thus, 
\begin{equation*}
    \|\mathbf{\tilde{H}}_{\mathcal{P}}-\mathbf{H}_\mathcal{P}^{(0)}\|_{2}\le\|\mathbf{\tilde{H}}_{\mathcal{P}}-\mathbf{H}_\mathcal{P}^{(0)}\|_{\text{F}}\le4(\rho N)^2R_0/(\sqrt{2\pi}\delta\sigma).
\end{equation*}
By Lemma~\ref{le-H0} and the settng of $m$, we have
\begin{equation*}
    \lambda_{\min}(\mathbf{\tilde{H}}_{\mathcal{P}})\ge\lambda_{\min}(\mathbf{H}_\mathcal{P}^{(0)})-4(\rho N)^2R_0/(\sqrt{2\pi}\delta\sigma)\ge\lambda_0/2.
\end{equation*}
To get upper bound of $\lambda_{\max}\left(\mathbf{\tilde{H}}_{\mathcal{P}}\right)$, we have
\begin{flalign*}
    &\|\mathbf{\tilde{H}}_{\mathcal{P}}\|_2
    =\|\mathbf{\tilde{A}}_{\mathcal{P}}\mathbf{\tilde{A}}_{\mathcal{P}}^\top\|_2=\|\mathbf{\tilde{A}}_{\mathcal{P}}^\top\mathbf{\tilde{A}}_{\mathcal{P}}\|_2\\
    =&\|\sum\nolimits_{i\in\mathcal{P}}\sum\nolimits_{j=1}^n\phi_{\boldsymbol{a}}(\boldsymbol{\tilde{W}}_i,x_{ij})\phi_{\boldsymbol{a}}(\boldsymbol{\tilde{W}}_i,x_{ij})^\top\|_2\\
    \le&\sum\nolimits_{i\in\mathcal{P}}\sum\nolimits_{j=1}^n\|\phi_{\boldsymbol{a}}\left(\boldsymbol{\tilde{W}}_i,x_{ij}\right)\phi_{\boldsymbol{a}}\left(\boldsymbol{\tilde{W}}_i,x_{ij}\right)^\top\|_2\\
    \le&\sum\nolimits_{i\in\mathcal{P}}\sum\nolimits_{j=1}^n\|\phi_{\boldsymbol{a}}\left(\boldsymbol{\tilde{W}}_i,x_{ij}\right)\|_2^2\le \rho N,
\end{flalign*}
where the last inequality uses Lemma~\ref{le-bounded_phi}.
\end{proof}

\begin{proof}[Proof of Lemma~\ref{le-fishdev}]
By \cite[Lemma~1]{jhunjhunwala2024fedfisher}, $\boldsymbol{\bar{W}}_{\mathcal{P}}^{\text{F}}$ is equivalent to the solution obtained via a gradient descent (GD) procedure. Let $\boldsymbol{\tilde{W}}_{\mathcal{P}}^{[t]}$ be the $t$-th iterate of the GD procedure with the initialization $\boldsymbol{\tilde{W}}_{\mathcal{P}}^{[0]}=\boldsymbol{\bar{W}}_{\mathcal{P}}$. 
Define the vectorized output $\tilde{f}(t)\in\mathbb{R}^N$ by stacking $\tilde{f}(\boldsymbol{\tilde{W}}_{\mathcal{P}}^{[t]}, x_{ij})$ across all $i\in[M], j\in[n]$:
\begin{equation*}
    \tilde{f}(t)\!\coloneqq\![\tilde{f}(\boldsymbol{\tilde{W}}_{\mathcal{P}}^{[t]}, x_{11}), \tilde{f}(\boldsymbol{\tilde{W}}_{\mathcal{P}}^{[t]}, x_{12}),\cdots,\tilde{f}(\boldsymbol{\tilde{W}}_{\mathcal{P}}^{[t]}, x_{Mn})]^\top\!\in\! \mathbb{R}^N.
\end{equation*}
Thus, GD produces the recurrence
\begin{equation}\label{eq-recurrence}
    \boldsymbol{\tilde{W}}_{\mathcal{P}}^{[t]} = \boldsymbol{\tilde{W}}_{\mathcal{P}}^{[t-1]} - \frac{\tilde{\eta}}{n} \mathbf{\tilde{A}}_{\mathcal{P}}^\top (\tilde{f}(t-1) - \boldsymbol{\tilde{y}}).
\end{equation}
which implies the prediction evolution
\begin{equation}\label{eq-predict-recurrence}
    \tilde{f}(t) = \tilde{f}(t-1) - \frac{\tilde{\eta}}{n} \mathbf{\tilde{H}}_{\mathcal{P}} (\tilde{f}(t-1) - \boldsymbol{\tilde{y}}).
\end{equation}
Unrolling the recurrence, we obtain for any $t$, 
\begin{equation}\label{eq-predict}
    \tilde{f}(t)- \boldsymbol{\tilde{y}}=(\mathbf{I} - \frac{\tilde{\eta}}{n} \mathbf{\tilde{H}}_{\mathcal{P}})(\tilde{f}(t-1)- \boldsymbol{\tilde{y}})=(\mathbf{I} - \frac{\tilde{\eta}}{n} \mathbf{\tilde{H}}_{\mathcal{P}})^t(\tilde{f}(0)- \boldsymbol{\tilde{y}}).
\end{equation}
The distance between the final global model and the initialization is expressed as a gradient series: 
\begin{equation}\label{eq-model-distance}
    \boldsymbol{\bar{W}}_{\mathcal{P}}^{\text{F}} - \boldsymbol{\bar{W}}_{\mathcal{P}} = -\frac{\tilde{\eta}}{n} \sum_{t=0}^\infty \mathbf{\tilde{A}}_{\mathcal{P}}^\top (\mathbf{I} - \frac{\tilde{\eta}}{n} \mathbf{\tilde{H}}_{\mathcal{P}})^t (\tilde{f}(0) - \boldsymbol{\tilde{y}}).
\end{equation}
Let $\mathbf{U}\coloneqq \frac{\tilde{\eta}}{n}(\mathbf{I} - \frac{\tilde{\eta}}{n} \mathbf{\tilde{H}}_{\mathcal{P}})^t$ be the geometric series matrix. Then
\begin{flalign*}
    \|\boldsymbol{\bar{W}}_{\mathcal{P}}^{\text{F}} - \boldsymbol{\bar{W}}_{\mathcal{P}}\|_2^2 
    &= \|\mathbf{\tilde{A}}_{\mathcal{P}}^\top \mathbf{U} (\tilde{f}(0) - \boldsymbol{\tilde{y}})\|_2^2\\
    &\le\left(\tilde{f}(0)-\boldsymbol{\tilde{y}}\right)^\top\mathbf{U}\mathbf{\tilde{H}}_{\mathcal{P}}\mathbf{U}\left(\tilde{f}(0)-\boldsymbol{\tilde{y}}\right).
\end{flalign*}
Using eigen-decomposition $\mathbf{\tilde{H}}_{\mathcal{P}} = \mathbf{V D V}^\top$, we obtain 
\begin{equation*}
    \mathbf{U}=\frac{\tilde{\eta}}{n}\mathbf{V}\sum_{t=1}^\infty(\mathbf{I}-\frac{\tilde{\eta}}{n}\mathbf{D})^t\mathbf{V}^\top=\mathbf{V}\mathbf{D}^{-1}\mathbf{V}^\top=\mathbf{\tilde{H}}_{\mathcal{P}}^{-1},
\end{equation*}
The squared distance becomes 
\begin{equation*}
    \|\boldsymbol{\bar{W}}_{\mathcal{P}}^{\text{F}} - \boldsymbol{\bar{W}}_{\mathcal{P}}\|_2^2
    \le(\tilde{f}(0) - \boldsymbol{\tilde{y}})^\top \mathbf{\tilde{H}}_{\mathcal{P}}^{-1} (\tilde{f}(0) - \boldsymbol{\tilde{y}})
    \le \frac{2}{\lambda_0} \| \tilde{f}(0) - \boldsymbol{\tilde{y}} \|_2^2,
\end{equation*}
where the last inequality is due to Lemma~\ref{le-eigen}. Finally, Lemma~\ref{le-lse} bounds $\| \tilde{f}(0) - \boldsymbol{\tilde{y}} \|_2^2$, which leads to the desired results.
\end{proof}

\begin{proof}[Proof of Lemma~\ref{le-fishopt}]
By~\eqref{eq-model-distance}, for each neuron $r$, we have, 
\begin{flalign*}
    &\|\boldsymbol{\bar{w}}_{\mathcal{P},r}^{\text{F}}-\boldsymbol{\bar{w}}_{\mathcal{P},r}\|_2\\
    =&\|\frac{\tilde{\eta}}{n}\mathbf{\tilde{A}}_{\mathcal{P}}[(r-1)p+1:rp]^\top(\mathbf{I} - \frac{\tilde{\eta}}{n} \mathbf{\tilde{H}}_{\mathcal{P}})^t(\tilde{f}(0)-\boldsymbol{\tilde{y}})\|_2\\
    \le&\frac{2}{\lambda_0}\|\mathbf{\tilde{A}}_{\mathcal{P}}[(r-1)p+1:rp]^\top\|_2\|\tilde{f}(0)-\boldsymbol{\tilde{y}}\|_2\\
    \le&\frac{2}{\lambda_0}\|\mathbf{\tilde{A}}_{\mathcal{P}}[(r-1)p+1:rp]^\top\|_{\text{F}}\|\tilde{f}(0)-\boldsymbol{\tilde{y}}\|_2\\
    \le&\frac{2\sqrt{\rho N}}{\lambda_0\sqrt{m}}\|\tilde{f}(0)-\boldsymbol{\tilde{y}}\|_2,
\end{flalign*}
where we use the notation $\boldsymbol{A}[x:y]$ to denote the submatrix of $\boldsymbol{A}$ having row numbers $x$ to $y$, and the last inequality follows from the definition of $\mathbf{\tilde{A}}_{\mathcal{P}}$ and the data normalization assumption (Assumption~\ref{asp-normalization}). Combing with Lemma~\eqref{le-lse}, we completes the proof.
\end{proof}

\begin{proof}[Proof of Lemma~\ref{le-mismatch-size}]
Each mismatch requires a sign flip between $\boldsymbol{\bar{w}}_{\mathcal{P},r}^{\text{F}}$ or $\boldsymbol{\tilde{w}}_{i,r}$ and the shared initialization $\boldsymbol{w}_r^{(0)}$. 
\begin{flalign*}
    &\mathbb{E}[|S_{ij}^{\text{F}\perp}|]  =\sum_{r=1}^m\mathbb{P}(\mathds{1}\{\boldsymbol{x}_{ij}^{\top}\boldsymbol{\bar{w}}_{\mathcal{P},r}^\text{F}\ge0\}\neq\mathds{1}\{\boldsymbol{x}_{ij}^{\top}\boldsymbol{\tilde{w}}_{i,r}\ge0\})\\
    \le& \sum_{r=1}^m\mathbb{P}(\{\mathds{1}\{\boldsymbol{x}_{ij}^{\top}\boldsymbol{\bar{w}}_{\mathcal{P},r}^\text{F}\ge0\}\neq\mathds{1}\{\boldsymbol{x}_{ij}^{\top}\boldsymbol{w}^{(0)}_{r}\ge0\}\}\\
    &\cup \{\mathds{1}\{\boldsymbol{x}_{ij}^{\top}\boldsymbol{\tilde{w}}_{i,r}\ge0\}\neq\mathds{1}\{\boldsymbol{x}_{ij}^{\top}\boldsymbol{w}^{(0)}_{r}\ge0\}\})\\
    \le&\sum_{r=1}^m[\mathbb{P}(\{\mathds{1}\{\boldsymbol{x}_{ij}^{\top}\boldsymbol{\bar{w}}_{\mathcal{P},r}^\text{F}\ge0\}\neq\mathds{1}\{\boldsymbol{x}_{ij}^{\top}\boldsymbol{w}^{(0)}_{r}\ge0\}\})\\
    &+ \mathbb{P}(\{\mathds{1}\{\boldsymbol{x}_{ij}^{\top}\boldsymbol{\tilde{w}}_{i,r}\ge0\}\neq\mathds{1}\{\boldsymbol{x}_{ij}^{\top}\boldsymbol{w}^{(0)}_{r}\ge0\}\})]\\
    \le&\sum_{r=1}^m(\mathbb{P}(\mathcal{E}_{ijr}(R_0+R_1))+\mathbb{P}(\mathcal{E}_{ijr}(R_0)))
    \le\frac{4m(R_0+R_1)}{\sqrt{2\pi}\sigma}.
\end{flalign*}
By applying Markov's inequality with failure probability $\delta/N$ and taking a union bound over all $(i,j)$ pairs, we obtain the desired bound. 
\end{proof}

\begin{proof}[Proof of Lemma~\ref{le-general-err2}]
By \cite[Lemma~2]{hu2023generalization}, $\sup_{f \in \mathcal{F}} \{ L_{\mathcal{T}}^{f,\ell} - L_S^{f,\ell} \}$ is less than
\begin{equation*}
\begin{aligned}
    & 2 \mathbb{E}_{\{D_i, \iota_i\}} [ \sup_{f \in \mathcal{F}} \frac{1}{M} \sum_{i=1}^M \iota_i \mathbb{E}_{\boldsymbol{z}_i \sim D_i} [\ell(f(\boldsymbol{W}, \boldsymbol{x}_i), y_i)] ]+ c\sqrt{\frac{\log\frac{2}{\delta}}{2P}}.
\end{aligned}
\end{equation*}
Let $\tilde{\ell}(\boldsymbol{W}, \boldsymbol{z}) \coloneqq \ell(f(\boldsymbol{W}, \boldsymbol{x}), y)$, the expectation term is
\begin{flalign*}
    &\mathbb{E}_{\substack{D_1,  \cdots, D_P\\ \iota_1, \cdots, \iota_P}}[\sup_{f\in\mathcal{F}}\frac{1}{M}\sum_{i=1}^M\iota_i\mathbb{E}_{\boldsymbol{z}_i\sim D_i}[\ell(f(\boldsymbol{W}, \boldsymbol{x}_i), y_i)]]\\
    = & \mathbb{E}_{\substack{\boldsymbol{z}_{11},  \cdots, \boldsymbol{z}_{Pn}\\ \iota_1, \cdots, \iota_P}}[\sup_{f\in\mathcal{F}}\frac{1}{M}\sum_{i=1}^M\frac{1}{n}\sum_{j=1}^n\iota_i\tilde{\ell}(\boldsymbol{W}, \boldsymbol{z}_{ij})]\\
    \le & \mathbb{E}_{\substack{\boldsymbol{z}_{11},  \cdots, \boldsymbol{z}_{Pn}\\ \iota_{11}, \cdots, \iota_{Pn}}}\![\sup_{f\in\mathcal{F}}\frac{1}{P}\sum_{i\in\mathcal{P}}\frac{1}{n}\sum_{j=1}^n\iota_{ij}\tilde{\ell}(\boldsymbol{W}, \boldsymbol{z}_{ij})]\\
    = & \mathbb{E}_{\boldsymbol{z}_{11},  \cdots, \boldsymbol{z}_{Pn}}[\mathbb{E}_{\iota_{11}, \cdots, \iota_{Pn}}[\sup_{f\in\mathcal{F}}\frac{1}{P}\sum_{i\in\mathcal{P}}\frac{1}{n}\sum_{j=1}^n\iota_{ij}\tilde{\ell}(\boldsymbol{W}, \boldsymbol{z}_{ij})]]\\
    = & \mathbb{E}_{\boldsymbol{z}_{11},  \cdots, \boldsymbol{z}_{Pn}}[\mathcal{R}(\mathcal{F})]
    \le \mathcal{R}(\mathcal{F}) + c\sqrt{\log(2/\delta)/(2N)},
\end{flalign*}
where the last inequality follows from McDiarmid's inequality, since replacing one point $\boldsymbol{z}_{ij}$ changes $\mathcal{R}(\mathcal{F})$ by at most $c/N$.
\end{proof}

\section{Additional Experimental Details and Results}
{
\subsection{Additional Experimental Details}
\subsubsection{Dataset and Models}
Beyond the default two-layer CNN, we further evaluate F-SOSA using an ImageNet-pretrained ResNet18 on CIFAR-10 and YOLOv8s~\cite{jocher2023yolov8} on BDD100K~\cite{yu2020bdd100k}. BDD100K is a large-scale real-world driving dataset containing 100,000 40-second videos recorded at 720p and 30 FPS across diverse road, weather, and lighting conditions. It supports ten autonomous-driving tasks, with bounding-box annotations for ten traffic-related object categories. For both settings, we simulate client heterogeneity using the same Dirichlet partitioning scheme ($\alpha\in\{0.1, 0.5, 0.8\}$; default $\alpha{=}0.5$).

\subsubsection{Hyperparameters}
For experiments on CIFAR-10 with ResNet18 and BDD100K with YOLOv8s, we set the local learning rate to $\eta=0.001$ and the total number of clients to $M=20$. For CIFAR-10, we set the number of local epochs to $K=50$ and the batch size to $64$; for BDD100K, we use $K=20$ local epochs, a batch size of $32$, and an input resolution of $640\times640$. Both settings employ SGD, with momentum $0.9$. For BDD100K, we additionally apply weight decay of $5\times10^{-4}$, a three-epoch warm-up, and cosine learning-rate decay.

\subsubsection{Metrics}
For the object-detection task on BDD100K, we use mAP$_{50\text{-}95}$ as the primary metric to evaluate the detection utility of F-SOSA and the baselines. It is computed by evaluating average precision (AP) at IoU thresholds from $0.50$ to $0.95$ in increments of $0.05$ and averaging the results across all thresholds and ten object categories. This metric jointly reflects object recognition and bounding-box localization accuracy, with higher values indicating better overall detection performance.

\subsection{F-SOSA Performance Across Model Architectures}
\label{app:model}

To demonstrate that F-SOSA generalizes across models with different capacities, we evaluate network architectures beyond the two-layer CNN, including ResNet18 on CIFAR-10 and YOLOv8s on the real-world autonomous-driving dataset BDD100K. We further evaluate F-SOSA under varying data heterogeneity levels ($\alpha \in \{0.1, 0.5, 0.8\}$) to assess its robustness across both model architectures and non-IID settings. As shown in Table~\ref{tab:my-table4}, on CIFAR-10 with ResNet18, F-SOSA consistently achieves the lowest communication cost of $462.44$ MB, reducing communication by $47.6\%$ relative to FedFisher, $48.3\%$ relative to Dense, and $94.3\%$ relative to FedLPA. It also achieves the lowest time cost across all heterogeneity levels, with reductions of up to $43.0\%$, $56.0\%$, and $91.1\%$ compared with FedFisher, Dense, and FedLPA, respectively. Although F-SOSA trails the best-performing baseline under high heterogeneity ($\alpha=0.1$), it still outperforms Dense by $7.13$ percentage points. At $\alpha=0.5$ and $0.8$, its accuracy increases to $65.98\%$ and $69.87\%$, narrowing the gap with FedFisher to $3.68$ and $2.74$ percentage points, respectively, while continuing to outperform Dense. Given its consistently lowest time and communication costs, these modest accuracy gaps represent a reasonable utility-efficiency trade-off. On BDD100K with YOLOv8s, F-SOSA maintains competitive detection performance, achieving an mAP$_{50\text{-}95}$ only 0.20 percentage points below the best result at $\alpha=0.1$, tying for the best result at $\alpha=0.5$, and obtaining the highest mAP$_{50\text{-}95}$ of $21.30\%$ at $\alpha=0.8$. Meanwhile, it consistently incurs the lowest communication cost of $0.46$ GB, representing reductions of $47.1\%$ over FedFisher, $47.7\%$ over Dense, and $95.2\%$ over FedLPA. F-SOSA also achieves the lowest time cost at $\alpha=0.1$ and $0.5$ and the second-lowest at $\alpha=0.8$, while delivering the highest communication efficiency under all settings. These results demonstrate that F-SOSA remains effective with larger
classification and detection models, offering a favorable utility--efficiency
trade-off across different model architectures, tasks, and degrees of data
heterogeneity.

\subsection{Scalability of F-SOSA-U and F-SOSA-I Across Client Populations}
\label{app:scalability}

To assess the scalability of F-SOSA-I/U with increasing edge-device density,
we consider five total client populations,
$\{20,50,100,200,500\}$. For each population, we execute $15$ sequential
incremental-learning requests and $15$ sequential unlearning requests.
Table~\ref{tab:my-table5} reports the mean accuracy, time-cost savings
relative to full retraining, and communication cost normalized to Retrain,
which together capture utility preservation and system efficiency.

For incremental learning, F-SOSA-I maintains accuracy comparable to Retrain
across the evaluated client scales. Even at $500$ clients, the accuracy gap
between F-SOSA-I and Retrain remains within $2.46$ percentage points across
all four datasets. In the $50$-client case, F-SOSA-I achieves the highest accuracy on MNIST, SVHN, and CIFAR-10, reaching $94.35\%$, $68.95\%$, and $59.99\%$, respectively. As the system scales to $200$ clients, it achieves the best performance on FEMNIST and CIFAR-10, with accuracies of $32.33\%$ and $46.27\%$, respectively. For the largest population of $500$ clients, the accuracy gap between F-SOSA-I and Retrain remains within $2.46$ percentage points across all four datasets, while F-SOSA-I outperforms Fine-tuning by $2.45$--$10.90$ percentage points. This utility advantage is particularly pronounced on the inherently heterogeneous FEMNIST dataset: as the client population increases from $100$ to $500$, F-SOSA-I maintains an accuracy of approximately $31\%$, whereas Fine-tuning remains around $20\%$. In terms of efficiency, the normalized communication cost of F-SOSA-I never exceeds $0.87\times$ that of Retrain and decreases to $0.67\times$ with $500$ clients, corresponding to a $33\%$ reduction. Although Fine-tuning incurs marginally lower communication overhead at some scales, these savings come at the cost of substantial accuracy degradation, particularly on MNIST, SVHN, and FEMNIST. F-SOSA-I therefore provides a more favorable utility--efficiency trade-off. Its time-cost savings remain positive across all configurations, ranging from $15.84\%$ to $49.04\%$. Even with a population of $500$ clients, F-SOSA-I saves $23.26\%$--$49.04\%$ of the retraining time, indicating that its computational benefits persist as the edge system becomes increasingly dense.

\renewcommand{\arraystretch}{1.3}
\setlength{\arrayrulewidth}{0.3mm}  
\begin{table*}[t]
\centering
\caption{Performance comparison of F-SOSA and baselines across CIFAR-10 and BDD100K under varying Dirichlet heterogeneity $\alpha$. Best results are in bold; second-best results are underlined.}
\label{tab:my-table4}
\resizebox{\textwidth}{!}{%
\begin{tabular}{cccccccccccccccc}
\hline
\multirow{2}{*}{Dataset} & \multirow{2}{*}{Metrics} & \multicolumn{4}{c}{$\alpha=0.1$} &  & \multicolumn{4}{c}{$\alpha=0.5$} &  & \multicolumn{4}{c}{$\alpha=0.8$} \\ \cline{3-6} \cline{8-11} \cline{13-16}
&  & FedFisher & FedLPA & Dense & F-SOSA &  & FedFisher & FedLPA & Dense & F-SOSA &  & FedFisher & FedLPA & Dense & F-SOSA \\ \hline
\multirow{5}{*}{\shortstack{CIFAR10\\(Resnet18)}} & Accuracy (\%) & \textbf{46.60} & {\ul 45.18} & 28.04 & 35.17 &  & \textbf{69.66} & {\ul 68.41} & 57.48 & 65.98 &  & \textbf{72.61} & {\ul 70.55} & 63.92 & 69.87 \\
& Time (s) & {\ul 361.02} & 2493.05 & 479.82 & \textbf{278.22} &  & {\ul 362.14} & 2355.83 & 474.14 & \textbf{210.34} &  & {\ul 365.50} & 1713.55 & 473.76 & \textbf{208.27} \\
& Comm. (MB) & {\ul 882.16} & 8066.21 & 895.13 & \textbf{462.44} &  & {\ul 882.16} & 8066.21 & 895.13 & \textbf{462.44} &  & {\ul 882.16} & 8066.21 & 895.13 & \textbf{462.44} \\
& Time Efficiency & \textbf{0.129} & 0.018 & 0.058 & {\ul 0.126} &  & {\ul 0.192} & 0.029 & 0.121 & \textbf{0.314} &  & {\ul 0.199} & 0.041 & 0.135 & \textbf{0.335} \\
& Comm. Efficiency & {\ul 0.053} & 0.006 & 0.031 & \textbf{0.076} &  & {\ul 0.079} & 0.008 & 0.064 & \textbf{0.143} &  & {\ul 0.082} & 0.009 & 0.071 & \textbf{0.151} \\ \hline
\multirow{5}{*}{\shortstack{BDD100K\\(YOLOv8s)}} & mAP$_{50\text{-}95}$ (\%) & {\ul 20.82} & \textbf{20.98} & 20.76 & 20.78 &  & \textbf{21.18} & {\ul 21.17} & 19.54 & \textbf{21.18} &  & 21.21 & {\ul 21.22} & 21.13 & \textbf{21.30} \\
& Time (min) & 39.88 & {\ul 36.41} & 37.28 & \textbf{32.52} &  & 37.77 & {\ul 31.32} & 34.67 & \textbf{28.96} &  & 37.69 & \textbf{31.10} & 37.11 & {\ul 34.31} \\
& Comm. (GB) & {\ul 0.87} & 9.55 & 0.88 & \textbf{0.46} &  & {\ul 0.87} & 9.55 & 0.88 & \textbf{0.46} &  & {\ul 0.87} & 9.55 & 0.88 & \textbf{0.46} \\
& Time Efficiency & 0.52 & {\ul 0.58} & 0.56 & \textbf{0.64} &  & 0.56 & {\ul 0.68} & 0.61 & \textbf{0.73} &  & 0.56 & \textbf{0.68} & 0.57 & {\ul 0.62} \\
& Comm. Efficiency & {\ul 23.93} & 2.20 & 23.59 & \textbf{45.17} &  & {\ul 24.34} & 2.22 & 22.20 & \textbf{46.04} &  & {\ul 24.38} & 2.22 & 24.01 & \textbf{46.30} \\ \hline
\end{tabular}%
}
\end{table*}

Regarding unlearning, F-SOSA-U exhibits similarly strong scalability. Across all evaluated client populations, its accuracy closely tracks or exceeds that of Retrain while consistently outperforming most unlearning baselines. The results for $100$ clients show that F-SOSA-U achieves the highest accuracy on MNIST, SVHN, and CIFAR-10. When the client population reaches 500, F-SOSA-U surpasses Retrain on SVHN, FEMNIST, and CIFAR-10 and remains within $0.89$ percentage points of Retrain on MNIST. In this largest-scale regime, F-SOSA-U also outperforms the strongest competing unlearning baseline by $1.24$--$4.85$ percentage points across the four datasets. Its normalized communication cost remains below that of Retrain throughout the evaluated range. With $500$ clients, F-SOSA-U reduces communication by $47\%$ relative to Retrain and FATS-CU and by approximately $60\%$ relative to FedEraser and FedRR, whose communication costs exceed $1.33\times$ that of Retrain. Moreover, F-SOSA-U continues to save $38.92\%$--$43.35\%$ of the retraining time at the maximum evaluated scale. Although FedEraser and FATS-CU achieve greater time-cost savings, their
accuracy degrades more noticeably as the client population increases.
Consequently, F-SOSA-U provides a more favorable overall trade-off among
utility, communication efficiency, and computational cost.

In addition, we note that the absolute accuracy on MNIST, SVHN, and CIFAR-10 decreases as the number of clients increases, whereas this trend is not observed on the inherently non-IID FEMNIST dataset. This behavior is a consequence of our data-partitioning strategy: increasing the number of clients reduces the amount of training data assigned to each client, thereby limiting the local data available for one-shot FL and subsequent editing. Importantly, F-SOSA-I and F-SOSA-U still achieve the best or near-best relative accuracy compared with their respective baselines across different client scales. This suggests that the decline is primarily associated with the reduced
amount of local data per client, rather than a deterioration unique to
F-SOSA-I/U.

Overall, these results show that the relative communication and computation benefits of F-SOSA-I/U remain substantial as the client population increases. Across the evaluated range of up to $500$ clients, both editing algorithms preserve utility close to Retrain while continuing to reduce its communication and computation costs. This indicates that the efficiency advantage of selective exact editing remains robust under increasingly dense client populations.

\subsection{Real-World Driving Case Study with BDD100K and SUMO Mobility}
\label{app:driving_case}

\begin{table*}[t]
\centering
\caption{Scalability over 15 incremental-learning or unlearning requests. We report mean accuracy (Acc.), mean time-cost saving (T-save), and mean communication cost normalized to Retrain (Comm.; Retrain $=1.00$). T-save is relative to Retrain; higher is better for Acc. and T-save, while lower is better for Comm. Best results are in bold; second-best results are underlined.}
\label{tab:my-table5}
\begingroup
\scriptsize
\setlength{\tabcolsep}{2.2pt}
\renewcommand{\arraystretch}{1.04}
\begin{tabular*}{\textwidth}{@{\extracolsep{\fill}}cl*{12}{c}@{}}
\toprule
\multirow{2}{*}{Clients} & \multirow{2}{*}{Method} & \multicolumn{3}{c}{MNIST} & \multicolumn{3}{c}{SVHN} & \multicolumn{3}{c}{FEMNIST} & \multicolumn{3}{c}{CIFAR10} \\ \cmidrule(lr){3-5} \cmidrule(lr){6-8} \cmidrule(lr){9-11} \cmidrule(lr){12-14}
 & & Acc. (\%) & T-save (\%) & Comm. ($\times$) & Acc. (\%) & T-save (\%) & Comm. ($\times$) & Acc. (\%) & T-save (\%) & Comm. ($\times$) & Acc. (\%) & T-save (\%) & Comm. ($\times$) \\
\midrule
\multicolumn{14}{c}{\textit{(a) Incremental learning}} \\
\midrule
\multirow{3}{*}{20} & Retrain & \textbf{96.33} & 0.00 & 1.00 & {\ul 73.07} & 0.00 & 1.00 & \textbf{27.06} & 0.00 & 1.00 & {\ul 68.51} & 0.00 & 1.00 \\
 & Fine-tuning & {\ul 95.74} & \textbf{79.68} & {\ul 0.78} & \textbf{74.67} & \textbf{81.48} & {\ul 0.78} & 21.65 & \textbf{69.13} & {\ul 0.78} & 67.86 & \textbf{78.89} & {\ul 0.78} \\
 & F-SOSA-I & 95.62 & {\ul 20.83} & \textbf{0.75} & 72.82 & {\ul 37.84} & \textbf{0.75} & {\ul 25.64} & {\ul 19.67} & \textbf{0.75} & \textbf{69.34} & {\ul 45.26} & \textbf{0.75} \\
\cmidrule(lr){1-14}
\multirow{3}{*}{50} & Retrain & {\ul 94.11} & 0.00 & 1.00 & {\ul 66.61} & 0.00 & 1.00 & \textbf{30.22} & 0.00 & 1.00 & {\ul 57.11} & 0.00 & 1.00 \\
 & Fine-tuning & 93.30 & \textbf{79.67} & \textbf{0.70} & 65.95 & \textbf{79.70} & \textbf{0.70} & 19.72 & \textbf{62.36} & \textbf{0.70} & 52.88 & \textbf{78.29} & \textbf{0.70} \\
 & F-SOSA-I & \textbf{94.35} & {\ul 21.64} & {\ul 0.72} & \textbf{68.95} & {\ul 38.63} & {\ul 0.72} & {\ul 30.09} & {\ul 20.47} & {\ul 0.72} & \textbf{59.99} & {\ul 46.08} & {\ul 0.72} \\
\cmidrule(lr){1-14}
\multirow{3}{*}{100} & Retrain & {\ul 92.81} & 0.00 & 1.00 & \textbf{62.55} & 0.00 & 1.00 & \textbf{31.24} & 0.00 & 1.00 & \textbf{56.01} & 0.00 & 1.00 \\
 & Fine-tuning & 91.32 & \textbf{72.55} & \textbf{0.68} & 55.83 & \textbf{79.64} & \textbf{0.68} & 20.28 & \textbf{64.89} & \textbf{0.68} & 47.70 & \textbf{77.37} & \textbf{0.68} \\
 & F-SOSA-I & \textbf{92.94} & {\ul 16.93} & {\ul 0.87} & {\ul 62.18} & {\ul 34.86} & {\ul 0.87} & {\ul 31.05} & {\ul 15.84} & {\ul 0.87} & {\ul 54.22} & {\ul 41.87} & {\ul 0.87} \\
\cmidrule(lr){1-14}
\multirow{3}{*}{200} & Retrain & \textbf{91.30} & 0.00 & 1.00 & \textbf{36.56} & 0.00 & 1.00 & {\ul 30.65} & 0.00 & 1.00 & {\ul 43.65} & 0.00 & 1.00 \\
 & Fine-tuning & 81.98 & \textbf{82.80} & \textbf{0.67} & 23.08 & \textbf{79.03} & \textbf{0.67} & 19.22 & \textbf{72.35} & \textbf{0.67} & 42.41 & \textbf{75.61} & \textbf{0.67} \\
 & F-SOSA-I & {\ul 90.22} & {\ul 18.25} & {\ul 0.87} & {\ul 35.94} & {\ul 36.04} & {\ul 0.87} & \textbf{32.33} & {\ul 17.16} & {\ul 0.87} & \textbf{46.27} & {\ul 43.18} & {\ul 0.87} \\
\cmidrule(lr){1-14}
\multirow{3}{*}{500} & Retrain & {\ul 85.38} & 0.00 & {\ul 1.00} & \textbf{19.60} & 0.00 & {\ul 1.00} & \textbf{31.88} & 0.00 & {\ul 1.00} & \textbf{35.65} & 0.00 & {\ul 1.00} \\
 & Fine-tuning & 75.52 & \textbf{80.26} & \textbf{0.67} & 14.77 & \textbf{79.95} & \textbf{0.67} & 20.15 & \textbf{76.82} & \textbf{0.67} & 30.74 & \textbf{74.48} & \textbf{0.67} \\
 & F-SOSA-I & \textbf{85.95} & {\ul 24.17} & \textbf{0.67} & {\ul 17.88} & {\ul 41.63} & \textbf{0.67} & {\ul 31.05} & {\ul 23.26} & \textbf{0.67} & {\ul 33.19} & {\ul 49.04} & \textbf{0.67} \\
\midrule
\multicolumn{14}{c}{\textit{(b) Unlearning}} \\
\midrule
\multirow{5}{*}{20} & Retrain & \textbf{95.73} & 0.00 & {\ul 1.00} & {\ul 65.04} & 0.00 & {\ul 1.00} & {\ul 26.12} & 0.00 & {\ul 1.00} & {\ul 64.75} & 0.00 & {\ul 1.00} \\
 & FedEraser & 93.95 & \textbf{78.64} & 1.35 & \textbf{68.48} & \textbf{79.12} & 1.35 & 21.17 & \textbf{77.85} & 1.35 & \textbf{67.74} & \textbf{78.31} & 1.35 \\
 & FedRR & 92.20 & 64.18 & 1.35 & 56.26 & 68.79 & 1.35 & 24.78 & 44.18 & 1.35 & 62.03 & 64.10 & 1.35 \\
 & FATS-CU & 88.72 & {\ul 76.92} & {\ul 1.00} & 61.25 & {\ul 77.36} & {\ul 1.00} & 20.11 & {\ul 75.94} & {\ul 1.00} & 61.85 & {\ul 76.58} & {\ul 1.00} \\
 & F-SOSA-U & {\ul 94.87} & 44.73 & \textbf{0.75} & 64.95 & 45.81 & \textbf{0.75} & \textbf{26.51} & 41.62 & \textbf{0.75} & 63.91 & 43.79 & \textbf{0.75} \\
\cmidrule(lr){1-14}
\multirow{5}{*}{50} & Retrain & \textbf{94.87} & 0.00 & {\ul 1.00} & \textbf{63.73} & 0.00 & {\ul 1.00} & {\ul 26.49} & 0.00 & {\ul 1.00} & 56.93 & 0.00 & {\ul 1.00} \\
 & FedEraser & 93.21 & \textbf{78.21} & 1.34 & 56.80 & \textbf{78.77} & 1.34 & 18.59 & \textbf{77.46} & 1.34 & \textbf{59.14} & \textbf{77.94} & 1.34 \\
 & FedRR & 91.97 & 63.76 & 1.34 & 39.11 & 68.31 & 1.34 & 21.22 & 43.76 & 1.34 & 56.10 & 63.69 & 1.34 \\
 & FATS-CU & 90.87 & {\ul 76.49} & {\ul 1.00} & 53.63 & {\ul 76.94} & {\ul 1.00} & 18.38 & {\ul 75.61} & {\ul 1.00} & 57.10 & {\ul 76.17} & {\ul 1.00} \\
 & F-SOSA-U & {\ul 94.28} & 44.16 & \textbf{0.87} & {\ul 62.88} & 45.22 & \textbf{0.87} & \textbf{28.01} & 41.08 & \textbf{0.87} & {\ul 57.86} & 43.25 & \textbf{0.87} \\
\cmidrule(lr){1-14}
\multirow{5}{*}{100} & Retrain & {\ul 92.26} & 0.00 & {\ul 1.00} & {\ul 58.58} & 0.00 & {\ul 1.00} & \textbf{30.68} & 0.00 & {\ul 1.00} & {\ul 51.64} & 0.00 & {\ul 1.00} \\
 & FedEraser & 91.31 & \textbf{77.83} & 1.34 & 45.50 & \textbf{78.39} & 1.34 & 23.82 & \textbf{77.02} & 1.34 & 51.57 & \textbf{77.56} & 1.34 \\
 & FedRR & 90.04 & 61.21 & 1.34 & 26.75 & 65.18 & 1.34 & 22.48 & 43.11 & 1.34 & 50.38 & 63.02 & 1.34 \\
 & FATS-CU & 90.24 & {\ul 76.12} & {\ul 1.00} & 44.65 & {\ul 76.58} & {\ul 1.00} & 23.51 & {\ul 75.18} & {\ul 1.00} & 50.67 & {\ul 75.83} & {\ul 1.00} \\
 & F-SOSA-U & \textbf{92.33} & 43.58 & \textbf{0.80} & \textbf{59.32} & 44.76 & \textbf{0.80} & {\ul 29.81} & 40.54 & \textbf{0.80} & \textbf{51.72} & 42.61 & \textbf{0.80} \\
\cmidrule(lr){1-14}
\multirow{5}{*}{200} & Retrain & \textbf{91.21} & 0.00 & {\ul 1.00} & \textbf{40.09} & 0.00 & {\ul 1.00} & {\ul 28.41} & 0.00 & {\ul 1.00} & \textbf{48.61} & 0.00 & {\ul 1.00} \\
 & FedEraser & 87.56 & \textbf{77.34} & 1.33 & 25.56 & \textbf{77.95} & 1.33 & 24.00 & \textbf{76.58} & 1.33 & 43.29 & \textbf{77.18} & 1.33 \\
 & FedRR & 86.95 & 59.84 & 1.33 & 19.59 & 62.46 & 1.33 & 25.90 & 42.42 & 1.33 & 40.06 & 62.74 & 1.33 \\
 & FATS-CU & 86.23 & {\ul 75.67} & {\ul 1.00} & 24.08 & {\ul 76.11} & {\ul 1.00} & 23.24 & {\ul 74.72} & {\ul 1.00} & 42.64 & {\ul 75.41} & {\ul 1.00} \\
 & F-SOSA-U & {\ul 90.28} & 42.91 & \textbf{0.67} & {\ul 39.62} & 44.13 & \textbf{0.67} & \textbf{29.09} & 39.86 & \textbf{0.67} & {\ul 46.22} & 41.94 & \textbf{0.67} \\
\cmidrule(lr){1-14}
\multirow{5}{*}{500} & Retrain & \textbf{86.62} & 0.00 & {\ul 1.00} & {\ul 19.66} & 0.00 & {\ul 1.00} & {\ul 30.45} & 0.00 & {\ul 1.00} & {\ul 33.71} & 0.00 & {\ul 1.00} \\
 & FedEraser & 80.99 & \textbf{76.82} & 1.33 & 19.59 & \textbf{77.46} & 1.33 & 25.11 & \textbf{76.11} & 1.33 & 31.90 & \textbf{76.73} & 1.33 \\
 & FedRR & 80.71 & 59.37 & 1.33 & 19.59 & 61.92 & 1.33 & 26.18 & 41.57 & 1.33 & 26.84 & 62.48 & 1.33 \\
 & FATS-CU & 79.63 & {\ul 75.13} & {\ul 1.00} & {\ul 19.66} & {\ul 75.62} & {\ul 1.00} & 25.38 & {\ul 74.29} & {\ul 1.00} & 31.71 & {\ul 74.96} & {\ul 1.00} \\
 & F-SOSA-U & {\ul 85.73} & 41.76 & \textbf{0.53} & \textbf{20.90} & 43.35 & \textbf{0.53} & \textbf{31.03} & 38.92 & \textbf{0.53} & \textbf{34.64} & 40.87 & \textbf{0.53} \\
\bottomrule
\end{tabular*}
\endgroup
\end{table*}

\begin{figure}[t]
    \centering
    \includegraphics[width=0.45\linewidth]{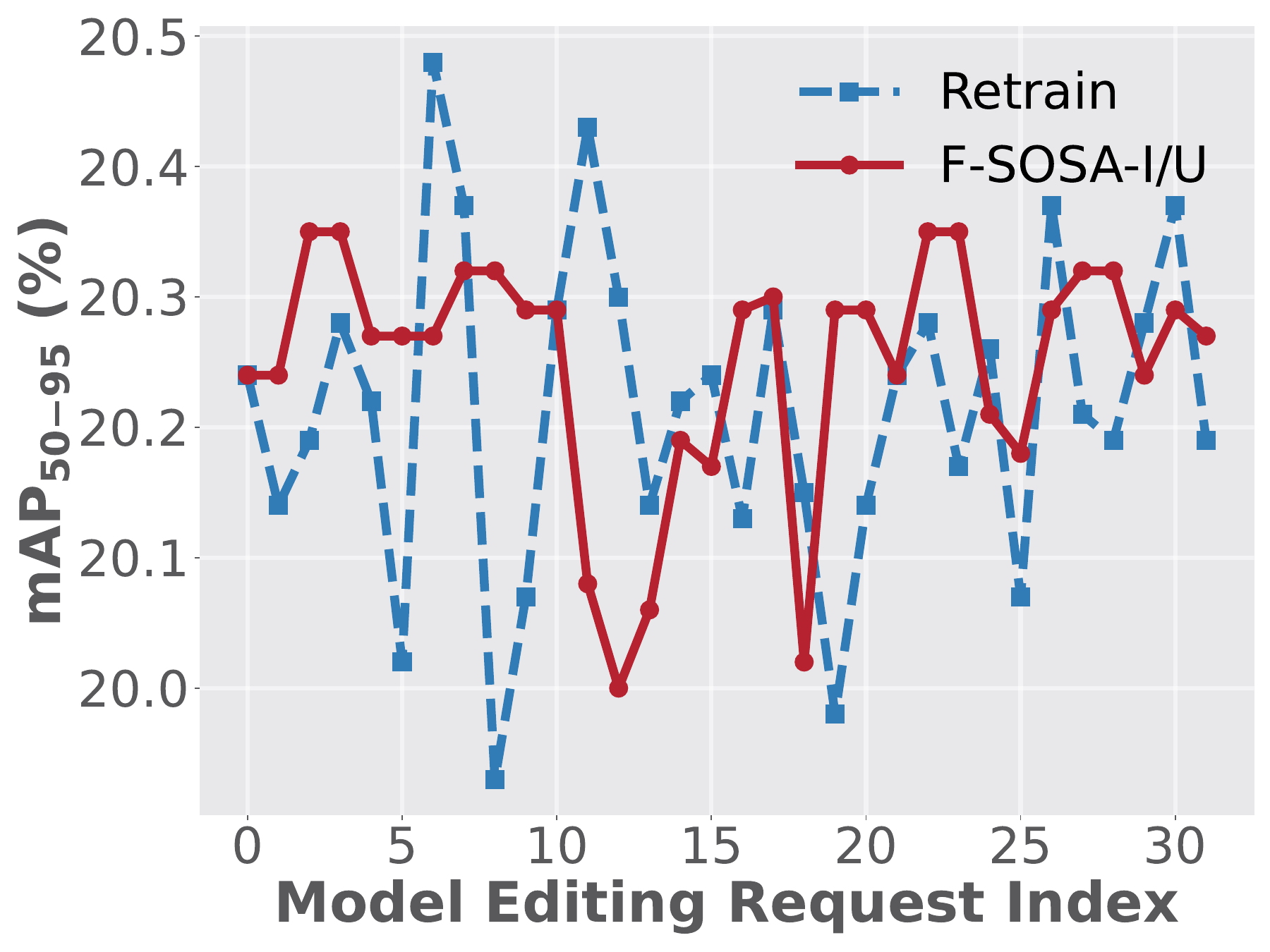}
    \hfill
    \includegraphics[width=0.45\linewidth]{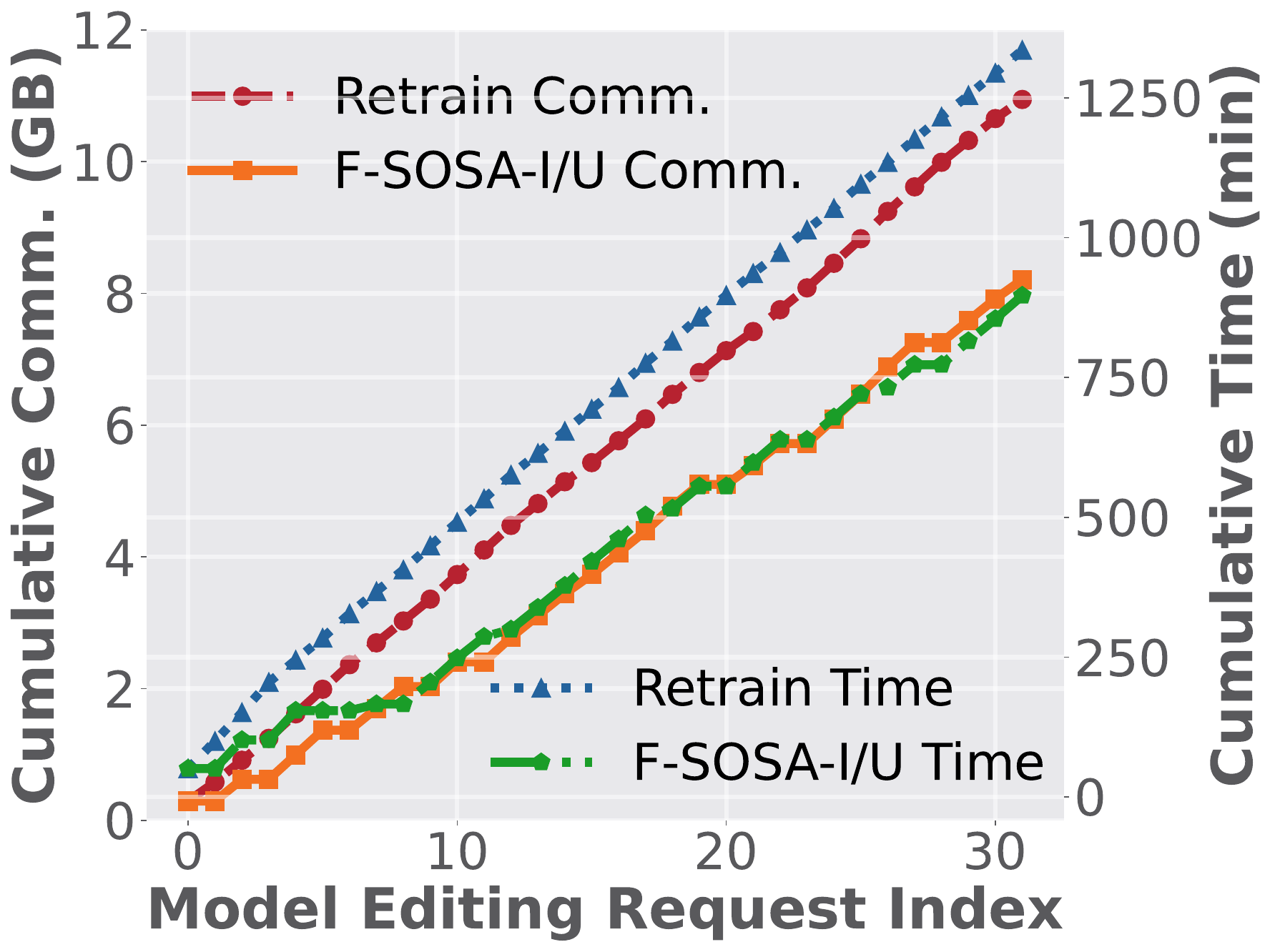}

    \caption{mAP$_{50\text{-}95}$ (left) and cumulative cost (right) of F-SOSA-I/U vs. Retrain under simulator-generated model editing requests on the real driving dataset BDD100K.}
    \label{fig:bdd100k-stream-editing}
\end{figure}

We further evaluate F-SOSA-I/U on a dynamic federated object-detection task by combining mobility traces generated by the Simulation of Urban MObility (SUMO) with the real-world BDD100K driving dataset. Following the mobility-trace setting described in section~\ref{sec:SUMO}, the trace produces $31$ sequential editing requests, comprising $16$ client arrivals and $15$ departures. We replay this request sequence on BDD100K using YOLOv8s and report the mAP$_{50\text{-}95}$ after each request, together with the cumulative computation time and communication cost, in Figure~\ref{fig:bdd100k-stream-editing}. F-SOSA-I/U closely tracks the utility of Retrain throughout the sequence, with a maximum mAP$_{50\text{-}95}$ difference of only $0.39$ percentage points. Across the $31$ editing requests, F-SOSA-I/U matches or outperforms Retrain after $19$ requests and achieves a slightly higher average mAP$_{50\text{-}95}$ of $20.25\%$, compared with $20.21\%$ for Retrain. After the final request, its mAP$_{50\text{-}95}$ reaches $20.27\%$, exceeding the $20.19\%$ achieved by Retrain. F-SOSA-I/U also exhibits greater stability: its mAP$_{50\text{-}95}$ varies within a range of $0.35$ percentage points, compared with $0.55$ for Retrain, while its maximum request-to-request variation is $0.28$ percentage points, versus $0.46$ for Retrain.

More notably, F-SOSA-I/U provides substantial efficiency gains. As shown in the figure, its communication and time curves remain essentially flat for $8$ of the $31$ requests. When the client involved in a request is not selected for aggregation, Bernoulli sampling avoids unnecessary computation and communication. These skipped updates save $2.74$ GB of communication and $343.73$ minutes of computation, accounting for $100\%$ and approximately $78.31\%$ of the overall communication and time savings, respectively. The savings do not arise solely from skipped updates: across the remaining $23$ requests that require model updates, F-SOSA-I/U saves an additional $95.19$ minutes. This additional time saving occurs because some requests require F-SOSA-I/U to update only one client with a relatively small local dataset, whereas Retrain must wait for all participating clients to complete local training and is therefore limited by the slowest client. After the complete request sequence, F-SOSA-I/U incurs $8.21$ GB of cumulative communication, compared with $10.94$ GB for Retrain, corresponding to a $25.00\%$ reduction. Its cumulative computation time decreases from $1335.21$ to $896.28$ minutes, yielding a $32.87\%$ reduction. These absolute savings are particularly significant on BDD100K, since repeatedly retraining YOLOv8s and communicating its model updates over complex real-world driving data is considerably more expensive than editing lightweight classification models.

Taken together, F-SOSA-I/U preserves near-retraining detection utility while substantially reducing the cost of frequent recomputation, demonstrating its practical value in dynamic and communication-constrained edge environments.

\subsection{Detailed Communication-Cost Breakdown}
\label{app:comm_break}

\begin{table*}[t]
\centering
\caption{Detailed communication costs (MB) of F-SOSA for $M=20$ and $\rho=0.5$, where $M$ is the total number of clients and $\rho$ is the client selection ratio. The downlink corresponds to one broadcast of the global model to all active clients. Model Up., Diag Up., and K-FAC Up. denote the aggregate uplink communication of the selected clients for local models, diagonal Fisher information, and Kronecker-Factored Approximate Curvature (K-FAC) Fisher information, respectively. Unselected clients incur no uplink communication. The final column reports the total communication costs of the Diag and K-FAC variants.}
\label{tab:fsosa_comm}
\setlength{\tabcolsep}{4pt}
\renewcommand{\arraystretch}{1.05}
\begin{tabular}{llrrrrrr}
\toprule
Dataset & Model
& \multicolumn{1}{c}{Broadcast}
& \multicolumn{3}{c}{Selected Clients}
& \multicolumn{1}{c}{Unselected Clients}
& Round Total \\
\cmidrule(lr){3-3}
\cmidrule(lr){4-6}
\cmidrule(lr){7-7}
& & Downlink & Model Up. & Diag Up. & K-FAC Up. & Uplink & Diag/K-FAC \\
\midrule
MNIST
& Two-Layer CNN
& 1.61 & 8.04 & 8.04 & 8.81 & 0
& 17.69/18.46 \\

SVHN
& Two-Layer CNN
& 3.58 & 17.92 & 17.92 & 18.81 & 0
& 39.42/40.31 \\

CIFAR10
& Two-Layer CNN
& 3.58 & 17.92 & 17.92 & 18.81 & 0
& 39.42/40.31 \\

CIFAR10
& ResNet18-Pretrained
& 42.63 & 213.13 & 213.13 & 206.68 & 0
& 468.89/462.44 \\

FEMNIST
& Two-Layer CNN
& 25.18 & 125.91 & 125.91 & 138.59 & 0
& 277.00/289.68 \\

BDD100K
& YOLOv8s
& 42.49 & 212.47 & 212.47 & 211.99 & 0
& 467.43/466.95 \\
\bottomrule
\end{tabular}
\end{table*}

To provide a clear breakdown of the communication costs, Table~\ref{tab:fsosa_comm} reports the per-round costs for selected and unselected clients in F-SOSA with $M=20$ and $\rho=0.5$. Under this configuration, ten clients are selected, while the remaining ten are unselected. The server first broadcasts the global model to all clients. Then, only the selected clients upload their local models and approximate Fisher information, represented by either Diag or K-FAC, whereas the unselected clients do not upload any information. Note that although four experiments in the table use two-layer CNNs, their architectures and parameter counts differ, resulting in different model sizes.

As shown in the table, the uplink cost of the Diag Fisher information is equal to that of the 16-bit quantized local model for all model architectures. Their combined size is therefore equal to the communication size of one FP32 model. This is because 16-bit quantization is applied to both the local models and the Diag Fisher information in our implementation. After receiving the quantized information, the server reconstructs the values and performs aggregation. More specifically, each local model and its diagonal Fisher information are quantized layer by layer to 16-bit precision before aggregation. They have the same communication cost because both contain one value for each model parameter. For K-FAC, the Fisher factors are further compressed using low-rank decomposition and quantization. Therefore, the K-FAC communication cost depends on the model architecture and the dimensions of its Fisher factors, and may differ slightly from the Diag cost. The table also shows that unselected clients incur no uplink cost. The final column reports the total communication costs of the Diag and K-FAC variants, respectively. The two-layer CNN has the same communication cost on SVHN and CIFAR-10 because the two settings use the same model architecture and number of parameters.

Furthermore, as shown in Tables~\ref{tab:mnist-mnistcnn-stream-comm-detail}--\ref{tab:bdd100k_stream_comm_detail}, we report a detailed breakdown of the actual amount of data transmitted during the initial training round and 31 sequential editing requests. The communication cost includes three parts: global-model broadcasts, local-model uploads from selected clients, and approximate Fisher information uploads from selected clients. Unselected clients do not upload models or Fisher information. Therefore, the reported totals include all communication required to exchange the current model and Fisher information between the server and clients. With this complete accounting, F-SOSA-I/U reduces the cumulative communication cost from $424.48$ to $329.61$ MB on MNIST, from $966.63$ to $717.63$ MB on SVHN, from $6942.18$ to $5416.97$ MB on FEMNIST, from $966.63$ to $750.59$ MB on CIFAR-10, and from $11207.23$ to $8405.42$ MB on BDD100K. These results correspond to reductions of $22.35\%$, $25.76\%$, $21.97\%$, $22.35\%$, and $25.00\%$, respectively. The detailed BDD100K results further show that the F-SOSA-I/U total of $8405.42$ MB consists of $1019.76$ MB of global-model broadcasts, $3696.95$ MB of local-model uploads, and $3688.71$ MB of Fisher-information uploads. In comparison, retraining requires $1359.68$, $4929.27$, and $4918.28$ MB for the same three parts, respectively. Thus, the reduction is obtained even after including the large Fisher-information uploads and all model transfers. When an editing request requires an update, F-SOSA-I/U has the same per-round communication components as retraining; however, it avoids the full broadcast and upload process when the request does not require updating the selected clients. This saves $94.87$, $249.00$, $1525.21$, $216.04$, and $2801.81$ MB on MNIST, SVHN, FEMNIST, CIFAR-10, and BDD100K, respectively. SVHN achieves the largest relative reduction of $25.76\%$, while BDD100K achieves the largest absolute reduction of $2801.81$ MB. Overall, F-SOSA-I/U provides consistent communication savings across datasets, ranging from $21.97\%$ on FEMNIST to $25.76\%$ on SVHN, while preserving performance comparable to retraining.

\begin{table*}[t]
\centering
\caption{Detailed communication costs (MB) of sequential editing on MNIST with $\rho=0.5$. Each communication entry is reported as F-SOSA-I/U versus Retrain, with the first value corresponding to F-SOSA-I/U and the second to Retrain. Model Up. and Fisher Up. denote the communication costs of uploading the model and the approximate Fisher information, respectively. In the Request column, the number identifies the affected client, while $+$ and $-$ denote client addition and client removal, respectively.}
\label{tab:mnist-mnistcnn-stream-comm-detail}
\setlength{\tabcolsep}{3pt}
\renewcommand{\arraystretch}{1.05}
\begin{tabular}{lrrrrr}
\toprule
Request & Broadcast & \multicolumn{2}{c}{Selected Clients}
& Unselected Clients
& Round Total \\
\cmidrule(lr){2-2}\cmidrule(lr){3-4}\cmidrule(lr){5-5}
& Downlink & Model Up. & Fisher Up. & Uplink & \\
\midrule
Initial & 1.61/1.61 & 4.82/4.82 & 4.82/4.82 & 0.00/0.00 & 11.25/11.25 \\
12+ & 0.00/1.61 & 0.00/4.82 & 0.00/4.82 & 0.00/0.00 & 0.00/11.25 \\
13+ & 1.61/1.61 & 5.63/5.63 & 5.63/5.63 & 0.00/0.00 & 12.87/12.87 \\
14+ & 0.00/1.61 & 0.00/5.63 & 0.00/5.63 & 0.00/0.00 & 0.00/12.87 \\
15+ & 1.61/1.61 & 6.43/6.43 & 6.43/6.43 & 0.00/0.00 & 14.47/14.47 \\
16+ & 0.00/1.61 & 0.00/6.43 & 0.00/6.43 & 0.00/0.00 & 0.00/14.47 \\
3- & 0.00/1.61 & 0.00/6.43 & 0.00/6.43 & 0.00/0.00 & 0.00/14.47 \\
6- & 1.61/1.61 & 5.63/5.63 & 5.63/5.63 & 0.00/0.00 & 12.87/12.87 \\
7- & 1.61/1.61 & 5.63/5.63 & 5.63/5.63 & 0.00/0.00 & 12.87/12.87 \\
17+ & 1.61/1.61 & 5.63/5.63 & 5.63/5.63 & 0.00/0.00 & 12.87/12.87 \\
18+ & 1.61/1.61 & 6.43/6.43 & 6.43/6.43 & 0.00/0.00 & 14.47/14.47 \\
19+ & 0.00/1.61 & 0.00/6.43 & 0.00/6.43 & 0.00/0.00 & 0.00/14.47 \\
1- & 1.61/1.61 & 6.43/6.43 & 6.43/6.43 & 0.00/0.00 & 14.47/14.47 \\
11- & 1.61/1.61 & 5.63/5.63 & 5.63/5.63 & 0.00/0.00 & 12.87/12.87 \\
16- & 1.61/1.61 & 5.63/5.63 & 5.63/5.63 & 0.00/0.00 & 12.87/12.87 \\
2- & 1.61/1.61 & 4.82/4.82 & 4.82/4.82 & 0.00/0.00 & 11.25/11.25 \\
20+ & 1.61/1.61 & 5.63/5.63 & 5.63/5.63 & 0.00/0.00 & 12.87/12.87 \\
21+ & 1.61/1.61 & 5.63/5.63 & 5.63/5.63 & 0.00/0.00 & 12.87/12.87 \\
22+ & 1.61/1.61 & 6.43/6.43 & 6.43/6.43 & 0.00/0.00 & 14.47/14.47 \\
4- & 1.61/1.61 & 5.63/5.63 & 5.63/5.63 & 0.00/0.00 & 12.87/12.87 \\
8- & 0.00/1.61 & 0.00/5.63 & 0.00/5.63 & 0.00/0.00 & 0.00/12.87 \\
9- & 1.61/1.61 & 4.82/4.82 & 4.82/4.82 & 0.00/0.00 & 11.25/11.25 \\
23+ & 1.61/1.61 & 5.63/5.63 & 5.63/5.63 & 0.00/0.00 & 12.87/12.87 \\
24+ & 1.61/1.61 & 5.63/5.63 & 5.63/5.63 & 0.00/0.00 & 12.87/12.87 \\
25+ & 1.61/1.61 & 6.43/6.43 & 6.43/6.43 & 0.00/0.00 & 14.47/14.47 \\
26+ & 1.61/1.61 & 6.43/6.43 & 6.43/6.43 & 0.00/0.00 & 14.47/14.47 \\
27+ & 1.61/1.61 & 7.23/7.23 & 7.23/7.23 & 0.00/0.00 & 16.07/16.07 \\
0- & 1.61/1.61 & 6.43/6.43 & 6.43/6.43 & 0.00/0.00 & 14.47/14.47 \\
10- & 0.00/1.61 & 0.00/6.43 & 0.00/6.43 & 0.00/0.00 & 0.00/14.47 \\
17- & 1.61/1.61 & 5.63/5.63 & 5.63/5.63 & 0.00/0.00 & 12.87/12.87 \\
18- & 1.61/1.61 & 5.63/5.63 & 5.63/5.63 & 0.00/0.00 & 12.87/12.87 \\
19- & 1.61/1.61 & 4.82/4.82 & 4.82/4.82 & 0.00/0.00 & 11.25/11.25 \\
\midrule
\textbf{Total} & \textbf{40.25/51.52} & \textbf{144.68/186.48} & \textbf{144.68/186.48} & \textbf{0.00/0.00} & \textbf{329.61/424.48} \\
\bottomrule
\end{tabular}
\end{table*}

\begin{table*}[t]
\centering
\caption{Detailed communication costs (MB) of sequential editing on SVHN with $\rho=0.5$. Each communication entry is reported as F-SOSA-I/U versus Retrain, with the first value corresponding to F-SOSA-I/U and the second to Retrain. Model Up. and Fisher Up. denote the communication costs of uploading the model and the approximate Fisher information, respectively. In the Request column, the number identifies the affected client, while $+$ and $-$ denote client addition and client removal, respectively.}
\label{tab:svhn-cnn-stream-comm-detail}
\setlength{\tabcolsep}{3pt}
\renewcommand{\arraystretch}{1.05}
\begin{tabular}{lrrrrr}
\toprule
Request & Broadcast & \multicolumn{2}{c}{Selected Clients} & Unselected Clients & Round Total \\
\cmidrule(lr){2-2}\cmidrule(lr){3-4}\cmidrule(lr){5-5}
& Downlink & Model Up. & Fisher Up. & Uplink & \\
\midrule
Initial & 3.58/3.58 & 10.75/10.75 & 11.28/11.28 & 0.00/0.00 & 25.61/25.61 \\
12+ & 3.58/3.58 & 10.75/10.75 & 11.28/11.28 & 0.00/0.00 & 25.61/25.61 \\
13+ & 3.58/3.58 & 12.55/12.55 & 13.16/13.16 & 0.00/0.00 & 29.29/29.29 \\
14+ & 3.58/3.58 & 12.55/12.55 & 13.16/13.16 & 0.00/0.00 & 29.29/29.29 \\
15+ & 3.58/3.58 & 14.34/14.34 & 15.04/15.04 & 0.00/0.00 & 32.96/32.96 \\
16+ & 3.58/3.58 & 14.34/14.34 & 15.04/15.04 & 0.00/0.00 & 32.96/32.96 \\
3- & 0.00/3.58 & 0.00/14.34 & 0.00/15.04 & 0.00/0.00 & 0.00/32.96 \\
6- & 3.58/3.58 & 12.55/12.55 & 13.16/13.16 & 0.00/0.00 & 29.29/29.29 \\
7- & 3.58/3.58 & 12.55/12.55 & 13.16/13.16 & 0.00/0.00 & 29.29/29.29 \\
17+ & 0.00/3.58 & 0.00/12.55 & 0.00/13.16 & 0.00/0.00 & 0.00/29.29 \\
18+ & 3.58/3.58 & 14.34/14.34 & 15.04/15.04 & 0.00/0.00 & 32.96/32.96 \\
19+ & 0.00/3.58 & 0.00/14.34 & 0.00/15.04 & 0.00/0.00 & 0.00/32.96 \\
1- & 3.58/3.58 & 14.34/14.34 & 15.04/15.04 & 0.00/0.00 & 32.96/32.96 \\
11- & 3.58/3.58 & 12.55/12.55 & 13.16/13.16 & 0.00/0.00 & 29.29/29.29 \\
16- & 0.00/3.58 & 0.00/12.55 & 0.00/13.16 & 0.00/0.00 & 0.00/29.29 \\
2- & 3.58/3.58 & 10.75/10.75 & 11.28/11.28 & 0.00/0.00 & 25.61/25.61 \\
20+ & 3.58/3.58 & 12.55/12.55 & 13.16/13.16 & 0.00/0.00 & 29.29/29.29 \\
21+ & 3.58/3.58 & 12.55/12.55 & 13.16/13.16 & 0.00/0.00 & 29.29/29.29 \\
22+ & 3.58/3.58 & 14.34/14.34 & 15.04/15.04 & 0.00/0.00 & 32.96/32.96 \\
4- & 3.58/3.58 & 12.55/12.55 & 13.16/13.16 & 0.00/0.00 & 29.29/29.29 \\
8- & 0.00/3.58 & 0.00/12.55 & 0.00/13.16 & 0.00/0.00 & 0.00/29.29 \\
9- & 3.58/3.58 & 10.75/10.75 & 11.28/11.28 & 0.00/0.00 & 25.61/25.61 \\
23+ & 3.58/3.58 & 12.55/12.55 & 13.16/13.16 & 0.00/0.00 & 29.29/29.29 \\
24+ & 0.00/3.58 & 0.00/12.55 & 0.00/13.16 & 0.00/0.00 & 0.00/29.29 \\
25+ & 3.58/3.58 & 14.34/14.34 & 15.04/15.04 & 0.00/0.00 & 32.96/32.96 \\
26+ & 0.00/3.58 & 0.00/14.34 & 0.00/15.04 & 0.00/0.00 & 0.00/32.96 \\
27+ & 3.58/3.58 & 16.13/16.13 & 16.93/16.93 & 0.00/0.00 & 36.64/36.64 \\
0- & 3.58/3.58 & 14.34/14.34 & 15.04/15.04 & 0.00/0.00 & 32.96/32.96 \\
10- & 0.00/3.58 & 0.00/14.34 & 0.00/15.04 & 0.00/0.00 & 0.00/32.96 \\
17- & 3.58/3.58 & 12.55/12.55 & 13.16/13.16 & 0.00/0.00 & 29.29/29.29 \\
18- & 3.58/3.58 & 12.55/12.55 & 13.16/13.16 & 0.00/0.00 & 29.29/29.29 \\
19- & 3.58/3.58 & 10.75/10.75 & 11.28/11.28 & 0.00/0.00 & 25.61/25.61 \\
\midrule
\textbf{Total} & \textbf{85.92/114.56} & \textbf{308.25/415.78} & \textbf{323.46/436.29} & \textbf{0.00/0.00} & \textbf{717.63/966.63} \\
\bottomrule
\end{tabular}
\end{table*}

\begin{table*}[t]
\centering
\caption{Detailed communication costs (MB) of sequential editing on FEMNIST with $\rho=0.5$. Each communication entry is reported as F-SOSA-I/U versus Retrain, with the first value corresponding to F-SOSA-I/U and the second to Retrain. Model Up. and Fisher Up. denote the communication costs of uploading the model and the approximate Fisher information, respectively. In the Request column, the number identifies the affected client, while $+$ and $-$ denote client addition and client removal, respectively.}
\label{tab:femnist-nistcnn-stream-comm-detail}
\setlength{\tabcolsep}{3pt}
\renewcommand{\arraystretch}{1.05}
\begin{tabular}{lrrrrr}
\toprule
Request & Broadcast & \multicolumn{2}{c}{Selected Clients} & Unselected Clients & Round Total \\
\cmidrule(lr){2-2}\cmidrule(lr){3-4}\cmidrule(lr){5-5}
& Downlink & Model Up. & Fisher Up. & Uplink & \\
\midrule
Initial & 25.18/25.18 & 75.55/75.55 & 83.15/83.15 & 0.00/0.00 & 183.88/183.88 \\
12+ & 0.00/25.18 & 0.00/75.55 & 0.00/83.15 & 0.00/0.00 & 0.00/183.88 \\
13+ & 25.18/25.18 & 88.14/88.14 & 97.01/97.01 & 0.00/0.00 & 210.33/210.33 \\
14+ & 0.00/25.18 & 0.00/88.14 & 0.00/97.01 & 0.00/0.00 & 0.00/210.33 \\
15+ & 25.18/25.18 & 100.73/100.73 & 110.87/110.87 & 0.00/0.00 & 236.78/236.78 \\
16+ & 0.00/25.18 & 0.00/100.73 & 0.00/110.87 & 0.00/0.00 & 0.00/236.78 \\
3- & 25.18/25.18 & 100.73/100.73 & 110.87/110.87 & 0.00/0.00 & 236.78/236.78 \\
6- & 25.18/25.18 & 88.14/88.14 & 97.01/97.01 & 0.00/0.00 & 210.33/210.33 \\
7- & 25.18/25.18 & 88.14/88.14 & 97.01/97.01 & 0.00/0.00 & 210.33/210.33 \\
17+ & 25.18/25.18 & 88.14/88.14 & 97.01/97.01 & 0.00/0.00 & 210.33/210.33 \\
18+ & 25.18/25.18 & 100.73/100.73 & 110.87/110.87 & 0.00/0.00 & 236.78/236.78 \\
19+ & 25.18/25.18 & 100.73/100.73 & 110.87/110.87 & 0.00/0.00 & 236.78/236.78 \\
1- & 0.00/25.18 & 0.00/100.73 & 0.00/110.87 & 0.00/0.00 & 0.00/236.78 \\
11- & 25.18/25.18 & 88.14/88.14 & 97.01/97.01 & 0.00/0.00 & 210.33/210.33 \\
16- & 25.18/25.18 & 88.14/88.14 & 97.01/97.01 & 0.00/0.00 & 210.33/210.33 \\
2- & 25.18/25.18 & 75.55/75.55 & 83.15/83.15 & 0.00/0.00 & 183.88/183.88 \\
20+ & 25.18/25.18 & 88.14/88.14 & 97.01/97.01 & 0.00/0.00 & 210.33/210.33 \\
21+ & 25.18/25.18 & 88.14/88.14 & 97.01/97.01 & 0.00/0.00 & 210.33/210.33 \\
22+ & 25.18/25.18 & 100.73/100.73 & 110.87/110.87 & 0.00/0.00 & 236.78/236.78 \\
4- & 25.18/25.18 & 88.14/88.14 & 97.01/97.01 & 0.00/0.00 & 210.33/210.33 \\
8- & 25.18/25.18 & 88.14/88.14 & 97.01/97.01 & 0.00/0.00 & 210.33/210.33 \\
9- & 25.18/25.18 & 75.55/75.55 & 83.15/83.15 & 0.00/0.00 & 183.88/183.88 \\
23+ & 25.18/25.18 & 88.14/88.14 & 97.01/97.01 & 0.00/0.00 & 210.33/210.33 \\
24+ & 0.00/25.18 & 0.00/88.14 & 0.00/97.01 & 0.00/0.00 & 0.00/210.33 \\
25+ & 25.18/25.18 & 100.73/100.73 & 110.87/110.87 & 0.00/0.00 & 236.78/236.78 \\
26+ & 25.18/25.18 & 100.73/100.73 & 110.87/110.87 & 0.00/0.00 & 236.78/236.78 \\
27+ & 25.18/25.18 & 113.32/113.32 & 124.73/124.73 & 0.00/0.00 & 263.23/263.23 \\
0- & 25.18/25.18 & 100.73/100.73 & 110.87/110.87 & 0.00/0.00 & 236.78/236.78 \\
10- & 0.00/25.18 & 0.00/100.73 & 0.00/110.87 & 0.00/0.00 & 0.00/236.78 \\
17- & 25.18/25.18 & 88.14/88.14 & 97.01/97.01 & 0.00/0.00 & 210.33/210.33 \\
18- & 0.00/25.18 & 0.00/88.14 & 0.00/97.01 & 0.00/0.00 & 0.00/210.33 \\
19- & 25.18/25.18 & 75.55/75.55 & 83.15/83.15 & 0.00/0.00 & 183.88/183.88 \\
\midrule
\textbf{Total} & \textbf{629.50/805.76} & \textbf{2279.04/2921.20} & \textbf{2508.43/3215.22} & \textbf{0.00/0.00} & \textbf{5416.97/6942.18} \\
\bottomrule
\end{tabular}
\end{table*}

\begin{table*}[t]
\centering
\caption{Detailed communication costs (MB) of sequential editing on CIFAR10 with $\rho=0.5$. Each communication entry is reported as F-SOSA-I/U versus Retrain, with the first value corresponding to F-SOSA-I/U and the second to Retrain. Model Up. and Fisher Up. denote the communication costs of uploading the model and the approximate Fisher information, respectively. In the Request column, the number identifies the affected client, while $+$ and $-$ denote client addition and client removal, respectively.}
\label{tab:cifar10-cnn-stream-comm-detail}
\setlength{\tabcolsep}{3pt}
\renewcommand{\arraystretch}{1.05}
\begin{tabular}{lrrrrr}
\toprule
Request & Broadcast & \multicolumn{2}{c}{Selected Clients} & Unselected Clients & Round Total \\
\cmidrule(lr){2-2}\cmidrule(lr){3-4}\cmidrule(lr){5-5}
& Downlink & Model Up. & Fisher Up. & Uplink & \\
\midrule
Initial & 3.58/3.58 & 10.75/10.75 & 11.28/11.28 & 0.00/0.00 & 25.61/25.61 \\
12+ & 3.58/3.58 & 10.75/10.75 & 11.28/11.28 & 0.00/0.00 & 25.61/25.61 \\
13+ & 3.58/3.58 & 12.55/12.55 & 13.16/13.16 & 0.00/0.00 & 29.29/29.29 \\
14+ & 0.00/3.58 & 0.00/12.55 & 0.00/13.16 & 0.00/0.00 & 0.00/29.29 \\
15+ & 3.58/3.58 & 14.34/14.34 & 15.04/15.04 & 0.00/0.00 & 32.96/32.96 \\
16+ & 3.58/3.58 & 14.34/14.34 & 15.04/15.04 & 0.00/0.00 & 32.96/32.96 \\
3- & 3.58/3.58 & 14.34/14.34 & 15.04/15.04 & 0.00/0.00 & 32.96/32.96 \\
6- & 3.58/3.58 & 12.55/12.55 & 13.16/13.16 & 0.00/0.00 & 29.29/29.29 \\
7- & 0.00/3.58 & 0.00/12.55 & 0.00/13.16 & 0.00/0.00 & 0.00/29.29 \\
17+ & 0.00/3.58 & 0.00/12.55 & 0.00/13.16 & 0.00/0.00 & 0.00/29.29 \\
18+ & 3.58/3.58 & 14.34/14.34 & 15.04/15.04 & 0.00/0.00 & 32.96/32.96 \\
19+ & 3.58/3.58 & 14.34/14.34 & 15.04/15.04 & 0.00/0.00 & 32.96/32.96 \\
1- & 0.00/3.58 & 0.00/14.34 & 0.00/15.04 & 0.00/0.00 & 0.00/32.96 \\
11- & 3.58/3.58 & 12.55/12.55 & 13.16/13.16 & 0.00/0.00 & 29.29/29.29 \\
16- & 3.58/3.58 & 12.55/12.55 & 13.16/13.16 & 0.00/0.00 & 29.29/29.29 \\
2- & 3.58/3.58 & 10.75/10.75 & 11.28/11.28 & 0.00/0.00 & 25.61/25.61 \\
20+ & 3.58/3.58 & 12.55/12.55 & 13.16/13.16 & 0.00/0.00 & 29.29/29.29 \\
21+ & 3.58/3.58 & 12.55/12.55 & 13.16/13.16 & 0.00/0.00 & 29.29/29.29 \\
22+ & 3.58/3.58 & 14.34/14.34 & 15.04/15.04 & 0.00/0.00 & 32.96/32.96 \\
4- & 3.58/3.58 & 12.55/12.55 & 13.16/13.16 & 0.00/0.00 & 29.29/29.29 \\
8- & 0.00/3.58 & 0.00/12.55 & 0.00/13.16 & 0.00/0.00 & 0.00/29.29 \\
9- & 3.58/3.58 & 10.75/10.75 & 11.28/11.28 & 0.00/0.00 & 25.61/25.61 \\
23+ & 3.58/3.58 & 12.55/12.55 & 13.16/13.16 & 0.00/0.00 & 29.29/29.29 \\
24+ & 3.58/3.58 & 12.55/12.55 & 13.16/13.16 & 0.00/0.00 & 29.29/29.29 \\
25+ & 3.58/3.58 & 14.34/14.34 & 15.04/15.04 & 0.00/0.00 & 32.96/32.96 \\
26+ & 0.00/3.58 & 0.00/14.34 & 0.00/15.04 & 0.00/0.00 & 0.00/32.96 \\
27+ & 3.58/3.58 & 16.13/16.13 & 16.93/16.93 & 0.00/0.00 & 36.64/36.64 \\
0- & 3.58/3.58 & 14.34/14.34 & 15.04/15.04 & 0.00/0.00 & 32.96/32.96 \\
10- & 0.00/3.58 & 0.00/14.34 & 0.00/15.04 & 0.00/0.00 & 0.00/32.96 \\
17- & 3.58/3.58 & 12.55/12.55 & 13.16/13.16 & 0.00/0.00 & 29.29/29.29 \\
18- & 3.58/3.58 & 12.55/12.55 & 13.16/13.16 & 0.00/0.00 & 29.29/29.29 \\
19- & 3.58/3.58 & 10.75/10.75 & 11.28/11.28 & 0.00/0.00 & 25.61/25.61 \\
\midrule
\textbf{Total} & \textbf{89.50/114.56} & \textbf{322.59/415.78} & \textbf{338.50/436.29} & \textbf{0.00/0.00} & \textbf{750.59/966.63} \\
\bottomrule
\end{tabular}
\end{table*}

\begin{table*}[t]
\centering
\caption{Detailed communication costs (MB) of sequential editing on BDD100K with $\rho=0.5$. Each communication entry is reported as F-SOSA-I/U versus Retrain, with the first value corresponding to F-SOSA-I/U and the second to Retrain. Model Up. and Fisher Up. denote the communication costs of uploading the model and the approximate Fisher information, respectively. In the Request column, the number identifies the affected client, while $+$ and $-$ denote client addition and client removal, respectively.}
\label{tab:bdd100k_stream_comm_detail}
\setlength{\tabcolsep}{3pt}
\renewcommand{\arraystretch}{1.05}
\begin{tabular}{lrrrrr}
\toprule
Request & Broadcast & \multicolumn{2}{c}{Selected Clients} & Unselected Clients & Round Total \\
\cmidrule(lr){2-2}\cmidrule(lr){3-4}\cmidrule(lr){5-5}
& Downlink & Model Up. & Fisher Up. & Uplink & \\
\midrule
Initial & 42.49/42.49 & 127.48/127.48 & 127.20/127.20 & 0.00/0.00 & 297.17/297.17 \\
12+ & 0.00/42.49 & 0.00/127.48 & 0.00/127.20 & 0.00/0.00 & 0.00/297.17 \\
13+ & 42.49/42.49 & 148.73/148.73 & 148.40/148.40 & 0.00/0.00 & 339.62/339.62 \\
14+ & 0.00/42.49 & 0.00/148.73 & 0.00/148.40 & 0.00/0.00 & 0.00/339.62 \\
15+ & 42.49/42.49 & 169.97/169.97 & 169.60/169.60 & 0.00/0.00 & 382.06/382.06 \\
16+ & 42.49/42.49 & 169.97/169.97 & 169.60/169.60 & 0.00/0.00 & 382.06/382.06 \\
3- & 0.00/42.49 & 0.00/169.97 & 0.00/169.60 & 0.00/0.00 & 0.00/382.06 \\
6- & 42.49/42.49 & 148.73/148.73 & 148.40/148.40 & 0.00/0.00 & 339.62/339.62 \\
7- & 42.49/42.49 & 148.73/148.73 & 148.40/148.40 & 0.00/0.00 & 339.62/339.62 \\
17+ & 0.00/42.49 & 0.00/148.73 & 0.00/148.40 & 0.00/0.00 & 0.00/339.62 \\
18+ & 42.49/42.49 & 169.97/169.97 & 169.60/169.60 & 0.00/0.00 & 382.06/382.06 \\
19+ & 0.00/42.49 & 0.00/169.97 & 0.00/169.60 & 0.00/0.00 & 0.00/382.06 \\
1- & 42.49/42.49 & 169.97/169.97 & 169.60/169.60 & 0.00/0.00 & 382.06/382.06 \\
11- & 42.49/42.49 & 148.73/148.73 & 148.40/148.40 & 0.00/0.00 & 339.62/339.62 \\
16- & 42.49/42.49 & 148.73/148.73 & 148.40/148.40 & 0.00/0.00 & 339.62/339.62 \\
2- & 42.49/42.49 & 127.48/127.48 & 127.20/127.20 & 0.00/0.00 & 297.17/297.17 \\
20+ & 42.49/42.49 & 148.73/148.73 & 148.40/148.40 & 0.00/0.00 & 339.62/339.62 \\
21+ & 42.49/42.49 & 148.73/148.73 & 148.40/148.40 & 0.00/0.00 & 339.62/339.62 \\
22+ & 42.49/42.49 & 169.97/169.97 & 169.60/169.60 & 0.00/0.00 & 382.06/382.06 \\
4- & 42.49/42.49 & 148.73/148.73 & 148.40/148.40 & 0.00/0.00 & 339.62/339.62 \\
8- & 0.00/42.49 & 0.00/148.73 & 0.00/148.40 & 0.00/0.00 & 0.00/339.62 \\
9- & 42.49/42.49 & 127.48/127.48 & 127.20/127.20 & 0.00/0.00 & 297.17/297.17 \\
23+ & 42.49/42.49 & 148.73/148.73 & 148.40/148.40 & 0.00/0.00 & 339.62/339.62 \\
24+ & 0.00/42.49 & 0.00/148.73 & 0.00/148.40 & 0.00/0.00 & 0.00/339.62 \\
25+ & 42.49/42.49 & 169.97/169.97 & 169.60/169.60 & 0.00/0.00 & 382.06/382.06 \\
26+ & 42.49/42.49 & 169.97/169.97 & 169.60/169.60 & 0.00/0.00 & 382.06/382.06 \\
27+ & 42.49/42.49 & 191.22/191.22 & 190.80/190.80 & 0.00/0.00 & 424.51/424.51 \\
0- & 42.49/42.49 & 169.97/169.97 & 169.60/169.60 & 0.00/0.00 & 382.06/382.06 \\
10- & 0.00/42.49 & 0.00/169.97 & 0.00/169.60 & 0.00/0.00 & 0.00/382.06 \\
17- & 42.49/42.49 & 148.73/148.73 & 148.40/148.40 & 0.00/0.00 & 339.62/339.62 \\
18- & 42.49/42.49 & 148.73/148.73 & 148.40/148.40 & 0.00/0.00 & 339.62/339.62 \\
19- & 42.49/42.49 & 127.48/127.48 & 127.20/127.20 & 0.00/0.00 & 297.17/297.17 \\
\midrule
\textbf{Total} & \textbf{1019.76/1359.68} & \textbf{3696.95/4929.27} & \textbf{3688.71/4918.28} & \textbf{0.00/0.00} & \textbf{8405.42/11207.23} \\
\bottomrule
\end{tabular}
\end{table*}

}


\end{document}